\documentclass{article}

 \PassOptionsToPackage{numbers, compress}{natbib}
\usepackage[final]{nips_2017}

\usepackage[utf8]{inputenc} 
\usepackage[T1]{fontenc}    
\usepackage{hyperref}       
\usepackage{url}            
\usepackage{booktabs}       
\usepackage{amsfonts}       
\usepackage{nicefrac}       
\usepackage{microtype}      

\usepackage{algorithm}
\usepackage{algorithmic}

\usepackage[usenames,dvipsnames]{xcolor}
\usepackage{ifthen}

\usepackage{array}
\usepackage{arydshln}
\usepackage{multirow}

\usepackage{amsfonts}
\usepackage{amsmath}
\usepackage[amsthm,amsmath,thmmarks]{ntheorem}
\usepackage{amssymb}
\usepackage{comment}

\usepackage{enumerate}
\usepackage[shortlabels]{enumitem}

\hypersetup{
	colorlinks   = true, 
	urlcolor     = blue, 
	linkcolor    = blue, 
	citecolor   = blue 
}

\usepackage[hyphenbreaks]{breakurl}

\newtheorem{mythm}{Theorem}
\newtheorem{mylem}{Lemma}
\newtheorem{myfact}{Fact}
\newtheorem{mycond}{Condition}

\newtheorem{mydef}{Definition}

\newtheorem{myclm}{Claim}

\usepackage[normalem]{ulem}

\usepackage{tikz}
\usepackage{caption}
\usepackage{subcaption}

\newcommand{\compilefullversion}{true}
\newcommand{\compileunmarkchanges}{true}
\newcommand{\compileincludeappendix}{true}
\newcommand{\compilehidecomments}{true}

\newcommand{\compilearxivversion}{true}

\ifthenelse{\equal{\compilearxivversion}{false}}{%
	\newcommand{\OnlyInArxiv}[1]{}
}{%
	\newcommand{\OnlyInArxiv}[1]{#1}%
}%

\ifthenelse{\equal{\compilefullversion}{false}}{%
	\newcommand{\OnlyInFull}[1]{}
	\newcommand{\OnlyInShort}[1]{#1}
}{%
	\newcommand{\OnlyInFull}[1]{#1}%
	\newcommand{\OnlyInShort}[1]{}%
}%

\ifthenelse{\equal{\compileunmarkchanges}{false}}{%
	\newcommand{\chgdel}[1]{{\color{violet}\sout{#1}}}%
	\newcommand{\chgins}[1]{{\color{red}#1}}%
}{%
	\newcommand{\chgdel}[1]{}%
	\newcommand{\chgins}[1]{#1}%
}%

\ifthenelse{ \equal{\compilehidecomments}{true} }{%
	\newcommand{\wei}[1]{}
	\newcommand{\qinshi}[1]{}
}{
\newcommand{\wei}[1]{{\color{blue!50!black}  [\text{Wei:} #1]}}
\newcommand{\qinshi}[1]{{\color{brown!60!black} [\text{Qinshi:} #1]}}
}

\newcommand{\calD}{\mathcal{D}}
\newcommand{\calE}{\mathcal{E}}
\newcommand{\calF}{\mathcal{F}}
\newcommand{\calN}{\mathcal{N}}
\newcommand{\calO}{\mathcal{O}}
\newcommand{\calS}{\mathcal{S}}

\newcommand{\R}{\mathbb{R}}

\newcommand{\Ns}{\ensuremath{\calN^{\textnormal{s}}}}
\newcommand{\Nt}{\ensuremath{\calN^{\textnormal{t}}}}

\newcommand{\eps}{\varepsilon}

\newcommand{\opt}{{\mathrm{opt}}}

\newcommand{\I}{\mathbb{I}}
\newcommand\dx{\,\mathrm{d}x}
\newcommand{\vmu}{\ensuremath{\boldsymbol \mu}}

\newcommand{\trig}[1]{\tilde{#1}}

\newcommand{\Dtrig}{{D^{\rm trig}}}

\newcommand{\mab}{\mathrm{MAB}}
\newcommand{\cmab}{\mathrm{CMAB}}

\newcommand{\kl}{\mathrm{kl}}
\newcommand{\ekl}{\widehat{\kl}}
\newcommand{\E}{\mathop{\mathbb{E{}}}}
\newcommand{\Ep}{\mathop{\mathbb{E}'}}
\newcommand{\Prp}{\mathop{\mathrm{Pr}'}}

\let\emptyset\varnothing

\title{
Improving Regret Bounds for Combinatorial Semi-Bandits with Probabilistically Triggered Arms and 
	Its Applications\OnlyInArxiv{\thanks{This is the full version of the paper appearing at 
		the 30th Annual Conference on Advances in Neural Information Processing Systems (NIPS'2017), Long Beach, U.S.A., December, 2017.
	The arxiv.v5 version contains several typo fixes, and a new Appendix~\ref{sec:linbandit} that provides a slightly better regret bound for CMAB-T
	with linear reward functions.}}
	}
\author{
	Qinshi Wang\\
	Princeton University\\
	Princeton, NJ 08544\\
	\texttt{qinshiw@princeton.edu}\\
	\And
	Wei Chen\\
	Microsoft Research \\
	Beijing, China\\
	\texttt{weic@microsoft.com}\\
}

\begin{document} 

	\maketitle
	
	\begin{abstract} 
		We study combinatorial multi-armed bandit with probabilistically triggered arms and semi-bandit feedback (CMAB-T).
		We resolve a serious issue in the prior CMAB-T studies where the regret bounds contain a possibly exponentially large
			factor of $1/p^*$, where $p^*$ is the minimum positive probability that an arm is triggered by any action.
		We address this issue by introducing a triggering probability modulated (TPM) bounded smoothness condition into the
			general CMAB-T framework, and show that many applications such as influence maximization bandit and combinatorial
			cascading bandit satisfy this TPM condition.
		As a result, we completely remove the factor of $1/p^*$ from the regret bounds, achieving significantly better
			regret bounds
			for influence maximization and cascading bandits than before.
		Finally, we provide lower bound results showing that the factor $1/p^*$ is unavoidable for general CMAB-T problems,
			suggesting that the TPM condition is crucial in removing this factor.
	\end{abstract} 
	
%
%
		
	\section{Introduction}
	
	Stochastic multi-armed bandit (MAB) is a classical online learning framework modeled as a game between a player and 
		the environment with $m$ arms.
	In each round, the player selects one arm and the environment generates a reward of the arm from
		a distribution unknown to the player.
	The player observes the reward, and use it as the feedback to the player's algorithm (or policy)
		to select arms in future rounds.
	The goal of the player is to cumulate as much reward as possible over time.
	MAB models the classical dilemma between exploration and exploitation:
		whether the player should keep exploring arms in search for a better arm, or should stick
		to the best arm observed so far to collect rewards.
	The standard performance measure of the player's algorithm is the {\em (expected) regret}, which is the difference in expected
		cumulative reward between always playing the best arm in expectation and playing according to the player's algorithm.
	
	In recent years, stochastic combinatorial multi-armed bandit (CMAB) receives many attention
		(e.g. \cite{Yi2012,CWYW16,ChenGeneral16,GMM14,KWAEE14,KWAS15,kveton2015cascading,kveton2015combinatorial,Combes2015}),
		because it has wide applications in wireless networking, online advertising and recommendation, viral marketing in social
		networks, etc.
	In the typical setting of CMAB, the player selects a combinatorial action to play in each round, which would trigger the play
		of a set of arms, and the outcomes of these triggered arms are observed as the feedback (called semi-bandit feedback).
	Besides the exploration and exploitation tradeoff, CMAB also needs to deal with the exponential explosion of the possible actions that makes exploring all actions infeasible.
		
	One class of the above CMAB problems involves probabilistically triggered arms~\cite{CWYW16,kveton2015cascading,kveton2015combinatorial}, in which actions may trigger arms probabilistically.
	We denote it as CMAB-T in this paper.
	\citet{CWYW16} provide such a general model 
		and apply it
		to the influence maximization bandit, which models stochastic influence diffusion in social networks and sequentially
		selecting seed sets to maximize the cumulative influence spread over time.
	\citet{kveton2015cascading,kveton2015combinatorial} study cascading bandits, in which arms
		are probabilistically triggered following a sequential order selected by the player as the action.
%
%
%
	However, in both studies, the regret bounds contain an undesirable factor of
		$1/p^*$, where $p^*$ is the minimum positive probability
		that any arm can be triggered by any action,\footnote{The factor of $1/f^*$
		used for the combinatorial disjunctive cascading bandits in~\cite{kveton2015combinatorial} is
		essentially $1/p^*$.}
		and this factor could be exponentially large for both influence maximization and cascading bandits.

%
	In this paper, we adapt the general CMAB framework
	of \cite{CWYW16} in a systematic way to completely remove the factor of $1/p^*$ for a large
	class of CMAB-T problems including both influence maximization and combinatorial cascading bandits.
	The key observation is that for these problems,
		a harder-to-trigger arm has less impact to the expected reward and thus we do not need to observe it
		as often.
	We turn this key observation into a triggering probability modulated (TPM) bounded smoothness condition, 
		adapted from the original bounded smoothness condition in~\cite{CWYW16}.
	We eliminates the $1/p^*$ factor in the regret bounds for all CMAB-T problems with the TPM condition,
		and show that influence maximization bandit and the conjunctive/disjunctive cascading
		bandits all satisfy the TPM condition. 
	Moreover, for general CMAB-T without the TPM condition, we show a lower bound result that $1/p^*$ is
		unavoidable, because the hard-to-trigger arms are crucial in determining the best arm and have
		to be observed enough times.

Besides removing the exponential factor, our analysis is also tighter in other 
	regret factors or constants comparing to the existing influence maximization bandit 
	results~\cite{CWYW16,Wen2016}, combinatorial cascading bandit~\cite{kveton2015combinatorial},
	and linear bandits without probabilistically triggered arms~\cite{KWAS15}.
Both the regret analysis based on the TPM condition and the proof that influence maximization bandit satisfies the TPM condition
	are technically involved and nontrivial, but due to the space constraint, we have to move the complete proofs to the supplementary
	material. Instead we introduce the key techniques used in the main text.

	
\vspace{-2mm}
\paragraph{Related Work.}
	Multi-armed bandit problem is originally formated by \citet{Robbins52}, and 
		has been extensively studied in the literature \citep[cf.][]{BF85,SB98,BCB12}.
	Our study belongs to the stochastic bandit research, while there is another line of research on adversarial 
		bandits~\cite{AuerCFS02}, for which we refer to a survey like~\cite{BCB12} for further information.
	For stochastic MABs, an important approach is Upper Confidence Bound (UCB) approach \citep{AuerCF02}, 
		on which most CMAB studies are based upon.
%
%
%

As already mentioned in the introduction, stochastic CMAB has received many attention in recent years.
Among the studies, we improve 
	(a) the general framework with probabilistically triggered arms of \cite{CWYW16},
	(b) the influence maximization bandit results in \cite{CWYW16} and \cite{Wen2016},
	(c) the combinatorial cascading bandit results in \cite{kveton2015combinatorial}, and
	(d) the linear bandit results in \cite{KWAS15}.
We defer the technical comparison with these studies to Section~\ref{sec:discussion}.
Other CMAB studies do not deal with probabilistically triggered arms.
Among them, \cite{Yi2012} is the first study on linear stochastic bandit, but its regret bound has since been
	improved by \citet{CWYW16,KWAS15}.
\citet{Combes2015} improve the regret bound of \cite{KWAS15} for linear bandits in a special case where arms are
	mutually independent.
Most studies above are based on the UCB-style CUCB algorithm or its minor variant, and differ on the assumptions and
	regret analysis.
\citet{GMM14} study Thompson sampling for complex actions, which is based on the 
	Thompson sample approach \cite{thompson1933likelihood} and
	can be applied to CMAB, but their regret bound has
	a large exponential constant term. 

Influence maximization is first formulated as a discrete optimization problem by \citet{kempe03}, and has been extensively
	studied since (cf. \cite{chen2013information}).
Variants of influence maximization bandit have also been studied \cite{LeiMMCS15,VaswaniL15,VKWGLS17}.
\citet{LeiMMCS15} use a different objective of maximizing the expected size of the union of the influenced nodes
	over time.
\citet{VaswaniL15} discuss how to transfer node level feedback to 
	the edge level feedback, and then apply the result of \cite{CWYW16}.
\citet{VKWGLS17} replace the original maximization objective of influence spread with a
	heuristic surrogate function,
	avoiding the issue of probabilistically triggered arms. But their regret 
	is defined against a weaker benchmark relaxed by the approximation ratio of the surrogate function,
	and thus their theoretical result is weaker than ours.

%

\section{General Framework}
	In this section we present the general framework of combinatorial multi-armed bandit
		with probabilistically triggered arms originally proposed in~\cite{CWYW16}
		with a slight adaptation, and denote it as CMAB-T. 
	We illustrate that the influence maximization bandit~\cite{CWYW16} and 
		combinatorial cascading bandits~\cite{kveton2015cascading,kveton2015combinatorial} are example instances
		of CMAB-T.

	CMAB-T is described as a learning game between a 
		learning agent (or player) and the environment.
	The environment consists of $m$ random variables $X_1, \ldots, X_m$ 
		called {\em base arms} (or {\em arms}) following a joint
	distribution $D$ over $[0,1]^m$.
	Distribution $D$ is picked by
		the environment from a class of distributions $\calD$ before the game starts.
The player knows $\calD$ but not the actual distribution $D$.
		
	The learning process proceeds in discrete rounds.
	In round $t\ge 1$, the player selects an action $S_t$ from an 
		action space $\calS$ based on the feedback history
		from the previous rounds, and the environment draws from the joint distribution $D$
		an independent sample $X^{(t)} = (X^{(t)}_1, \ldots, X^{(t)}_m)$.
	When action $S_t$ is played on the environment outcome $X^{(t)}$, a random subset of arms $\tau_t \subseteq [m]$ are triggered, 
		and the outcomes of $X^{(t)}_i$ for all $i\in \tau_t$ are observed as the feedback to the player.
	The player also obtains a nonnegative reward $R(S_t, X^{(t)}, \tau_t)$ fully determined by $S_t, X^{(t)},$ and $\tau_t$.
	A learning algorithm aims at properly selecting actions $S_t$'s over time based on the past feedback to cumulate as much reward as
		possible.
	Different from~\cite{CWYW16}, we allow the action space $\calS$ to be infinite.
	In the supplementary material, we discuss an example of continuous influence
	maximization~\cite{YangMPH16} that uses continuous and infinite
	action space while the number of base arms is still finite.
		
	We now describe the triggered set $\tau_t$ in more detail, which is not explicit in~\cite{CWYW16}.
	In general, $\tau_t$ may have additional randomness beyond the randomness of $X^{(t)}$.
	Let $\Dtrig(S,X)$ denote a distribution of the triggered subset of $[m]$
		for a given action $S$ and an environment outcome $X$.
	We assume that $\tau_t$ is drawn independently from $\Dtrig(S_t,X^{(t)})$.
	We refer $\Dtrig$ as the {\em probabilistic triggering function}.

	To summarize, a {\em CMAB-T problem instance} is a tuple $([m], \calS, \calD, \Dtrig, R)$, with
		elements already described above.
	These elements are known to the player, and hence establishing the problem input to the player.
	In contrast, the {\em environment instance} is the actual distribution $D \in \calD$ picked by the environment, and is unknown to the player.
	The problem instance and the environment instance together form the {\em (learning) game instance}, in which
		the learning process would unfold.
	In this paper, we fix the environment instance $D$, unless we need to refer to more than one environment 
		instances.

	For each arm $i$, let $\mu_i=\E_{X \sim D}[X_i]$. 
	Let vector $\vmu = (\mu_1, \ldots, \mu_m)$ denote the expectation vector of arms.
	Note that vector $\vmu$ is determined by $D$.
	Same as in~\cite{CWYW16}, we assume that the expected reward $\E[R(S,X,\tau)]$, where the expectation is 
		taken over $X \sim D$ and $\tau \sim \Dtrig(S,X)$,
		is a function of action $S$ and the expectation vector $\vmu$ of the arms.
	Henceforth, we denote $r_S(\vmu) \triangleq  \E[R(S,X,\tau)]$.
We remark that Chen et al.~\cite{ChenGeneral16} relax the above assumption and
	consider the case where the entire distribution $D$, not just the mean of $D$, 
	is needed to determine the expected reward.
However, they need to assume that arm outcomes are mutually independent, and they
	do not consider probabilistically triggered arms.
It might be interesting to incorporate probabilistically triggered arms into
	their setting, but this is out of the scope of the current paper.
	\chgins{To allow algorithm to estimate $\mu_i$ directly from samples,
		we assume the outcome of an arm does not depend on whether itself is triggered,
		i.e. $\E_{X \sim D, \tau\sim \Dtrig(S, X)}[X_i\mid i\in \tau]=\E_{X\sim D}[X_i]$.
	}
	
	The performance of a learning algorithm $A$ is measured by its {\em (expected) regret}, which is the difference in expected cumulative 
	reward between always playing the best action and playing actions selected by algorithm $A$.
	Formally, let  $\opt_{\vmu} = \sup_{S\in \calS} r_{S}(\vmu)$, where
	$\vmu = \E_{X\sim D}[X]$, and we assume that $\opt_{\vmu}$ is finite.
	Same as in~\cite{CWYW16}, we assume that the learning algorithm has access to an offline 
		{\em $(\alpha, \beta)$-approximation oracle ${\calO}$}, which takes
		$\vmu=(\mu_1, \ldots, \mu_m)$ as input and outputs an action $S^{\calO}$ such that
		$\Pr\{r_{\vmu}(S^{\calO}) \ge \alpha \cdot \opt_{\vmu}\} \ge \beta$, where $\alpha$ is the {\em approximation ratio}
		and $\beta$ is the success probability.
	Under the $(\alpha, \beta)$-approximation oracle, the benchmark cumulative reward should be the
		$\alpha\beta$ fraction of the optimal reward, and thus we use the following $(\alpha, \beta)$-approximation regret:
	\begin{mydef}[$(\alpha, \beta)$-approximation Regret]
		The $T$-round $(\alpha, \beta)$-approximation 
		regret of a learning algorithm $A$ (using an $(\alpha, \beta)$-approximation oracle)
		for a CMAB-T game instance $([m], \calS, \calD, \Dtrig, R, D)$
		with $\vmu = \E_{X\sim D}[X]$ is
		\begin{equation*}
		Reg^A_{\vmu,\alpha,\beta}(T) = T\cdot\alpha \cdot \beta \cdot \opt_{\vmu} - 
			\E\left[\sum_{i=1}^T R(S_t^A,X^{(t)},\tau_t)\right]
			= T\cdot\alpha \cdot \beta \cdot \opt_{\vmu} - 
			\E\left[\sum_{i=1}^T r_{S_t^A}(\vmu)\right],
		\end{equation*}
	\end{mydef}
	where $S^A_t$ is the action $A$ selects in round $t$, and
	the expectation is taken over the randomness of the environment outcomes
	$X^{(1)}, \ldots, X^{(T)}$, the triggered sets $\tau_1, \ldots, \tau_T$, as
	well as the possible randomness of algorithm $A$ itself.	

We remark that because probabilistically triggered arms may strongly impact 
	the determination of the
	best action, but they may be hard to trigger and observe, 
	the regret could be worse and the regret analysis is in general harder than
	CMAB without probabilistically triggered arms.

The above framework essentially follows \cite{CWYW16}, but we decouple actions from subsets of
	arms, allow action space to be infinite, and explicitly model
	triggered set distribution, which makes the framework more powerful in modeling certain
	applications (see supplementary material for more discussions).

\subsection{Examples of CMAB-T: Influence Maximization and Cascading Bandits}

%
In social influence maximization~\cite{kempe03}, we are given a weighted directed graph $G=(V,E,p)$, where $V$ and $E$ are sets of
	vertices and edges respectively, and each edge $(u,v)$ is associated with a probability $p(u,v)$.
Starting from a seed set $S \subseteq V$, influence propagates in $G$ as follows: nodes in $S$ are activated at time $0$, and
	at time $t\ge 1$, a node $u$ activated in step $t-1$ has one chance to activate its inactive out-neighbor $v$
	with an independent probability $p(u,v)$.
The {\em influence spread} of seed set $S$, $\sigma(S)$, is the expected number of activated nodes after the propagation ends.
The offline problem of {\em influence maximization} is to find at most $k$ seed nodes in $G$ such that
the influence spread is maximized.
\citet{kempe03} provide a greedy algorithm with approximation ratio
$1-1/e - \varepsilon$ and success probability $1 - 1/|V|$, for any $\varepsilon > 0$.


For the online influence maximization bandit~\cite{CWYW16}, the edge probabilities $p(u,v)$'s are unknown and need to be learned over time
	through repeated influence maximization tasks: in each round $t$, 
	$k$ seed nodes $S_t$ are selected, the influence propagation from $S_t$ is
	observed, the reward is the number of nodes activated in this round, and one
	wants to repeat this process to cumulate as much reward as possible.
Putting it into the CMAB-T framework, the set of edges $E$ is the set of arms $[m]$, and
	their outcome distribution $D$ is the joint distribution of $m$ independent Bernoulli distributions with means
	$p(u,v)$ for all $(u,v)\in E$.
Any seed set $S \subseteq V$ with at most $k$ nodes is an action.
The triggered arm set $\tau_t$ is
	the set of edges $(u,v)$ reached by the propagation, 
	that is, $u$ can be reached from $S_t$ by passing through only edges $e\in E$ with $X^{(t)}_e = 1$.
In this case, the distribution $\Dtrig(S_t,X^{(t)})$ degenerates to a deterministic triggered set.
The reward $R(S_t,X^{(t)},\tau_t)$ equals to the number of nodes in $V$ that is reached from $S$ through only edges $e\in E$
	with $X^{(t)}_e = 1$, and the expected reward is exactly the influence spread $\sigma(S_t)$.
The offline oracle is a $(1-1/e-\varepsilon, 1/|V|)$-approximation greedy algorithm.
We remark that the general triggered set distribution $\Dtrig(S_t,X^{(t)})$ (together with infinite
	action space) can be used to model extended versions of
	influence maximization, such as randomly selected seed sets in general marketing actions~\cite{kempe03}
	and continuous influence maximization~\cite{YangMPH16} (see supplementary material).

Now let us consider combinatorial cascading bandits~\cite{kveton2015cascading,kveton2015combinatorial}.
In this case, we have $m$ independent Bernoulli random variables $X_1, \ldots, X_m$ as base arms.
An action is to select an ordered sequence from a subset of these arms satisfying certain constraint.
Playing this action means that the player reveals the outcomes
	of the arms one by one following the sequence order until certain stopping condition is satisfied.
The feedback is the outcomes of revealed arms and the reward is a function form of these arms.
In particular, in the disjunctive form the player stops when the first $1$ is revealed and she gains reward
	of $1$, or she reaches the end and gains reward $0$.
In the conjunctive form, the player stops when the first $0$ is revealed (and receives reward $0$) 
	or she reaches the end with all $1$ outcomes (and receives reward $1$).
Cascading bandits can be used to model online recommendation and advertising (in the disjunctive form with outcome
	$1$ as a click) or network routing reliability (in the conjunctive form with outcome $0$ as the routing
	edge being broken).
It is straightforward to see that cascading bandits fit into the CMAB-T framework: $m$ variables are base arms,
	ordered sequences are actions, and the triggered set is the prefix set of arms until the stopping condition holds.

\section{\chgins{Triggering Probability Modulated Condition}}
	\label{sec:conditions}


	\citet{CWYW16} use two conditions to guarantee the theoretical regret bounds. The first one is monotonicity, 
		which we also use in this paper,
		and is restated below.

	\begin{mycond}[Monotonicity]
		\label{cond:monotone}
		We say that a CMAB-T problem instance satisfies {\em monotonicity}, 
		if for any action $S \in \calS$,  for any two distributions
		$D,D'\in \calD$ with expectation vectors
		$\vmu=(\mu_1, \ldots, \mu_m)$ and $\vmu' = (\mu'_1, \ldots, \mu'_m)$,
		we have
		$r_S(\vmu) \le r_S(\vmu')$ if $\mu_i \le \mu_i'$ for all $i\in [m]$.
	\end{mycond}
	
%
	The second condition is bounded smoothness. 
	One key contribution of our paper is to properly strengthen the original bounded smoothness condition in~\cite{CWYW16}
		so that we can both get rid of the undesired $1/p^*$ term in the regret bound and guarantee that 
		many CMAB problems 
		still satisfy the conditions.
	Our important change is to use triggering probabilities to modulate the condition, and thus we call such conditions
		{\em triggering probability modulated (TPM)} conditions.
	The key point of TPM conditions is including the triggering probability in the condition.
	We use $p_i^{D, S}$ to denote the probability that action $S$ triggers arm $i$ when the environment instance is $D$.
	With this definition, we can also technically define
		$p^*$ as $p^* = \inf_{i\in[m], S\in \calS, p_i^{D, S}>0} p_i^{D, S}$.
	In this section, we further use 1-norm based conditions instead of the infinity-norm based condition in~\cite{CWYW16},
		since they lead to better regret bounds for the influence maximization and cascading bandits.

\begin{mycond}[1-Norm TPM Bounded Smoothness]
	\label{cond:1-normTPM}
	We say that a CMAB-T problem instance satisfies 1-norm TPM bounded smoothness, if 
	there exists $B \in \R^{+}$ (referred as the {\em bounded smoothness constant}) such that,
	for any two distributions $D, D'\in \calD$ with expectation vectors $\vmu$ and $\vmu'$, and any action $S$,
	we have $|r_S(\vmu)-r_S(\vmu')|\le B \sum_{i\in[m]}p_i^{D, S}|\mu_i-\mu'_i|$.
\end{mycond}

Note that the corresponding non-TPM version of the above condition would remove $p_i^{D, S}$ in the 
	above condition, which is a generalization of the linear condition used in linear bandits~\cite{KWAS15}.
Thus, the TPM version is clearly stronger than the non-TPM version (when the bounded smoothness constants are the same).
The intuition of incorporating the triggering probability $ p_i^{D, S}$ to modulate the 1-norm condition
	is that, when an arm $i$ is unlikely triggered by action $S$ (small $ p_i^{D, S}$), the importance
	of arm $i$ also diminishes in that a large change in $\mu_i$ only causes a small change in the expected reward
	$r_S(\vmu)$.
This property sounds natural in many applications, and it is important for bandit learning ---
	although an arm $i$ may be difficult to observe when playing $S$, it is also not important to
	the expected reward of $S$ and thus does not need to be learned as accurately as others more easily triggered
	by $S$.

	\section{CUCB Algorithm and Regret Bound with TPM Bounded Smoothness} \label{sec:1-norm}
	We use the same CUCB algorithm as in \citep{CWYW16} (Algorithm~\ref{alg:cucb}).
	The algorithm maintains the empirical estimate $\hat \mu_{i}$ for the true mean $\mu_i$,
		and feed the upper confidence bound $\bar{\mu}_i$ to the offline oracle to obtain the next
		action $S$ to play.
	The upper confidence bound $\bar{\mu}_i$ is large if arm $i$ is not triggered often ($T_i$ is small),
		providing optimistic estimates for less observed arms.
	We next provide its regret bound.
		
	
	\renewcommand{\algorithmicrequire}{\textbf{Input:}}
	\begin{algorithm}[t]
		\centering
		\caption{CUCB with computation oracle.}
		\label{alg:cucb}
		\label{alg:oracleucb}
		\label{algorithmoracle}
		\begin{algorithmic}[1]
			\REQUIRE $m, \textsf{Oracle}$
			\STATE For each arm $i$, $T_{i}\leftarrow 0$
			\COMMENT{maintain the total number of times arm $i$ is played so far}
			\STATE For each arm $i$, $\hat \mu_{i}\leftarrow 1$
			\COMMENT{maintain the empirical mean of $X_i$}
			\FOR{$t=1, 2, 3, \ldots$}
			\STATE \label{line:Ri} 
				For each arm $i\in [m]$, $\rho_i \leftarrow \sqrt{\frac{3\ln t}{2T_i}}$
			\COMMENT{the confidence radius, $\rho_i=+\infty$ if $T_i=0$}
			\STATE For each arm $i\in [m]$, $\bar{\mu}_{i}=\min\left \{\hat{\mu}_{i} + \rho_i,1\right \}$
			\label{alg:adjustCUCB}
			\COMMENT{the upper confidence bound}
			\STATE $S \leftarrow \textsf{Oracle}(\bar{\mu}_{1}, \ldots, \bar{\mu}_{m})$
			\label{alg:selectsarm}
			\STATE Play action $S$, which triggers a set $\tau \subseteq [m]$ of base arms with
				feedback $X_i^{(t)}$'s, $i\in \tau $
			\STATE For every $i \in \tau $, update $T_i$ and $\hat{\mu}_{i}$: 
				$T_i = T_i + 1$, $\hat{\mu}_{i} = \hat{\mu}_{i} + (X_i^{(t)}-\hat{\mu}_{i})/T_i$
			\ENDFOR
		\end{algorithmic}
	\end{algorithm}
	
%
	

	
	
	
	\begin{mydef}[Gap]
		Fix a distribution $D$ and its expectation vector $\vmu$.
		For each action $S$, we define the gap $\Delta_S=\max(0, \alpha\cdot \opt_{\vmu}-r_S(\vmu))$.
		For each arm $i$, we define
		\begin{align*}
			\Delta_{\min}^i=\inf_{S\in \calS: p_i^{D, S}>0, \Delta_S>0} \Delta_S, \quad \quad
			\Delta_{\max}^i=\sup_{S\in \calS: p_i^{D, S}>0, \Delta_S>0} \Delta_S.
		\end{align*}
		As a convention, if there is no action $S$ such that $p_i^{D, S}>0$ and $\Delta_S>0$,
		we define $\Delta_{\min}^i=+\infty$, $\Delta_{\max}^i=0$.
		We define $\Delta_{\min}=\min_{i\in [m]} \Delta_{\min}^i$, and $\Delta_{\max}=\max_{i\in [m]} \Delta_{\max}^i$.
	\end{mydef}

	Let $\trig{S} = \{i\in [m]\mid p_i^{\vmu,S} > 0\}$ be the set of arms that could be triggered by $S$.
	Let $K=\max_{S\in \calS} |\trig{S}|$.
	For convenience, we use $\lceil x\rceil_0$ to denote $\max\{\lceil x\rceil, 0\}$ for any real number $x$.
	
%

	\begin{mythm}
		\label{thm:1-normTPM}
		For the CUCB algorithm on a CMAB-T problem instance that satisfies monotonicity (Condition~\ref{cond:monotone})
			and 1-norm TPM bounded smoothness (Condition~\ref{cond:1-normTPM})
			with bounded smoothness constant $B$,
	(1) if $\Delta_{\min} > 0$, we have distribution-dependent bound
			\begin{align}
			&Reg_{\vmu, \alpha, \beta}(T) \le \sum_{i\in[m]} \frac{576B^2K\ln T}{\Delta_{\min}^i} 
			+ \sum_{i\in[m]}\left(\left\lceil\log_2 \frac{2BK}{\Delta_{\min}^i}\right\rceil_0 + 2\right)
			\cdot \frac{\pi^2}{6} \cdot \Delta_{\max} +4Bm;
			\label{eq:1-normTPM.dep}
			\end{align}
	(2) we have distribution-independent bound
			\begin{align}
			&Reg_{\vmu, \alpha, \beta}(T) \le 12B\sqrt{mKT\ln T}
			+ \left(\left\lceil\log_2 \frac{T}{18\ln T}\right\rceil_0+2\right) \cdot m \cdot \frac{\pi^2}{6}\cdot \Delta_{\max}
			+2Bm.
			\label{eq:1-normTPM.ind}
			\end{align}
	\end{mythm}

For the above theorem, we remark that the regret bounds are tight (up to 
	a $O(\sqrt{\log T})$ factor in the case of distribution-independent bound)
	base on a lower bound result in~\cite{KWAS15}.
More specifically, ~\citet{KWAS15} show that for linear bandits (a special class of
	CMAB-T without probabilistic triggering), the distribution-dependent regret
	is lower bounded by $\Omega(\frac{(m-K)K}{\Delta}\log T)$, and the
	distribution-independent regret is lower bounded by
	$\Omega(\sqrt{mKT})$ when $T\ge m/K$, for some instance where
	$\Delta^i_{\min} = \Delta$ for all $i\in [m]$ and $\Delta^i_{\min} < \infty$.
Comparing with our regret upper bound in the above theorem, 
	(a) for distribution-dependent bound, we have the regret upper bound $O(\frac{(m-K)K}{\Delta}\log T)$ since for that instance $B=1$ and there are $K$ arms with $\Delta^i_{\min} =
	\infty$, so tight with the lower bound in~\cite{KWAS15}; 
	and (b) for distribution-independent bound, we have the regret upper bound
	$O(\sqrt{mKT \log T})$, tight to the lower bound up to a $O(\sqrt{\log T})$ factor,
	same as the upper bound for the linear bandits in~\cite{KWAS15}.
This indicates that parameters $m$ and $K$ appeared in the above regret bounds are
	all needed.
As for parameter $B$, we can view it simply as a scaling parameter.
If we scale the reward of an instance to $B$ times larger than before, certainly, the
	regret is $B$ times larger.
Looking at the distribution-dependent regret bound (Eq.~\eqref{eq:1-normTPM.dep}), 
	$\Delta^i_{\min}$ would also be scaled by a factor of $B$, canceling one $B$ factor
	from $B^2$, and $\Delta_{\max}$ is also scaled by a factor of $B$, 
	and thus the regret bound in Eq.~\eqref{eq:1-normTPM.dep} is also 
	scaled by a factor of $B$.
In the distribution-independent regret bound (Eq.~\eqref{eq:1-normTPM.ind}), the scaling
	of $B$ is more direct.
Therefore, we can see that all parameters $m$, $K$, and $B$ appearing in the
	above regret bounds are needed.
\chgins{Finally, we remark that the TPM Condition~\ref{cond:1-normTPM} can be refined such that $B$ is replaced by arm-dependent $B_i$ 
	that is moved inside the summation, and $B$ in Theorem~\ref{thm:1-normTPM} is replaced with $B_i$ accordingly. 
	\OnlyInFull{See Appendix~\ref{app:refineB} for details}\OnlyInShort{See the supplementary material for details}.}

\subsection{Novel Ideas in the Regret Analysis} \label{sec:proofideas}

Due to the space limit, the full proof of Theorem~\ref{thm:1-normTPM} is moved
to the supplementary material.
Here we briefly explain the novel aspects of our analysis that allow us to achieve new regret bounds
and differentiate us from previous analyses such as the ones in~\cite{CWYW16} and~\cite{kveton2015combinatorial,KWAS15}.

We first give an intuitive explanation on how to incorporate the TPM bounded smoothness condition to 
remove the factor $1/p^*$ in the regret bound.
Consider a simple illustrative example of two actions $S_0$ and $S$, where $S_0$ has a fixed reward $r_0$
as a reference action, 
and $S$ has a stochastic reward depending on the outcomes of its triggered base arms.
Let $\trig{S}$ be the set of arms that can be triggered by $S$.
For $i\in \trig{S}$, suppose $i$ can be triggered by action $S$ with probability $p_i^S$, and its true mean is $\mu_i$
and its empirical mean at the end of round $t$ is $\hat\mu_{i,t}$.
The analysis in~\cite{CWYW16} would need a property that, if for all
$i \in \trig{S}$ $|\hat{\mu}_{i,t} - \mu_i| \le \delta_{i}$
for some properly defined $\delta_{i}$, then $S$ no longer generates regrets.
The analysis would conclude that arm $i$ needs to be triggered $\Theta(\log T / \delta_i^2)$ times for the above 
condition to happen.
Since arm $i$ is only triggered with probability $p_i^S$, it means action $S$ may need to be played 
$\Theta(\log T / (p_i^S \delta_i^2))$ times.
This is the essential reason why the factor $1/p^*$ appears in the regret bound.

%
%
%

Now with the TPM bounded smoothness, we know that the impact of $|\hat{\mu}_{i,t} - \mu_i| \le \delta_{i}$ to
the difference in the expected reward is only $p_i^S \delta_{i}$, or equivalently, we could relax the requirement
to $|\hat{\mu}_{i,t} - \mu_i| \le \delta_{i}/p_i^S$ to achieve the same effect as in the previous analysis.
This translates to the result that action $S$ would generate regret in at most 
$O(\log T / (p_i^S (\delta_i/p_i^S)^2)) = O(p_i^S \log T/ \delta_i^2)$ rounds. 

We then need to handle the case when we have multiple actions that could trigger arm $i$.
The simple addition of $\sum_{S:p_i^S>0} p_i^S \log T/ \delta_i^2$ is not feasible since we may have exponentially or even infinitely
many such actions.
Instead, we introduce the key idea of {\em triggering probability groups}, such that the above actions are divided into
groups by putting their triggering probabilities $p_i^S$ into geometrically separated bins: $(1/2, 1], (1/4, 1/2]
\ldots, (2^{-j}, 2^{-j+1}], \ldots$
The actions in the same group would generate regret in at most $O(2^{-j+1} \log T/ \delta_i^2)$ rounds
with a similar argument, and summing up together,
they could generate regret in at most
$O(\sum_j 2^{-j+1} \log T/ \delta_i^2) = O(\log T/ \delta_i^2)$ rounds.
Therefore, the factor of $1/p_i^S$ or $1/p^*$ is completely removed from the regret bound.

Next, we briefly explain our idea to achieve the improved bound over the linear bandit result in~\cite{KWAS15}.
The key step is to bound regret $\Delta_{S_t}$ generated in round $t$.
By a derivation similar to \cite{KWAS15,CWYW16} together with the 1-norm TPM bounded smoothness condition, we would 
obtain that $\Delta_{S_t} \le B\sum_{i\in \trig{S_t}} p_i^{D, S_t}(\bar\mu_{i, t} - \mu_i)$ with high probability.
The analysis in \cite{KWAS15} would analyze the errors $|\bar\mu_{i, t} - \mu_i|$ by a cascade of infinitely many sub-cases
of whether there are $x_j$ arms with errors larger than $y_j$ with decreasing $y_j$,
but it may still be loose.
Instead we directly work on the above summation.
Naive bounding the about error summation would not give a $O(\log T)$ bound because there could be too many arms with small errors.
Our trick is to use a {\em reverse amortization}:  we cumulate small errors on many sufficiently sampled arms
and treat them as errors of insufficiently sample arms, such that an arm sampled $O(\log T)$ times would not
contribute toward the regret.
This trick tightens our analysis and leads to significantly improved constant factors.

\OnlyInFull{The reverse amortization trick can be seen in Appendix~\ref{app:notriggering}
	Eq.\eqref{eq:nontriggering.transform} and the derivation that follows for the no triggered arm case, 
	as well as in Appendix~\ref{app:1normTPM}, Eq.~\eqref{eq:TPMkappa.transform} in the proof of Lemma~\ref{lem:TPMkappa}
	for the 1-norm case.}

%

	
\subsection{Applications to Influence Maximization and Combinatorial Cascading Bandits}
\label{sec:app}

The following two lemmas show that both the cascading bandits and the influence maximization bandit
	satisfy the TPM condition.

\begin{mylem}
	\label{lem:cascading}
	For both disjunctive and conjunctive cascading bandit problem instances,
	1-norm TPM bounded smoothness (Condition~\ref{cond:1-normTPM}) holds with bounded smoothness
	constant 
	$B = 1$.
\end{mylem}
\begin{mylem}
	\label{lem:IMbandit}
	For the influence maximization bandit problem instances,
	1-norm TPM bounded smoothness (Condition~\ref{cond:1-normTPM}) holds with bounded smoothness constant $B = \tilde{C}$, where $\tilde{C}$ is the largest number of nodes any node can reach in
		the directed graph $G=(V,E)$.
\end{mylem}

The proof of Lemma~\ref{lem:cascading} involves a technique called {\em bottom-up modification}.
Each action in cascading bandits can be viewed as a chain from top to bottom.
When changing the means of  arms below, the triggering probability of arms above is not changed.
Thus, if we change $\vmu$ to $\vmu'$ backwards,
the triggering probability of each arm is unaffected before its expectation is changed,
	and when changing the mean of an arm $i$,
the expected reward of the action is at most changed by $p_i^{D, S} |\mu'_i-\mu_i|$.

The proof of Lemma~\ref{lem:IMbandit} is more complex, since
	the bottom-up modification does not work directly on graphs with cycles.
To circumvent this problem, we develop an {\em influence tree decomposition} technique as follows.
First, we order all influence paths from the seed set $S$ to a target $v$.
Second, each edge is independently sampled based on its edge probability to form a random {\em live-edge
	graph}.
Third, we divide the reward portion of activating $v$ among all paths from $S$ to $v$: for each live-edge
	graph $L$ in which $v$ is reachable from $S$, assign the probability of
	$L$ to the first path from $S$ to $v$ in $L$ according to the path total order.
Finally, we compose all the paths from $S$ to $v$ into a tree with $S$ as the root and copies of $v$
	as the leaves, so that we can do bottom-up modification
	on this tree and properly trace the reward changes based on the reward division we made among the paths.
%
%
%
%

\vspace{-2mm}
	\subsection{Discussions and Comparisons}
	\label{sec:discussion}
\vspace{-1mm}
	We now discuss the implications of Theorem~\ref{thm:1-normTPM} together with
	Lemmas~\ref{lem:cascading} and~\ref{lem:IMbandit} by comparing them with several existing results.
	
	\paragraph{Comparison with \cite{CWYW16} and CMAB with $\infty$-norm bounded smoothness conditions.}
	Our work is a direct adaption of the study in \cite{CWYW16}.
	Comparing with \cite{CWYW16}, we see that 
		the regret bounds in Theorem~\ref{thm:1-normTPM} 
		are not dependent on the inverse of triggering probabilities, which is the main issue
		in \cite{CWYW16}. 
	When applied to	influence maximization bandit, our result is strictly stronger than that of \cite{CWYW16} in two aspects:
		(a) we remove the factor of $1/p^*$ by using the TPM condition; 
		(b) we reduce a factor of $|E|$ and $\sqrt{|E|}$ in the dominant terms of distribution-dependent and -independent
			bounds, respectively, due to our use of 1-norm instead of $\infty$-norm conditions used
			in~\citet{CWYW16}.
	In the supplementary material, we further provide the corresponding $\infty$-norm TPM bounded 
		smoothness conditions and the regret bound results, since 
		in general the two sets of results do not imply each other.

%
	
	%
	
	\chgins{
	\paragraph{Comparison with \cite{Wen2016} on influence maximization bandits. }
	Conceptually, our work deals with the general CMAB-T framework with influence maximization and combinatorial cascading bandits
		as applications, while \citet{Wen2016} only work on influence maximization bandit.
	\citet{Wen2016} further study a generalization of linear transformation of edge probabilities, which is orthogonal to our
		current study, and could be potentially incorporated into the general CMAB-T framework.
	Technically, both studies eliminate the exponential factor $1/p^*$ in the regret bound.
	Comparing the rest terms in the regret bounds, our regret bound depends on a topology
		dependent term $\tilde{C}$ (Lemma~\ref{lem:IMbandit}), while their bound depends on a complicated term $C_*$, which is related to both topology and edge probabilities.
	Although in general it is hard to compare the regret bounds, for the several graph families for which \citet{Wen2016} provide concrete 
		topology-dependent regret bounds, our bounds are always better by a factor from 
		$O(\sqrt{k})$ to $O(|V|)$, where $k$ is the number of seeds selected in each round and $V$ is the node set
		in the graph.
	This indicates that, in terms of characterizing the topology effect on the regret bound, our simple complexity term $\tilde{C}$ is more effective
		than their complicated term $C_*$.
	\OnlyInFull{See Appendix~\ref{app:compareWen} for the detailed table of comparison.}\OnlyInShort{See the supplementary material for
		the detailed table of comparison.}
		}

	\OnlyInShort{%
		 \vspace{-0.5mm}%
	}
	\paragraph{Comparison with \cite{kveton2015combinatorial} on combinatorial cascading bandits }
	By Lemma~\ref{lem:cascading}, we can apply Theorem~\ref{thm:1-normTPM} to combinatorial 
		conjunctive and disjunctive cascading bandits with
		bounded smoothness constant $B=1$, 
		achieving $O(\sum\frac{1}{\Delta^i_{\min}}K \log T)$ distribution-dependent, and
		$O(\sqrt{mKT\log T})$ distribution-independent regret.
	In contrast, besides having exactly these terms, 
		the results in \cite{kveton2015combinatorial} have an extra factor of $1/f^*$, where 
		$f^* = \prod_{i \in S^*} p(i)$ for conjunctive cascades, and $f^* = \prod_{i \in S^*} (1- p(i))$ for disjunctive 
		cascades, with $S^*$ being the optimal solution and $p(i)$ being the probability of success for item (arm) $i$.
	For conjunctive cascades, $f^*$ could be reasonably close to $1$ in practice as argued in \cite{kveton2015combinatorial}, but
		for disjunctive cascades, $f^*$ could be exponentially small since items in optimal solutions typically have large $p(i)$ values.
	Therefore, our result completely removes the dependency on $1/f^*$ and is better than their result.
	Moreover, we also have much smaller constant factors owing to the new reverse amortization 
		method described in Section~\ref{sec:proofideas}.
	
	\OnlyInShort{%
		\vspace{-0.5mm}%
	}
	\paragraph{Comparison with \cite{KWAS15} on linear bandits. }
	When there is no probabilistically triggered arms (i.e. $p^*=1$), Theorem~\ref{thm:1-normTPM} would have
	tighter bounds since some analysis dealing with probabilistic triggering is not needed.
	In particular, in Eq.~\eqref{eq:1-normTPM.dep} the leading constant $624$ would be reduced to $48$, the
	$\lceil\log_2 x \rceil_0$ term is gone, and $6Bm$ becomes $2Bm$; in Eq.~\eqref{eq:1-normTPM.ind} the leading constant
	$50$ is reduced to $14$, and the other changes are the same as above (see the supplementary material).
	The result itself is also a new contribution, since it
		generalizes the linear bandit of \cite{KWAS15} to general 1-norm conditions
	with matching regret bounds, while significantly reducing the leading constants (their constants are $534$ and $47$ for distribution-dependent
	and independent bounds, respectively).
	This improvement comes from the new reversed amortization method described in Section~\ref{sec:proofideas}.

\vspace{-2mm}
	\section{Lower Bound of the General CMAB-T Model}
	\label{sec:lowerbound}
	
\vspace{-1mm}
	In this section, we show that there exists some CMAB-T problem instance such that the regret bound
		in \cite{CWYW16} is tight, i.e. the factor $1/p^*$ in the distribution-dependent bound and
		$\sqrt{1/p^*}$ in the distribution-independent bound are unavoidable, where $p^*$ is the minimum positive probability
		that any base arm $i$ is triggered by any action $S$.
	It also implies that the TPM bounded smoothness may not be applied to all CMAB-T instances.
	
	For our purpose, we only need a simplified version of the bounded smoothness condition of \cite{CWYW16} as below:
	There exists a bounded smoothness constant $B$ such that, for every action $S$ and every pair of mean outcome vectors $\vmu$ and $\vmu'$, we have $|r_S(\vmu) - r_S(\vmu')| \le B \max_{i\in \trig{S}} |\mu_i - \mu'_i|$, where $\trig{S}$ is the set 
		of arms that could possibly be triggered by $S$.

	We prove the lower bounds using the following CMAB-T problem instance $([m], \calS, \calD, \Dtrig, R)$.
	For each base arm $i \in [m]$, we define an action $S_i$, with the set of actions 
		$\calS = \{S_1, \ldots, S_m \}$.
	The family of distributions $\calD$ consists of distributions generated by
		every $\vmu\in [0, 1]^m$ such that the arms are independent Bernoulli variables.
	When playing action $S_i$ in round $t$, with a fixed probability $p$,  arm $i$ is triggered and its outcome
	$X_i^{(t)}$ is observed, and the reward of playing $S_i$ is $p^{-1} X_i^{(t)}$;
		otherwise with probability $1-p$ no arm is triggered, no feedback is observed and the reward is $0$.
	Following the CMAB-T framework, this means that $\Dtrig(S_i, X)$, as a distribution on the subsets of $[m]$,
		is either $\{i\}$ with probability $p$ or $\emptyset$ with probability $1-p$,
		and the reward $R(S_i, X, \tau) = p^{-1} X_i \cdot \I\{\tau = \{i\} \}$.
	The expected reward $r_{S_i}(\vmu) = \mu_i$.
	So this instance satisfies the above bounded smoothness with constant $B=1$.
	We denote the above instance as FTP($p$), standing for fixed triggering probability instance.
	This instance is similar with position-based model \citep{Lagree2016}
		with only one position,
		while the feedback is different.
	For the FTP($p$) instance, we have $p^*=p$ and $r_{S_i}(\vmu) = p \cdot p^{-1} \mu_i = \mu_i$.
	Then applying the result in \cite{CWYW16}, we have distributed-dependent upper bound
		$O(\sum_i \frac{1}{p\Delta^i_{\min}}\log T)$
	and distribution-independent upper bound $O(\sqrt{p^{-1} m T \log T})$.

	We first provide the distribution-independent lower bound result.	
	\begin{mythm} \label{thm:indlowerbound}
		Let $p$ be a real number with  $0< p < 1$.
		Then for any CMAB-T algorithm $A$, if $T \ge 6p^{-1}$,
			there exists a CMAB-T environment instance $D$ with mean $\vmu$ such that on instance FTP($p$),
		\begin{equation*}
		Reg_{\vmu}^A(T) \ge \frac{1}{170}\sqrt{\frac{m T}{p}}.
		\end{equation*}
	\end{mythm}
	
	The proof of the above and the next theorem are all based on the results for the classical MAB problems.
	Comparing to the upper bound $O(\sqrt{p^{-1} m T \log T})$.
		obtained from \cite{CWYW16},
	Theorem~\ref{thm:indlowerbound} implies that the regret upper bound of CUCB in \cite{CWYW16} 
		is tight up to a $O(\sqrt{\log T})$ factor.
	This means that the $1/p^*$ factor in the regret bound of \cite{CWYW16} cannot be avoided in the
		general class of CMAB-T problems.
		

	
	
	
	
	\newcommand{\none}{\perp}
	
	Next we give the distribution-dependent lower bound.
	For a learning algorithm, we say that it is {\em consistent} if, for every $\vmu$, every non-optimal arm is played $o(T^a)$ times in expectation, for any real number $a>0$.
	Then we have the following distribution-dependent lower bound.
	
	\begin{mythm}
		\label{thm:deplowerbound}
		For any consistent algorithm $A$ running on instance FTP($p$)
		and $\mu_i < 1$ for every arm $i$, we have
		$$\liminf_{T\to +\infty} \frac{Reg^A_{\vmu}(T)}{\ln T} \ge \sum_{i:\mu_i<\mu^*} \frac{p^{-1}\Delta_i}{\kl(\mu_i, \mu^*)},$$
		where $\mu^*=\max_i \mu_i$, $\Delta_i = \mu^* - \mu_i$, 
			and $\kl(\cdot, \cdot)$ is the Kullback-Leibler divergence function.
	\end{mythm}
%
	Again we see that the distribution-dependent upper bound obtained from \cite{CWYW16} asymptotically match the lower bound above.
	Finally, we remark that even if we rescale the reward from $[1, 1/p]$ 
	back to $[0,1]$, the corresponding scaling factor $B$ would become $p$, and thus
	we would still obtain the conclusion that the regret bounds
	in \cite{CWYW16} is tight (up to a $O(\sqrt{\log T})$ factor), and thus
	$1/p^*$ is in general needed in those bounds.
	

	\section{Conclusion and Future Work}
	In this paper, we propose the TPM bounded smoothness condition, which conveys
		the intuition that an arm difficult to trigger is also less important in
		determining the optimal solution. 
	We show that this condition
		is essential to guarantee low regret, and prove that important applications,
		such as influence maximization bandits and combinatorial cascading bandits
		all satisfy this condition.
		
	There are several directions one may further pursue.
	One is to improve the regret bound for some specific problems.
	For example, for the influence maximization bandit, can we give a better algorithm
		or analysis to achieve a better regret bound than the one provided by the 
		general TPM condition?
	Another direction is to look into other applications with probabilistically triggered
		arms that may not satisfy the TPM condition or need other conditions to 
		guarantee low regret.
	Combining the current CMAB-T framework with the linear generalization 
		as in~\cite{Wen2016} to achieve scalable learning result is also an interesting
		direction.
	
	\OnlyInShort{\clearpage}
	
	\section*{Acknowledgment}
	Wei Chen is partially supported by 
	  the National Natural Science Foundation of China (Grant No. 61433014).

    \def\bibfont{\small}
	\bibliography{bibdatabase}
    \bibliographystyle{plainnat}

	\ifthenelse{\equal{\compileincludeappendix}{true}}{
	\onecolumn
	\appendix

\section*{Supplementary Materials}

\wei{I did not see the use of symbols $p_i^{\vmu,*}$ and $p^{\vmu,*}$, so removed them from the main text.
	If they are used in the appendix, please include their definitions in the appendix.}

\section{Model Discussions}

\qinshi{Here is a problem while formalization. The set of arms that are triggered should be known by the player.
	When the triggering probability depends on the arms, this set may reveal information about the expectations of arms.
	We should make some restriction that the set itself does not reveal more information than the feedback on the arms.
	But currently I feel my argument is too complicated.}
\wei{We will address the above issue, that is, the general distribution $\tau(S,X)$ may allow extra information
	leading about the base arms, here.}

\subsection{Comparison with the framework of~\cite{CWYW16}}

The CMAB-T framework described above essentially follows the framework of~\cite{CWYW16}, but with the following
noticeable differences.
First, we refer to $S$ as an abstract action from an action space $\calS$, while in~\cite{CWYW16}, $S$ is referred
to as a super arm, which is a subset of base arms $[m]$.
In the case of CMAB without probabilistically triggered arms, we can simply let every super arm $S$ be
an action, and $\tau(S,X) = S$, meaning that playing super arm $S$ deterministically triggers all and only
base arms in $S\subseteq [m]$.
Second, we explicitly allows action space to be infinite or even continuous space, while in~\cite{CWYW16}, the action
space is the subsets of base arms and thus is finite.
We will see later that the infinite action space does not make essential difference in the analysis.
Third, for probabilistically triggered arms, we explicitly use $\tau(S,X)$ to model them, 
and allows $\tau(S,X)$ to have additional randomness besides the randomness of $X$.
In~\cite{CWYW16}, probabilistic triggering is explained as further base arms being triggered based on the outcomes
of previously triggered base arms, and to model certain triggering structure or
additional randomness in triggering an arm, dummy base arms need to
be added.
However, this may require introducing a large number of dummy base arms.
For example, for the cascading bandits, to specify the order of the cascade sequence, we need to add 
dummy base arms corresponding to every possible order of the base arms.
Moreover, $\tau(S,X)$ cleanly separates the randomness known to the player from the unknown randomness from the environment
outcome.
For example, in the discount-based continuous influence maximization~\cite{YangMPH16}, 
$\tau(c,X)$ includes the randomness of activating the seed set from the discount vector $c$ given by $\eta_i$'s, 
which are known to the player.
In contrast, the distribution of $X_{(u,v)}$, namely probability $p(u,v)$ on edges are unknown and need to be learned.
In this case, if we use dummy base arms to model such additional triggering behavior from marketing actions to seed sets, 
these dummy base arms will be mixed together with edge base arms for which the learning algorithm need to learn,
unless further distinction is made.

Therefore, we believe that our current adaptation CMAB-T provides a cleaner framework and is more easily to be applied
to various problem instances.
We remark that all the analysis and results in~\cite{CWYW16} remain unchanged with our current adaptation.

\subsection{Modeling general marketing actions in influence maximization}

\wei{We will have some discussion here, as an illustration of the usefulness of our general triggering function $\Dtrig$.}

Note that we can also use randomized $\tau(S,X)$ to model some extended versions of influence maximization.
For example, general marketing actions are proposed in~\cite{kempe03} and continuous discount actions
are proposed in~\cite{YangMPH16}, both allowing activating seed nodes with a probability depending on
the marketing intensity on the node.
In particular, an action in the discount-based continuous influence maximization in~\cite{YangMPH16} is a vector
$c=(c_1, c_2, \ldots, c_n)$, where $c_i\in[0,1]$ is the discount to be given to node $i$.
Discount $c_i$ is translated to probability $\eta_i(c_i)$ that node $i$ is activated as a seed, where $\eta_i(\cdot)$ is a 
monotonically non-decreasing function with $\eta_i(0)=0$ and $\eta_i(1)=1$.
In this case, the probabilistic triggering function $\tau(c, X)$ includes the randomness from $c$ to seed activations based on $\eta_i$'s, 
beyond the randomness of $X$.
That is, even when $c$ and $X$ are fixed, $\tau(c, X)$ is still a random set.
We further remark that in this case, the action space of all discount vectors is a continuous and infinite space, which
is allowed in our adapted CMAB-T model.

\section{Main Regret Analysis (Proofs Related to Theorem~\ref{thm:1-normTPM})}

\subsection{Basics of CMAB-T problems}
\qinshi{(2.22) I realize that the concepts of $\Ns$ and $\Nt$ even does not depend on CUCB algorithm. They are nature of CMAB-T.} 

We utilize the following well known tail bound in our analysis.

\begin{myfact}[Hoeffding's Inequality \citep{hoeffding63}]\label{fact:hoeffding}
	Let $X_1, \cdots , X_n$ be independent and identically distributed random variables with
	common support $[0,1]$ and mean $\mu$. 
	Let
	$Y=X_1+\cdots+X_n$. Then for all $\delta \geq0$,
	\[\Pr\{|Y-n\mu|\geq \delta \} \leq 2e^{-2\delta^2/n}.
	\]
\end{myfact}

\begin{myfact}[Multiplicative Chernoff Bound \citep{MU05}]\footnote{The result in the book by \cite{MU05} 
		(Theorem 4.5 together with Exercise 4.7) only covers the case where random variables $X_i$'s are independent. 
		However the result can be easily generalized to our case with an almost identical proof.
		The only main change is to replace $\E\left[e^{{t(\sum_{j=1}^{i-1} X_j + X_i)}}\right]=
		\E\left[e^{{t\sum_{j=1}^{i-1} X_j}}\right]\E\left[e^{{t X_i}}\right]$
		with $\E\left[e^{t(\sum_{j=1}^{i-1} X_j + X_i)}\right] = 
		\E\left[e^{t\sum_{j=1}^{i-1} X_j} \E\left[e^{tX_i}\mid X_1,\ldots, X_{i-1}\right]\right]$.}
	\label{fact:chernoff}
	Let $X_1, \cdots , X_n$ be Bernoulli random variables taking values from $\{0,1\}$, 
	and $\E[X_t|X_1,\cdots, X_{t-1}]\geq\mu$ for every $t\leq n$. 
	Let $Y=X_1+\cdots+X_n$. 
	Then for all $0<\delta <1$,
	\[\Pr\{Y\leq (1-\delta)n\mu\}
	\leq e^{-\frac{\delta^2n\mu }{2}}.
	\]
\end{myfact}


We introduce the following definition to assist our analysis.
\begin{mydef}[Event-Filtered Regret] \label{def:eventregret}
	For any series of events $\{\calE_t\}_{t\ge 1}$ indexed by round number $t$,  
	we define $Reg^A_{\vmu, \alpha}(T, \{\calE_t\}_{t\ge 1})$
	as the regret filtered by events $\{\calE_t\}_{t\ge 1}$, that is, regret is only counted in round $t$ if $\calE_t$ happens in round $t$.
	Formally,
	$$Reg^A_{\vmu, \alpha}(T, \{\calE_t\}_{t\ge 1}) = \E\left[\sum_{t=1}^T \I(\calE_t)(\alpha\cdot \opt_{\vmu}-r_{\vmu}(S_t^A))\right].$$
	For convenience, $A$, $\alpha$, $\vmu$ and/or $T$ can be omitted when the context is clear, and we simply use $Reg^A_{\vmu, \alpha}(T, \calE_t)$ instead of
	$Reg^A_{\vmu, \alpha}(T, \{\calE_t\}_{t\ge 1})$.
\end{mydef}

The following definition describes an unlikely event that $\hat\mu_{i,t-1}$ is not as accurate as expected.
\begin{mydef} \label{def:ns}
	We say that the \emph{sampling is nice} at the beginning of round $t$
	if for every arm $i \in [m]$, $|\hat\mu_{i,t-1} - \mu_i|<\rho_{i, t}$,
	where $\rho_{i, t}=\sqrt{\frac{3\ln t}{2T_{i, t-1}}}$ in round $t$.
	Let $\Ns_t$ be such event.
\end{mydef}

\begin{mylem}
	\label{lem:ns}
	For each round $t \ge 1$, $\Pr\{\lnot \Ns_t\}\le 2mt^{-2}$.
\end{mylem}

\begin{proof}
	For each round $t \ge 1$, we have
	\begin{align}
	\Pr\{\lnot \Ns_t\} & = \Pr\left\{\exists i\in [m], |\hat\mu_{i,t-1} - \mu_i| \ge \sqrt{\frac{3\ln t}{2T_{i, t-1}}}\right\} \nonumber \\
	& \le\sum_{i\in [m]} \Pr\left\{|\hat\mu_{i,t-1} - \mu_i| \ge \sqrt{\frac{3\ln t}{2T_{i, t-1}}}\right\}. \nonumber \\
	& =\sum_{i\in [m]} \sum_{k=1}^{t-1}\Pr\left\{T_{i,t-1}=k, |\hat\mu_{i,t-1} - \mu_i| \ge \sqrt{\frac{3\ln t}{2T_{i, t-1}}}\right\}.  \label{eq:nicesample1}
	\end{align}
	When $T_{i,t-1}=k$, $\hat\mu_{i,t-1}$ is the average of $k$ i.i.d.\ random variables
	$X_i^{[1]}, \dots, X_i^{[k]}$, where $X_i^{[j]}$ is the outcome of arm $i$ when it is triggered for the $j$-th time during the execution.
	That is, $\hat\mu_{i,t-1} = \sum_{j=1}^k X_i^{[j]} / k$.
	Then we have
	\begin{align}
	\Pr\left\{T_{i,t-1}=k, |\hat\mu_{i,t-1} - \mu_i| \ge \sqrt{\frac{3\ln t}{2T_{i, t-1}}}\right\}
	& = \Pr\left\{ T_{i,t-1}=k, \left|\sum_{j=1}^k X_i^{[j]} / k - \mu_i\right| \ge \sqrt{\frac{3\ln t}{2k}} \right \} \nonumber \\
	& \le \Pr\left\{ \left|\sum_{j=1}^k X_i^{[j]} - k \mu_i\right| \ge \sqrt{\frac{3k\ln t}{2}} \right \} \le 2 t^{-3},
	\label{eq:nicesample2}
	\end{align}
	where the last inequality uses the Hoeffding's Inequality (Fact~\ref{fact:hoeffding}).
	Combining Inequalities~\eqref{eq:nicesample1} and~\eqref{eq:nicesample2}, we thus prove the lemma.
\end{proof}

\begin{mydef}[Triggering probability (TP) group]
	\label{def:tpgroup}
	Let $i$ be an arm and $j$ be a positive natural number, define the triggering probability group (of actions)
	$$\calS_{i, j}^{D}=\{S\in \calS\mid 2^{-j}< p_i^{D, S}\le 2^{-j+1}\}.$$
	Notice $\{\calS_{i, j}^{D}\}_{j \ge 1}$
	forms a partition of $\{S\in \calS\mid p_i^{D, S}>0\}$.
\end{mydef}

\begin{mydef}[Counter] \label{def:counter}
	For each TP group $\calS_{i, j}$, we define a corresponding counter $N_{i,j}$.
	In a run of a learning algorithm, the counters are maintained in the following manner.
	All the counters are initialized to $0$. In each round $t$, if the action $S_t$ is chosen,
	then update $N_{i, j}$ to $N_{i, j}+1$ for every $(i, j)$ that $S_t \in \calS_{i, j}^{D}$.
	Denote $N_{i, j}$ at the end of round $t$ with $N_{i, j, t}$.
	In other words, we can define the counters with the recursive equation below:
	$$N_{i, j, t} = \begin{cases}
	0,                 &\mbox{if }t=0\\
	N_{i, j, t-1} + 1, &\mbox{if }t>0, S_t\in \calS_{i, j}^{D}\\
	N_{i, j, t-1},     &\mbox{otherwise}.
	\end{cases}$$
\end{mydef}

\begin{mydef} \label{def:nt}
	Given a series of integers $\{j_{\max}^i\}_{i\in[m]}$, we say that the \emph{triggering is nice} at the beginning of round $t$ (with respect to $j_{\max}^i$),
	if for every TP group (Definition~\ref{def:tpgroup}) identified by arm $i$ and $1 \le j \le j_{\max}^i$, as long as $\sqrt{\frac{6\ln t}{\frac{1}{3}N_{i, j, t-1}\cdot 2^{-j}}}\le 1$, there is $T_{i, t-1} \ge \frac{1}{3}N_{i, j, t-1}\cdot 2^{-j}$.
	We denote this event with $\Nt_t$. It implies
	\begin{equation*}
	\rho_{i, t} = \sqrt{\frac{3\ln t}{2T_{i, t-1}}} \le \sqrt{\frac{3\ln t}{2\cdot\frac{1}{3}N_{i, j, t-1}\cdot 2^{-j}}}.
	\end{equation*}
\end{mydef}

\begin{mylem} \label{lem:nt}
	For a series of integers $\{j_{\max}^i\}_{i\in[m]}$, $\Pr\{\lnot\Nt_t\} \le \sum_{i\in[m]}j_{\max}^it^{-2}$ for every round $t\ge 1$.
\end{mylem}

\begin{proof}
	We prove this lemma by showing $\Pr\{N_{i, j, t-1} = s, T_{i, t-1} \le \frac{1}{3} N_{i, j, t-1} \cdot 2^{-j}\} \le t^{-3}$, 
	for any fixed $s$ with $0\le s \le t-1$ and $\sqrt{\frac{6\ln t}{\frac{1}{3} \cdot s\cdot 2^{-j}}} \le 1$.
	Let $t_k$ be the round that $N_{i, j}$ is increased for the $k$-th time,
	for $1\le k \le s$.
	Let $Y_k = \I\{i\in \tau_{t_k}\}$ be a Bernoulli variable, that is, $i$ is triggered in round $t_k$.
	When fixing the action $S_{t_k}$, $Y_k$ is independent from $Y_1, \dots, Y_{k-1}$.
	Since $S_{t_k}\in \calS_{i, j}$, $\E[Y_k \mid Y_1, \dots, Y_{k-1}]\ge 2^{-j}$.
	Let $Z=Y_1 + \dots + Y_s$. By multiplicative Chernoff bound (Fact~\ref{fact:chernoff}), we have
	\[\Pr\left\{Z \le \frac{1}{3}s\cdot 2^{-j}\right\} \le \exp\left(\frac{{-}\left(\frac{2}{3}\right)^2 s \cdot 2^{-j}}{2}\right) 
	\le \exp\left(\frac{{-}\left(\frac{2}{3}\right)^2 18\ln t}{2}\right) < \exp(-3\ln t) = t^{-3}.\]
	By the definition of $T_{i,t-1}$ and the condition $N_{i, j, t-1} = s$, we have $T_{i, t-1}\ge Z$. 
	Thus
	\begin{align*}
	& \Pr\{N_{i, j, t-1} = s, T_{i, t-1} \le \frac{1}{3} N_{i, j, t-1} \cdot 2^{-j}\} \\
	& \le \Pr\{N_{i, j, t-1} = s, Z \le \frac{1}{3} s \cdot 2^{-j}\} \\
	& \le \Pr\{Z \le \frac{1}{3} s \cdot 2^{-j}\} \\
	& \le t^{-3}.
	\end{align*}
	By taking $i$ over $[m]$, $j$ over $1, \dots, j_{\max}^i$, $s$ over $0, \dots, t-1$,
	the lemma holds.
\end{proof}

\subsection{The Case of No Probabilistically Triggered Arms} \label{app:notriggering}

In this section, we state and prove a theorem for the case of no probabilistically triggered arms, i.e. $p^*=1$, when
the CMAB-T instance satisfies the 1-norm (non-TPM) bounded smoothness condition below.

\begin{mycond}[1-Norm Bounded Smoothness]
	\label{cond:1-norm}
	We say that a CMAB-T problem instance satisfies 1-norm bounded smoothness, if 
	there exists a bounded smoothness constant $B \in \R^{+}$ such that,
	for any two distributions $D, D'\in \calD$ with expectation vectors $\vmu$ and $\vmu'$, and any action $S$,
	we have $|r_S(\vmu)-r_S(\vmu')|\le B \sum_{i\in \trig{S}}|\mu_i-\mu'_i|$, where $\trig{S}$ 
	is the set of arms that are triggered by $S$.
\end{mycond}

As discussed in the main text, this theorem provides better bounds than Theorem~\ref{thm:1-normTPM} with probabilistically triggered arms.
Its proof is also simpler, so the readers could choose to either get oneself familiar with the analysis with this proof first, or directly
jump to the next section for the proof of Theorem~\ref{thm:1-normTPM}.

\begin{mythm}
	\label{thm:1-norm.nontriggering}
	For the CUCB algorithm on a CMAB (without triggering, i.e. $p^*=1$) problem
	that satisfies 1-norm bounded smoothness (Condition~\ref{cond:1-norm}) with bounded smoothness
	constant $B$,
	\begin{enumerate}
		\item if $\Delta_{\min} > 0$, we have distribution-dependent bound
		\begin{align} \label{eq:1-norm.nontriggering1}
		Reg_{\vmu, \alpha, \beta}(T) \le
		\sum_{i\in[m]}\frac{48B^2K\ln T}{\Delta_{\min}^i} + 2Bm
		+ \frac{\pi^2}{3}\cdot m\cdot \Delta_{\max};
		\end{align}
		\item we have distribution-independent bound
		\begin{align}
		Reg_{\vmu, \alpha, \beta}(T) \le
		14B\sqrt{KmT\ln T} + 2Bm
		+ \frac{\pi^2}{3}\cdot m\cdot \Delta_{\max};
		\end{align}
	\end{enumerate}
\end{mythm}

\begin{proof} [Proof of Theorem~\ref{thm:1-norm.nontriggering}]
	To unify the proofs for distribution-dependent and distribution-independent bounds,
	we introduce a positive real number $M_i$ for each arm $i$.
	Let $\calF_t$ be the event $\{r_{S_t}(\bar\vmu) < \alpha\cdot \opt(\bar\vmu)\}$.
	In other words, $\calF_t$ means the  oracle fails in round $t$.
	By assumption, $\Pr\{\calF_t\} \le 1-\beta$.
	Define $M_S = \max_{i\in \trig{S}} M_i$ for each action $S$,
	specifically, $M_S = 0$ if $\trig{S} = \emptyset$.
	Define
	$$\kappa_{T}(M, s) = \begin{cases}
	2B, &\mbox{if } s=0,\\
	2B\sqrt{\frac{6 \ln T}{s}}, &\mbox{if } 1\le s\le \ell_{T}(M),\\
	0, &\mbox{if } s \ge \ell_{T}(M)+1,
	\end{cases}$$
	where
	$$\ell_{T}(M)=\left\lfloor\frac{24 B^2 K^2 \ln T}{M^2}\right\rfloor.$$
	We then show that if $\{\Delta_{S_t} \ge M_{S_t}\}$, $\lnot \calF_t$ and $\Ns_t$ hold, we have
	\begin{equation}
	\Delta_{S_t} \le \sum_{i\in \trig{S_t}}\kappa_T(M_i, T_{i, t-1}).
	\label{eq:1-norm.nontriggering.kappa}
	\end{equation}
	The right hand side of the inequality is non-negative, so it holds naturally if $\Delta_{S_t}=0$. We only need to consider $\Delta_{S_t}>0$.
	By $\Ns_t$ and $\lnot \calF_t$, we have
	$$r_{S_t}(\bar\vmu_t)\ge \alpha\cdot \opt(\bar\vmu_t)\ge \alpha\cdot \opt(\vmu) = r_{S_t}(\vmu) + \Delta_{S_t},$$
	Then by Condition~\ref{cond:1-norm},
	$$\Delta_{S_t}\le r_{S_t}(\bar\vmu_t)-r_{S_t}(\vmu)\le B\sum_{i\in \trig{S_t}} (\bar\mu_{i, t} - \mu_i).$$
	We are going to bound $\Delta_{S_t}$ by bounding $\bar\mu_{i, t} - \mu_i$. 
	But before doing so, we first perform a transformation.
	As we have $\Delta_{S_t} \ge M_{S_t}$, so
	$B\sum_{i\in \trig{S_t}} (\bar\mu_{i, t} - \mu_i)\ge \Delta_{S_t} \ge M_{S_t}$.
	We have
	\begin{align}
	\Delta_{S_t}
	&\le B\sum_{i\in \trig{S_t}} (\bar\mu_{i, t} - \mu_i)\nonumber\\
	&\le -M_{S_t} + 2B\sum_{i\in \trig{S_t}} (\bar\mu_{i, t} - \mu_i)\nonumber\\
	&= 2B\sum_{i\in \trig{S_t}} \left[(\bar\mu_{i, t} - \mu_i) - \frac{M_{S_t}}{2B\left|\trig{S_t}\right|}\right] \nonumber\\
	&\le 2B\sum_{i\in \trig{S_t}} \left[(\bar\mu_{i, t} - \mu_i) - \frac{M_{S_t}}{2BK}\right]\nonumber\\
	&\le 2B\sum_{i\in \trig{S_t}} \left[(\bar\mu_{i, t} - \mu_i) - \frac{M_i}{2BK}\right]. \label{eq:nontriggering.transform}
	\end{align}
	By $\Ns_t$, we have $\bar\mu_{i, t} - \mu_i \le \min\{2\rho_{i, t}, 1\}$.
	So
	$$(\bar\mu_{i, t} - \mu_i) - \frac{M_i}{2BK} \le \min\{2\rho_{i, t}, 1\} - \frac{M_i}{2BK} \le \min\left\{2\sqrt{\frac{3\ln T}{2T_{i, t-1}}}, 1\right\} - \frac{M_i}{2BK}.$$
	If $T_{i, t-1}\le \ell_T(M_i)$, we have $(\bar\mu_{i, t} - \mu_i) - \frac{M_i}{2BK} \le \min\left\{2\sqrt{\frac{3\ln T}{2T_{i, t-1}}}, 1\right\} \le \frac{1}{2B}\kappa_T(M_i, T_{i, t-1})$.
	If $T_{i, t-1}\ge \ell_T(M_i)+1$, then $2\sqrt{\frac{3\ln T}{2T_{i,t-1}}} \le \frac{M_i}{2BK}$, so $(\bar\mu_{i, t} - \mu_i) - \frac{M_i}{2BK} \le 0 = \frac{1}{2B}\kappa_T(M_i, T_{i, t-1})$.
	In conclusion, we continue \eqref{eq:nontriggering.transform} with
	$$\eqref{eq:nontriggering.transform} \le \sum_{i\in \trig{S_t}} \kappa_T(M_i, T_{i, t-1}).$$
	
	Then in each run,
	\begin{align*}
	\sum_{t=1}^T \I(\{\Delta_{S_t} \ge M_{S_t}\} \land \lnot \calF_t \land \Ns_t) \cdot \Delta_{S_t}
	&\le \sum_{t=1}^T \sum_{i\in \trig{S_t}} \kappa_T(M_i, T_{i, t-1})\\
	&= \sum_{i\in[m]} \sum_{s=0}^{T_{i, T}} \kappa_T(M_i, s)\\
	&\le \sum_{i\in[m]} \sum_{s=0}^{\ell_{T}(M_i)} \kappa_T(M_i, s)\\
	&= 2Bm + \sum_{i\in[m]} \sum_{s=1}^{\ell_{T}(M_i)} 2B\sqrt{\frac{6\ln T}{s}}\\
	&\le 2Bm + \sum_{i\in[m]} \int_{s=0}^{\ell_{T}(M_i)} 2B\sqrt{\frac{6\ln T}{s}} \mathrm{d} s\\
	&\le 2Bm + \sum_{i\in[m]} 4B\sqrt{6\ln T \ell_T(M_i)}\\
	&\le 2Bm + \sum_{i\in[m]} 4B\sqrt{6\ln T \cdot \frac{24 B^2 K^2 \ln T}{M_i^2}}\\
	&\le 2Bm + \sum_{i\in[m]} \frac{48B^2K\ln T}{M_i}.
	\end{align*}
	So
	\begin{align*}
	Reg(\{\Delta_{S_t} \ge M_{S_t}\} \land \lnot \calF_t \land \Ns_t)
	& = \E\left[\sum_{t=1}^T \I(\{\Delta_{S_t} \ge M_{S_t}\} \land \lnot \calF_t \land \Ns_t) \cdot \Delta_{S_t}\right] \\
	& \le 2Bm + \sum_{i\in[m]} \frac{48B^2K\ln T}{M_i}.
	\end{align*}
	
	By Lemma~\ref{lem:ns}, $\Pr\{\lnot \Ns_t\} \le 2mt^{-2}$.
	Then, as $Reg(\calE_t) \le \sum_{t=1}^T \Pr\{\calE_t\} \Delta_{\max}$ by definition of filtered regret,
	$$Reg(\lnot \Ns_t) \le \sum_{t=1}^T 2mt^{-2} \cdot \Delta_{\max} \le \frac{\pi^2}{3} m\cdot \Delta_{\max},$$
	$$Reg(\calF_t) \le (1-\beta)T \cdot \Delta_{\max}.$$
	The filtered regret with null event
	\begin{align*}
	Reg(\{\})
	&\le Reg(\lnot \Ns_t) + Reg(\calF_t) + Reg(\Delta_{S_t} < M_{S_t})
	+ Reg(\{\Delta_{S_t} \ge M_{S_t}\} \land \lnot \calF_t \land \Ns_t)\\
	&\le (1-\beta)T \cdot \Delta_{\max} + \frac{\pi^2}{3} m\cdot \Delta_{\max}
	+ 2Bm + \sum_{i\in[m]} \frac{48B^2K\ln T}{M_i} + Reg(\Delta_{S_t} < M_{S_t}).
	\end{align*}
	
	By definition of filtered regret, $Reg_{\vmu, \alpha, \beta}(T) = Reg(T, \{\}) - (1-\beta)T\cdot \Delta_{\max}$, so
	$$Reg_{\vmu, \alpha, \beta}(T) \le \frac{\pi^2}{3}m \cdot \Delta_{\max}
	+ 2Bm + \sum_{i\in[m]} \frac{48B^2K\ln T}{M_i} + Reg(\Delta_{S_t} < M_{S_t}).$$
	For distribution-dependent bound, take $M_i=\Delta_{\min}^i$,
	then $Reg(\Delta_{S_t} < M_{S_t}) = 0$ and we have
	$$Reg_{\vmu, \alpha, \beta}(T) \le \sum_{i\in[m]} \frac{48B^2K\ln T}{\Delta_{\min}^i} + 2Bm + \frac{\pi^2}{3} \cdot \Delta_{\max}.$$
	For distribution-independent bound, take $M_i=M=\sqrt{(48B^2mK\ln T)/T}$,
	then $Reg(\Delta_{S_t} < M_{S_t}) \le TM$ and we have
	\begin{align*}
	Reg_{\vmu, \alpha, \beta}(T)
	&\le \sum_{i\in[m]} \frac{48B^2K\ln T}{M_i} + 2Bm + \frac{\pi^2}{3}m \cdot \Delta_{\max} + Reg(\Delta_{S_t} < M_{S_t})\\
	&\le \frac{48B^2mK\ln T}{M} + 2Bm + \frac{\pi^2}{3}m \cdot \Delta_{\max} + TM\\
	&= 2\sqrt{48B^2mKT\ln T} + \frac{\pi^2}{3}m \cdot \Delta_{\max} + 2Bm\\
	&\le 14B\sqrt{mKT\ln T} + \frac{\pi^2}{3}m \cdot \Delta_{\max} + 2Bm.
	\end{align*}
\end{proof}

\subsection{Proof of Theorem~\ref{thm:1-normTPM} (1-Norm Case Regret Bound)} \label{app:1normTPM}

We first show the distribution-dependent upper bound (Eq. \eqref{eq:1-normTPM.dep}) and the distribution-independent upper bound below, which is a weaker version of Eq. \eqref{eq:1-normTPM.ind}:
\begin{align}
&Reg_{\vmu, \alpha, \beta}(T) \le 48B\sqrt{mKT\ln T} 
+ \left(\left\lceil\log_2\sqrt{\frac{KT}{288m\ln T}}\right\rceil_0+2\right)
\cdot m\cdot \frac{\pi^2}{6}\cdot \Delta_{\max} + 4Bm.
\label{eq:1-normTPM.weakind}
\end{align}
We show full proof of Eq. \eqref{eq:1-normTPM.ind} later in Section \ref{sec:improveind}.
The proof of Eq.~\eqref{eq:1-normTPM.weakind} is based on the
	distribution-dependent bound (Eq. \eqref{eq:1-normTPM.dep}) similar to
	other analysis, and thus could be
	more familiar to readers and easier to follow, while
	Eq.~\eqref{eq:1-normTPM.ind} has better constant and requires an independent
	proof as given in Section \ref{sec:improveind}.

Let $\calF_t$ be the event $\{r_{S_t}(\bar\vmu) < \alpha\cdot \opt(\bar\vmu)\}$.
In other words, $\calF_t$ means the  oracle fails in round $t$.
By assumption, $\Pr\{\calF_t\} \le 1-\beta$.

To unify the proofs for distribution-dependent and distribution-independent bounds,
we introduce a positive real number $M_i$ for each arm $i$.
Define $M_S = \max_{i\in \trig{S}} M_i$ for each action $S$,
specifically, $M_S = 0$ if $\trig{S} = \emptyset$.
To prove the distribution-dependent bound, we will let $M_i=\Delta_{\min}^i$.
To prove the distribution-independent bound, we will let $M_i=M=\tilde{\Theta}(T^{-1/2})$
to balance bounds for $Reg(\{\Delta_{S_t} \ge M_{S_t}\}$ and $Reg(\{\Delta_{S_t} < M_{S_t}\})$.
Implement definition of $\Nt_t$ (Definition~\ref{def:nt}) with
$j_{\max}^i = j_{\max}(M_i) = \left\lceil\log_2 \frac{2BK}{M_i}\right\rceil_0$.
Define
$$\kappa_{j, T}(M, s) = \begin{cases}
4\cdot 2^{-j}B, &\mbox{if } s=0,\\
2B\sqrt{\frac{72\cdot 2^{-j} \ln T}{s}}, &\mbox{if } 1\le s\le \ell_{j, T}(M),\\
0, &\mbox{if } s \ge \ell_{j, T}(M)+1,
\end{cases}$$
where
$$\ell_{j, T}(M)=\left\lfloor\frac{288\cdot 2^{-j} B^2 K^2 \ln T}{M^2}\right\rfloor,$$
and the following lemma explains that $\kappa$ is the contribution to regret.

\begin{mylem} \label{lem:TPMkappa}
	In every run of the CUCB algorithm on a problem instance that satisfies 1-norm TPM bounded smoothness (Condition~\ref{cond:1-normTPM}), for any vector $\{M_i\}_{i\in[m]}$ of positive real numbers and $1\le t \le T$, if $\{\Delta_{S_t} \ge M_{S_t}\}, \lnot \calF_t, \Ns_t$ and $\Nt_t$ hold,
	we have
	$$\Delta_{S_t} \le \sum_{i\in \trig{S_t}} \kappa_{j_i, T}(M_i, N_{i, j_i, t-1}),$$
	where $j_i$ is the index of the TP group with $S_t\in \calS_{i, j_i}$ (See Definition~\ref{def:tpgroup}).
\end{mylem}

\begin{proof} 
	The right hand side of the inequality is non-negative, so it holds naturally if $\Delta_{S_t}=0$. We only need to consider $\Delta_{S_t}>0$. 
	By $\Ns_t$ and $\lnot \calF_t$, we have
	$$r_{S_t}(\bar\vmu_t)\ge \alpha\cdot \opt(\bar\vmu_t)\ge \alpha\cdot \opt(\vmu) = r_{S_t}(\vmu) + \Delta_{S_t},$$
	Then by Condition~\ref{cond:1-normTPM},
	\begin{equation}
	\Delta_{S_t}\le r_{S_t}(\bar\vmu_t)-r_{S_t}(\vmu)\le B\sum_{i\in \trig{S_t}} p_i^{D, S_t}(\bar\mu_{i, t} - \mu_i).\label{eq:TPMkappa.applycondition}
	\end{equation}
	We are going to bound $\Delta_{S_t}$ by bounding $p_i^{D, S_t}(\bar\mu_{i, t} - \mu_i)$. 
	But before doing so, we first perform a transformation.
	As we have $\Delta_{S_t} \ge M_{S_t}$, so
	$B\sum_{i\in \trig{S_t}} p_i^{D, S_t}(\bar\mu_{i, t} - \mu_i)\ge \Delta_{S_t} \ge M_{S_t}$.
	We have
	\begin{align}
	\Delta_{S_t}
	&\le B\sum_{i\in \trig{S_t}} p_i^{D, S_t}(\bar\mu_{i, t} - \mu_i)\nonumber\\
	&\le -M_{S_t} + 2B\sum_{i\in \trig{S_t}} p_i^{D, S_t}(\bar\mu_{i, t} - \mu_i)\nonumber\\
	&= 2B\sum_{i\in \trig{S_t}} \left[p_i^{D, S_t}(\bar\mu_{i, t} - \mu_i) - \frac{M_{S_t}}{2B\left|\trig{S_t}\right|}\right] \nonumber\\
	&\le 2B\sum_{i\in \trig{S_t}} \left[p_i^{D, S_t}(\bar\mu_{i, t} - \mu_i) - \frac{M_{S_t}}{2BK}\right]\nonumber\\
	&\le 2B\sum_{i\in \trig{S_t}} \left[p_i^{D, S_t}(\bar\mu_{i, t} - \mu_i) - \frac{M_i}{2BK}\right]. \label{eq:TPMkappa.transform}
	\end{align}
	Then we bound $p_i^{D, S_t}(\bar\mu_{i, t} - \mu_i)$.
	By $\Ns_t$,
	$$\bar\mu_{i, t} - \mu_i < 2\rho_{i, t} = 2\sqrt{\frac{3\ln t}{2T_{i, t-1}}}.$$
	Both $\bar\mu_{i, t}$ and $\mu_i$ are in $[0, 1]$, so $\bar\mu_{i, t} - \mu_i \le 1$.
	We then bound $p_i^{D, S_t} (\bar\mu_{i, t} - \mu_i)$ in different cases.
	
	
	\begin{itemize}
		
		\item {\em Case I: $1 \le j_i \le j_{\max}^i$.}
		Then we have $p_i^{D, S_t} \le 2\cdot 2^{-j_i}$.
		If $\sqrt{\frac{6\ln t}{\frac{1}{3}N_{i, j_i, t-1}\cdot 2^{-j_i}}}\le 1$, by $\Nt_t$,
		$$\bar\mu_{i, t} - \mu_i \le 2\sqrt{\frac{3\ln t}{2T_{i, t-1}}} \le \sqrt{\frac{6\ln t}{\frac{1}{3}N_{i, j_i, t-1}\cdot 2^{-j_i}}},$$
		so
		$$\bar\mu_{i, t} - \mu_i \le \min\left\{\sqrt{\frac{6\ln t}{\frac{1}{3}N_{i, j_i, t-1}\cdot 2^{-j_i}}}, 1\right\},$$
		and
		\begin{align*}
		&p_i^{D, S_t} (\bar\mu_{i, t} - \mu_i)\\
		\le\:& 2\cdot 2^{-j_i}\cdot \min\left\{\sqrt{\frac{6\ln t}{\frac{1}{3}N_{i, j_i, t-1}\cdot 2^{-j_i}}}, 1\right\}\\
		=\:& \min\left\{\sqrt{\frac{72\cdot 2^{-j_i}\ln T}{N_{i, j_i, t-1}}}, 2\cdot 2^{-j_i}\right\}.
		\end{align*}
	If $N_{i, j_i, t-1} \ge \ell_{j_i, T}(M_i) + 1$,
	then $\sqrt{\frac{72\cdot 2^{-j_i}\ln T}{N_{i, j_i, t-1}}} \le \frac{M_i}{2BK}$ and
	$p_i^{D, S_t}(\bar\mu_{i, t} - \mu_i) - \frac{M_i}{2BK} \le 0$.
	If $N_{i, j_i, t-1}=0$, we use the bound
	$p_i^{D, S_t}(\bar\mu_{i, t} - \mu_i) \le 2\cdot 2^{-j_i}$.
	Otherwise, i.e. $1\le N_{i, j_i, t-1} \le \ell_{j_i, T}(M_i)$,
	we use $p_i^{D, S_t}(\bar\mu_{i, t} - \mu_i) \le \sqrt{\frac{72\cdot 2^{-j_i}\ln T}{N_{i, j_i, t-1}}}$.
	Recall the definition of $\kappa_{j, T}(M, s)$,
	then, for $1\le j_i\le j_{\max}^i$, we have
	\begin{equation}
	p_i^{D, S_t}(\bar\mu_{i, t} - \mu_i) - \frac{M_i}{2BK} \le \frac{1}{2B} \kappa_{j_i, T}(M_i, N_{i, j_i, t-1}).
	\label{eq:TPMkappa.cases12}
	\end{equation}
		\item {\em Case II: $j_i \ge j_{\max}^i+1 = \left\lceil\log_2 \frac{2BK}{M_i}\right\rceil_0 + 1$.}
		Then we have
		\begin{align*}
		p_i^{D, S_t} (\bar\mu_{i, t} - \mu_i) &\le p_i^{D, S_t} \le 2\cdot 2^{-j_i}\\
		&\le 2\cdot 2^{-\log_2 \frac{2BK}{M_i}-1} = \frac{M_i}{2BK}.
		\end{align*}
		So
		\begin{equation}
		p_i^{D, S_t}(\bar\mu_{i, t} - \mu_i) - \frac{M_i}{2BK} \le 0 \le \frac{1}{2B} \kappa_{j_i, T}(M_i, N_{i, j_i, t-1}).
		\label{eq:TPMkappa.case3}
		\end{equation}
		Combining Eq. \eqref{eq:TPMkappa.transform}, \eqref{eq:TPMkappa.cases12} and \eqref{eq:TPMkappa.case3}, we conclude the proof with
		\[\Delta_{S_t} \le 2B\sum_{i\in \trig{S_t}} \left[p_i^{D, S_t}(\bar\mu_{i, t} - \mu_i) - \frac{M_i}{2BK}\right]\\
		\le \sum_{i\in \trig{S_t}} \kappa_{j_i, T}(M_i, N_{i, j_i, t-1}).\]
	\end{itemize}
\end{proof}

We remark that the proof of Lemma~\ref{lem:TPMkappa}, in particular the derivation leading to Eq.~\eqref{eq:TPMkappa.transform}
together with the argument in the paragraph before Eq.\eqref{eq:TPMkappa.cases12},	
contains the reverse amortization trick we mentioned in the main text.
In particular, by the derivation of Eq.~\eqref{eq:TPMkappa.transform}, the contribution of every arm $i$ to regret
$\Delta_{S_t}$ is accounted as $2B \left[p_i^{D, S_t}(\bar\mu_{i, t} - \mu_i) - \frac{M_i}{2BK}\right]$.
Then by the argument in the paragraph before Eq.\eqref{eq:TPMkappa.cases12}, if $N_{i, j_i, t-1} \ge \ell_{j_i, T}(M_i) + 1$,
meaning that $i$ has been triggered by actions in group $j_i$ for at least $\ell_{j_i, T}(M_i) + 1$, its error
$|\bar\mu_{i, t} - \mu_i|$ would be small enough such that its contribution to the regret $\Delta_{S_t}$ is not positive.
This trick eliminates the need of summing up small errors from many sufficiently sampled arms, 
leading to a tighter regret bound.
The same trick can be seen in Appendix~\ref{app:notriggering},
Eq.\eqref{eq:nontriggering.transform} and the derivation that follows for the no triggered arm case.


\qinshi{Do we capitalize starting letter of each word in titles of definitions?}

\qinshi{$+6am$ is outside the summation.}

\begin{mylem} \label{lem:1-norm.insuf}
	For the CUCB algorithm on a problem instance that satisfies TPM bounded smoothness with 1-norm (Condition~\ref{cond:1-normTPM}),
	$$Reg(\{\Delta_{S_t} \ge M_{S_t}\} \land \lnot \calF_t \land \Ns_t \land \Nt_t)
	\le \sum_{i\in [m]} \frac{576B^2K\ln T}{M_i} + 4Bm.$$
\end{mylem}

\begin{proof}
	We bound $Reg(\{\Delta_{S_t} \ge M_{S_t}\} \land \lnot \calF_t \land \Ns_t \land \Nt_t)$
	with Lemma~\ref{lem:TPMkappa}.
	In every run,
	\begin{align}
	\sum_{t=1}^T \I(\{\Delta_{S_t} \ge M_{S_t}\} \land \lnot \calF_t \land \Ns_t \land \Nt_t) \Delta_{S_t}
	&\le \sum_{t=1}^T \sum_{i\in \trig{S_t}} \kappa_{j_i, T}(M_i, N_{i, j_i, t-1})\nonumber\\
	&= \sum_{i\in [m]} \sum_{j=1}^{+\infty} \sum_{s=0}^{N_{i, j, T}-1} \kappa_{j, T}(M_i, s),\label{eq:1-norm.byN}
	\end{align}
	where \eqref{eq:1-norm.byN} is due to $N_{i, j_i}$ is increased if and only if $i\in \trig{S_t}$.
	For every arm $i$ and $j\ge 1$,
	\begin{align}
	\sum_{s=0}^{N_{i, j, T}-1} \kappa_{j, T}(M_i, s)
	&\le \sum_{s=0}^{\ell_{j, T}(M_i)} \kappa_{j, T}(M_i, s)\label{eq:1-norm.replacewithell}\\
	&= \kappa_{j, T}(M_i, 0) + \sum_{s=1}^{\ell_{j, T}(M_i)} \kappa_{j, T}(M_i, s)\nonumber\\
	&= \kappa_{j, T}(M_i, 0) + \sum_{s=1}^{\ell_{j, T}(M_i)} 2B\sqrt{\frac{72\cdot 2^{-j_i} \ln T}{s}}\nonumber\\
	&\le \kappa_{j, T}(M_i, 0) + 4B\sqrt{72\cdot 2^{-j_i}\ln T}
	\sqrt{\ell_{j, T}(M_i)},\label{eq:1-norm.integral}
	\end{align}
	where\eqref{eq:1-norm.replacewithell} is due to $\kappa_{j, T}(s) = 0$ when $s \ge \ell_{j, T}(M)+1$,
	and \eqref{eq:1-norm.integral} is due to the fact that, for every natural number integer $n$,
	$$\sum_{s=1}^n \sqrt{\frac{1}{s}} \le \int_{s=0}^n \sqrt{\frac{1}{s}}\ \mathrm{d}s
	= 2\sqrt{n}.$$
	
	By definition, $\ell_{j, T}(M_i) \le \frac{288\cdot 2^{-j_i} B^2 K^2 \ln T}{M_i^2}$, so
	\begin{align*}
	\eqref{eq:1-norm.integral}
	&\le \kappa_{j, T}(M, 0) + 4B\sqrt{72\cdot 2^{-j_i}\ln T}\sqrt{\frac{288\cdot 2^{-j_i} B^2 K^2 \ln T}{M_i^2}}\\
	&= 4\cdot 2^{-j}B + \frac{576\cdot 2^{-j_i}B^2K\ln T}{M_i}.
	\end{align*}
	Then we continue \eqref{eq:1-norm.byN} with
	\begin{align*}
	\eqref{eq:1-norm.byN}
	&\le \sum_{i\in [m]} \sum_{j=1}^{+\infty} \left(4\cdot 2^{-j}B + \frac{576\cdot 2^{-j_i}B^2K\ln T}{M_i}\right)\\
	&= \sum_{i\in [m]}\left[\left(4B + \frac{576B^2K\ln T}{M_i}\right)\cdot \sum_{j=1}^{+\infty} 2^{-j}\right]\\
	&= \sum_{i\in [m]}\left(4B + \frac{576B^2K\ln T}{M_i}\right)\\
	&= \sum_{i\in [m]}\frac{576B^2K\ln T}{M_i} + 4Bm.
	\end{align*}
	By taking expectation over all possible runs,
	\begin{align*}
	Reg(\{\Delta_{S_t} \ge M_{S_t}\} \land \lnot \calF_t \land \Ns_t \land \Nt_t)
	& = \E[\I(\{\Delta_{S_t} \ge M\} \land \lnot \calF_t \land \Ns_t \land \Nt_t) \Delta_{S_t}] \\
	& \le \sum_{i\in [m]}\frac{576B^2K\ln T}{M_i} + 4Bm.
	\end{align*}
\end{proof}

\begin{proof}[Proof of Theorem~\ref{thm:1-normTPM}]
	Recall Definition~\ref{def:eventregret}, the definition of event-filtered regret:
	$$Reg^A_{\vmu}(T, \{\calE_t\}_{t\ge 1}) = \E\left[\sum_{t=1}^T \I(\calE_t)(\alpha \cdot
	\opt_{\vmu}-r_{S_t^A}(\vmu))\right]
	= T\cdot \alpha \cdot \opt_{\vmu} - \E\left[\sum_{t=1}^T \I(\calE_t)(r_{S_t^A}(\vmu))\right].$$
	Then for filtered regret with null event (the event that is always true), 
	we have $Reg(\{\}) = Reg_{\vmu,\alpha,\beta} + (1-\beta) T \cdot \alpha \cdot \opt_{\vmu}$.
	We divide this filtered regret into parts as
	\begin{align}
	Reg(\{\}) & \le Reg(\{\Delta_{S_t} < M_{S_t}\}) + Reg(\calF_t) + Reg(\lnot \Ns_t) + Reg(\lnot \Nt_t) 
		\nonumber \\
		& \quad \quad + Reg(\{\Delta_{S_t} \ge M_{S_t}\} \land \lnot \calF_t \land \Ns_t \land \Nt_t).
	\label{eq:1-norm.totalreg}
	\end{align}
	
	By definition of filtered regret, $Reg(\calE_t) \le \sum_{t=1}^T \I\{\calE_t\} \Delta_{S_t} \le \sum_{t=1}^T \Pr\{\calE_t\} \cdot \Delta_{\max}$,
	then
	\begin{align}
	Reg(\calF_t) &\le \sum_{t=1}^T \Pr\{\calF_t\} \Delta_{\max}
	= (1-\beta)T \cdot \Delta_{\max},
	\label{eq:1-normTPM.calF}\\
	Reg(\lnot \Ns_t) &\le \sum_{t=1}^T \Pr\{\lnot \Ns_t\} \Delta_{\max}
	\le \frac{\pi^2}{3} \cdot m \cdot \Delta_{\max},
	\label{eq:1-normTPM.ns}\\
	Reg(\lnot \Nt_t) &\le \sum_{t=1}^T \Pr\{\lnot \Nt_t\} \Delta_{\max}
	\le \frac{\pi^2}{6} \cdot \sum_{i\in[m]}j_{\max}^i \cdot \Delta_{\max}.
	\label{eq:1-normTPM.nt}
	\end{align}
	By Lemma~\ref{lem:1-norm.insuf},
	$$Reg(\{\Delta_{S_t} \ge M_{S_t}\} \land \lnot \calF_t \land \Ns_t \land \Nt_t)
	\le \sum_{i\in[m]} \frac{576B^2K\ln T}{M_i} + 4Bm.$$
	
	Take $M_i=\Delta_{\min}^i$. If $\Delta_{S_t} < M_{S_t}$, then $\Delta_{S_t}=0$,
	since we have either $\trig{S_t} = \emptyset$ or $\Delta_{S_t} < M_{S_t} \le M_i$ for some $i\in \trig{S_t}$. So $Reg(\{\Delta_{S_t} < M_{S_t}\}) = 0$.
	Then we have
	\begin{equation}
	Reg(\{\}) \le (1-\beta)T \cdot \Delta_{\max} + \sum_{i\in[m]} \frac{576B^2K\ln T}{\Delta_{\min}^i} + 4Bm
	+ \frac{\pi^2}{6} \cdot \sum_{i\in[m]}\left(j_{\max}(\Delta_{\min}^i) + 2\right) \cdot \Delta_{\max},
	\label{eq:1-normTPM.nulleventdep}
	\end{equation}
	where we abuse the notation of $j_{\max}(M) = \left\lceil\log_2 \frac{2BK}{M_i}\right\rceil_0$.
	
	On the other hand, take $M_i=M=\sqrt{(576B^2mK\ln T)/T}$, then $\Delta_{S_t}$ is also $M$ for every action $S_t$ that $\trig{S_t}$ is non-empty.
	We bound $Reg(\{\Delta_{S_t} < M_\})$ with
	$$Reg(\{\Delta_{S_t} < M_{S_t}\}) = \sum_{t=1}^T \I\{\Delta_{S_t} < M_{S_t}\} \Delta_{S_t}
	\le \sum_{t=1}^T \I\{\Delta_{S_t} < M_{S_t}\} M \le TM.$$
	So the filtered regret with null event is bounded by
	\begin{align}
	Reg(\{\}) &\le (1-\beta)T \cdot \Delta_{\max} + \frac{576B^2mK\ln T}{M} + 4Bm + TM
	+ \frac{\pi^2}{6} \cdot \left(j_{\max}(M) + 2\right) \cdot m \cdot \Delta_{\max} \nonumber\\
	&= (1-\beta)T \cdot \Delta_{\max} + \frac{576B^2mK\ln T}{\sqrt{(576B^2mK\ln T)/T}} + 4Bm + T\sqrt{(576B^2mK\ln T)/T} \nonumber\\
	&\qquad+ \frac{\pi^2}{6} \cdot \left(j_{\max}(M) + 2\right) \cdot m \cdot \Delta_{\max} \nonumber\\
	&\le (1-\beta)T \cdot \Delta_{\max} + 48B\sqrt{mKT\ln T} + 4Bm
	+ \frac{\pi^2}{6} \cdot \left(j_{\max}(M) + 2\right) \cdot m \cdot \Delta_{\max}.
	\label{eq:1-normTPM.nulleventind}
	\end{align}
	
	Since $Reg_{\vmu, \alpha, \beta} = Reg(\{\}) - (1-\beta)T \cdot \alpha \cdot \opt_{\vmu} \le Reg(\{\}) - (1-\beta)T \cdot \Delta_{\max}$, \eqref{eq:1-normTPM.nulleventdep} implies \eqref{eq:1-normTPM.dep} and \eqref{eq:1-normTPM.nulleventind} implies \eqref{eq:1-normTPM.weakind}.
\end{proof}

\subsubsection{Further improvement on distribution-independent upper bound}
\label{sec:improveind}

We now prove the tighter distribution-independent bound (Eq.~\eqref{eq:1-normTPM.ind})
	without going through distribution-dependent bound.
We start with 
\begin{equation}
\Delta_{S_t}\le B\sum_{i\in \trig{S_t}} p_i^{D, S_t}(\bar\mu_{i, t} - \mu_i)\le B\sum_{i\in \trig{S_t}}p_i^{D, S_t}\min\left\{1, 2\sqrt{\frac{3\ln T}{2T_{i, t-1}}}\right\},
\tag{\ref{eq:TPMkappa.applycondition}}
\end{equation}
when events $\lnot \calF_t$ and $\Ns_t$ are true.
Use $j_{\max}=\left\lceil\log_2\frac{T}{18\ln T}\right\rceil_0$ to define $\Nt_t$.
When $\Nt_t$,
	$\sqrt{\frac{3\ln T}{2T_{i, t-1}}}\le \sqrt{\frac{18\cdot 2^{-j_i}\ln T}{N_{i, j_i, t-1}}}$
	if $j_i\le j_{\max}$ by definition of $\Nt_t$,
	then
	$p_i^{D, S_t}\min\left\{1, 2\sqrt{\frac{3\ln T}{2T_{i, t-1}}}\right\}\le \min\left\{2^{-j_i+1}, \sqrt{\frac{72\cdot 2^{-j_i}\ln T}{N_{i, j_i, t-1}}}\right\}$
	as $p_i^{D, S_t} \le 2^{-j_i+1}$.
If $j_i>j_{\max}$, we still have $p_i^{D, S_t}\le 2^{-j_i+1}$. Because $N_{i, j_i, t-1}<T$, we have $2^{j_i+1}\ge \sqrt{\frac{72\cdot 2^{-j_i}\ln T}{N_{i, j_i, t-1}}}$.
The conclusion is
\begin{equation}
p_i^{D, S_t}\min\left\{1, 2\sqrt{\frac{3\ln T}{2T_{i, t-1}}}\right\}\le \min\left\{2^{-j_i+1}, \sqrt{\frac{72\cdot 2^{-j_i}\ln T}{N_{i, j_i, t-1}}}\right\} \label{eq:1-norm2.kappa}
\end{equation}
always holds,
	regardless $j\le j_{\max}$ or $j>j_{\max}$.
So we define $\kappa$ as following in this proof:
\[\kappa_{j, T}(s)=\min\left\{2B\cdot 2^{-j}, B\sqrt{\frac{72\cdot 2^{-j}\ln T}{s}}\right\}.\]

According to \eqref{eq:TPMkappa.applycondition} and \eqref{eq:1-norm2.kappa},
\begin{align}
Reg(\lnot \calF_t \land \Ns_t \land \Nt_t)
&\le \sum_{t=1}^T \I(\lnot \calF_t \land \Ns_t \land \Nt_t) \Delta_{S_t}\nonumber\\
&\le \sum_{t=1}^T \sum_{i\in \trig{S_t}} \kappa_{j_i, T}(N_{i, j_i, t-1})\nonumber\\
&= \sum_{i\in [m]} \sum_{j=1}^{+\infty} \sum_{s=0}^{N_{i, j, T}-1} \kappa_{j, T}(s).\label{eq:1-norm2.byN}
\end{align}
In each round, there are at most $K$ of the counters $\{N_{i, j}\}_{i\in [m], j\in \mathbb{N}^+}$ are increased by 1,
	so $\sum_{i\in [m]}\sum_{j=1}^{+\infty} N_{i, j, T}\le KT$.
To maximize the right hand side of $\eqref{eq:1-norm2.byN}$
	is to choose $KT$ largest elements from the multiset $\{\kappa_{j, T}(s)\}_{i\in [m], j\in \mathbb{N}^+, s\in \mathbb{N}}$,
	consider the continuous version below which is more tractable than finding $KT$ largest elements:
\begin{align}
\sum_{i\in [m]} \sum_{j=1}^{+\infty} \sum_{s=0}^{N_{i, j, T}-1} \kappa_{j, T}(s)
&\le \sum_{i\in [m]}\sum_{j=1}^{+\infty}\left(\kappa_{j, T}(0)+\sum_{s=1}^{\max\{0, N_{i, j, T}-1\}}\kappa_{j, T(s)}\right)\nonumber\\
&\le 2Bm + \sum_{i\in [m]} \sum_{j=1}^{+\infty} \int_{s=0}^{N_{i, j, T}} \kappa_{j, T}(s)\mathrm{d}s\nonumber\\
&\le 2Bm + \max_{\sum_{i, j}x_{i, j} \le KT} \left[ \sum_{i\in [m]} \sum_{j=1}^{+\infty} \int_{s=0}^{x_{i, j}}  B\sqrt{\frac{72\cdot 2^{-j}\ln T}{s}}\mathrm{d}s \right].
\label{eq:1-norm2.sumofint}
\end{align}
To maximize the above sum of integral, we must have $B\sqrt{\frac{72\cdot 2^{-j}\ln T}{x_{i, j}}}=B\sqrt{\frac{72\cdot 2^{-j'}\ln T}{x_{i', j'}}}$ for every $i,i'\in m$, $j, j'\in \mathbb{N}^+$.
The solution is $x_{i, j} = 2^{-j}KT/m$.
By taking the solution into \eqref{eq:1-norm2.sumofint}, we have
\begin{align}
\eqref{eq:1-norm2.sumofint}
&= 2Bm + \sum_{i\in [m]} \sum_{j=1}^{+\infty} \int_{s=0}^{2^{-j}KT/m}  B\sqrt{\frac{72\cdot 2^{-j}\ln T}{s}}\mathrm{d}s\nonumber\\
&= 2Bm + \sum_{i\in [m]} \sum_{j=1}^{+\infty} B\sqrt{144\cdot 2^{-j}\cdot 2^{-j}KT\ln T/m}\nonumber\\
&= 2Bm + 12B\sqrt{mKT\ln T}.
\end{align}
Combining with Lemmas \ref{lem:ns} \& \ref{lem:nt}, we have
\[Reg(\{\})\le (1-\beta)T\cdot \Delta_{\max} + 12B\sqrt{mKT\ln T}
+ \left(\left\lceil\log_2 \frac{T}{18\ln T}\right\rceil_0+2\right) \cdot m \cdot \frac{\pi^2}{6}\cdot \Delta_{\max}
+2Bm,\]
implying \eqref{eq:1-normTPM.ind}.

\subsection{Refining Parameter $B$} \label{app:refineB}
We can refine 1-norm bounded smoothness (Condition~\ref{cond:1-norm}) by replacing the parameter $B$ with a separate parameter $B_i$ for each arm $i$.

\begin{mycond}[Refined 1-Norm TPM Bounded Smoothness]
	\label{cond:refined.1-normTPM}
	We say that a CMAB-T problem instance satisfies refined 1-norm TPM bounded smoothness, if 
	there exists $B_i\in \R^{+}$ for every arm $i$ (referred as the {\em bounded smoothness constant}) such that,
	for any two distributions $D, D'\in \calD$ with expectation vectors $\vmu$ and $\vmu'$, and any action $S$,
	we have $|r_S(\vmu)-r_S(\vmu')|\le \sum_{i\in[m]}B_ip_i^{D, S}|\mu_i-\mu'_i|$.
\end{mycond}

Then in Theorem~\ref{thm:1-normTPM}, 
	we may replace $B$ with $B_i$ in distribution-dependent bound and replace $B\sqrt{m}$ with $\sqrt{\sum_{i\in[m]}B_i^2}$ 
	in distribution-independent bound, \chgins{except that for the last constant term we replace $Bm$ with $\sum_{i\in [m]} B_i$}.
More specifically, we have
(1) if $\Delta_{\min} > 0$, we have distribution-dependent bound
	\begin{align}
	&Reg_{\vmu, \alpha, \beta}(T) \le \sum_{i\in[m]} \frac{576B_i^2K\ln T}{\Delta_{\min}^i} 
	+ \sum_{i\in[m]}\left(\left\lceil\log_2 \frac{2B_iK}{\Delta_{\min}^i}\right\rceil_0 + 2\right)
	\cdot \frac{\pi^2}{6} \cdot \Delta_{\max} +4\sum_{i\in [m]}B_i;
	\end{align}
(2) we have distribution-independent bound
	\begin{align}
	&Reg_{\vmu, \alpha, \beta}(T) \le 12\sqrt{\sum_{i\in [m]}B_i^2KT\ln T}
	+ \left(\left\lceil\log_2 \frac{T}{18\ln T}\right\rceil_0+2\right) \cdot m \cdot \frac{\pi^2}{6}\cdot \Delta_{\max}
	+2\sum_{i\in [m]}B_i.
	\end{align}

The proof of this refinement is almost straightforward replacement of $B$ with $B_i$, except a few points that we want to highlight.
The definition of $\kappa$ and $\ell$ will be
$$\kappa_{i, j, T}(M, s) = \begin{cases}
4\cdot 2^{-j}B_i, &\mbox{if } s=0,\\
2B_i\sqrt{\frac{72\cdot 2^{-j} \ln T}{s}}, &\mbox{if } 1\le s\le \ell_{i, j, T}(M),\\
0, &\mbox{if } s \ge \ell_{i, j, T}(M)+1,
\end{cases}$$
where  \wei{Should above $\ell_{j, T}(M)$ be $\ell_{i, j, T}(M)$? I made the change}
$$\ell_{i, j, T}(M)=\left\lfloor\frac{288\cdot 2^{-j} B_i^2 K^2 \ln T}{M^2}\right\rfloor.$$
To maximize the sum of integral in \eqref{eq:1-norm2.sumofint} \chgins{(with $B$ replaced by $B_i$)},
	we need $B_i\sqrt{\frac{72\cdot 2^{-j}\ln T}{x_{i,j}}}=B_{i'}\sqrt{\frac{72\cdot 2^{-j'}\ln T}{x_{i',j'}}}$ for every $i,i'\in [m]$ and $j,j'\in \mathbb{N}^+$.
So $x_{i, j}\propto 2^{-j}B_i^2$,
and then $x_{i, j}=2^{-j}B_i^2KT/\sum_{i\in[m]}B_i^2$.

\section{Proofs for Applications of CMAB-T (Lemmas~\ref{lem:cascading} and~\ref{lem:IMbandit} in Section~\ref{sec:app})}
\subsection{Proof of Lemma~\ref{lem:cascading}}
\label{sec:proof.cascading}
\begin{proof}
	Let $S$ be an action. We regard $S$ as a permutation of $k$ of the arms.
	Without loss of generality, we may assume $S=(1,\dots,k)$ for some $k\le K$.
	For convenience, we use $p_i^{\vmu, S}$ instead of $p_i^{D, S}$,
		as arms are independent Bernoulli variables so that $D$ can be determined by $\vmu$.
	For an arm $i > k$, $i$ will not be triggered by action $S$, and thus
	$p_i^{\vmu, S} = 0$.
	The reward also does not depend on those arms.
	So we may only consider the arms $1, \dots, k$.
	For convenience, we only list the expectations of arms in $S$, so that $\vmu=(\mu_1, \dots, \mu_k)$ and
	$\vmu'=(\mu_1', \dots, \mu_k')$.
	
	Informally speaking, we can change the expectation of the arms from $\mu_i$ to $\mu_i'$, 
		in the reverse order from $k$ to $1$. 
	Changing the expectation of an arm $j$ does not affect the triggering probability
		of an arm $i$ ordered in front of $j$, i.e. $i < j$.
	And when changing an arm from $\mu_i$ to $\mu'_i$,
		the reward changes by at most $p_i^{\vmu,S}|\mu_i-\mu'_i|$.
	Therefore the total difference of reward is at most $\sum_{i=1}^k p_i^{\vmu,S}|\mu_i-\mu'_i|$.
	
	Formally, for the conjunctive cascading bandit, $r_S(\vmu)=\prod_{j=1}^k \mu_j$, and $p_i^{\vmu, S}=\prod_{j=1}^{i-1}\mu_j$ for $i=1,\dots, k$.
	For every $j =0, 1, \ldots, k$, let
	$$\vmu^{(j)}=(\mu_1, \dots, \mu_j, \mu'_{j+1}, \dots, \mu'_k),$$
	specifically, $\vmu^{(k)}=\vmu$, $\vmu^{(0)}=\vmu'$.
	Then,
	\begin{align*}
	\left|r_S(\vmu^{(j)})-r_S(\vmu^{(j-1)})\right|
	&= \left|\prod_{i=1}^k \mu^{(j)}_i-\prod_{i=1}^k \mu^{(j-1)}_i\right|\\
	&= \prod_{i, i\ne j} \mu^{(j)}_i \left|\mu^{(j)}_j - \mu^{(j-1)}_j\right|\\
	&\le \prod_{i=1}^{j-1} \mu^{(j)}_i \left|\mu^{(j)}_j - \mu^{(j-1)}_j\right|\\
	&= \prod_{i=1}^{j-1} \mu_i \left|\mu_j - \mu'_j\right|\\
	&= p_j^{\vmu, S}\left|\mu_j - \mu'_j\right|,
	\end{align*}
	\begin{align*}
	\left|r_S(\vmu)-r_S(\vmu')\right|
	&= \left|r_S(\vmu^{(k)})-r_S(\vmu^{(0)})\right|\\
	&\le \sum_{j=1}^k \left|r_S(\vmu^{(j)})-r_S(\vmu^{(j-1)})\right|\\
	&\le \sum_{j=1}^k p_j^{\vmu, S}\left|\mu_j - \mu'_j\right|.
	\end{align*}
	For the disjunctive case, let $\lambda_i = 1-\mu_i$ for $i\in [m]$.
	Then we have
	$r_S(\vmu) = 1 - \prod_{j=1}^{k} \lambda_i$, and $p_i^{\vmu,S} = \prod_{j=1}^{i-1} \lambda_j$.
	The rest analysis follows the same pattern as the conjunctive case.
\end{proof}

\subsection{Proof of Lemma~\ref{lem:IMbandit}}

\newcommand{\vx}{\ensuremath{\boldsymbol x}}
\newcommand{\rootnode}{\mathrm{root}}
\newcommand{\Path}{\mathrm{Path}}
\newcommand{\Node}{\mathrm{Node}}
\newcommand{\Edge}{\mathrm{Edge}}
\newcommand{\Parent}{\mathrm{Parent}}

\subsubsection{Sufficient Condition}

In influence maximization, there is a directed graph $G=(V, E)$.
For convenience, we use an edge $e$ as the index, e.g. $\mu_e$.
In this application, action $S$ is a set of at most $k$ nodes, so we also
	interpret $S$ as a set of nodes.

Recall TPM bounded smoothness (Condition~\ref{cond:1-normTPM}).
The formula that we need to satisfy is
\begin{equation}
|r_S(\vmu)-r_S(\vmu')|\le B \sum_{e\in E}p_e^{\vmu, S}|\mu_e-\mu'_e|,
\label{eq:IMbandit.TPM}
\end{equation}
where $B=\max_{u\in V} |\{v\in V\mid v\text{ can be reached from }u\} |$ for influence maximization bandit,
\chgins{and $p_e^{\vmu, S}$ stands for $p_e^{D, S}$ as $D$ can be uniquely determined by $\vmu$}.

Let $r_S^{v}(\vmu)$ be the probability that $v$ is activated.
We claim that if for every node $v$ and every $\vmu$ and $\vmu'$ vectors, we have
\begin{equation} \label{eq:IMbandit.nodeTPM}
|r_S^{v}(\vmu)-r_S^{v}(\vmu')|\le \sum_{e\in E}p_e^{\vmu, S}|\mu_e-\mu'_e|,
\end{equation}
Then we have Inequality~\eqref{eq:IMbandit.TPM}.
The reason is as follows.
First, we show that Inequality~\eqref{eq:IMbandit.nodeTPM} holds for 
	all $\vmu$ and $\vmu'$ is equivalent to 
$|r_S^{v}(\vmu)-r_S^{v}(\vmu')|\le \sum_{e\in E, e\text{ can reach }v}p_e^{\vmu, S}|\mu_e-\mu'_e|$ for all $\vmu$ and $\vmu'$.
In fact, the direction from the above
	inequality to Inequality~\eqref{eq:IMbandit.nodeTPM} is trivial.
For the reverse direction, 
	let $\vmu''$ be an expectation vector such that for every edge $e$ that
	can reach $v$, $\mu''_e = \mu_e'$, and for every edge $e$ that cannot reach $v$,
	$\mu''_e = \mu_e$.
Since the $r_S^{v}(\vmu')$ is only affected by edges that can reach $v$, we have
	$r_S^{v}(\vmu') = r_S^{v}(\vmu'')$.
Then, we have $|r_S^{v}(\vmu)-r_S^{v}(\vmu')| = |r_S^{v}(\vmu)-r_S^{v}(\vmu'')| 
	\le \sum_{e\in E}p_e^{\vmu, S}|\mu_e-\mu''_e|
	= \sum_{e\in E, e\text{ can reach }v}p_e^{\vmu, S}|\mu_e-\mu'_e|$.
Next, assuming $|r_S^{v}(\vmu)-r_S^{v}(\vmu')|\le \sum_{e\in E, e\text{ can reach }v}p_e^{\vmu, S}|\mu_e-\mu'_e|$ holds for all $v\in V$, we have
\begin{align*}
|r_S(\vmu)-r_S(\vmu')| & = 
|\sum_{v\in V} r_S^{v}(\vmu) - \sum_{v\in \Gamma(S)} r_S^{v}(\vmu')| \\
& \le \sum_{v\in V} | r_S^{v}(\vmu)-r_S^{v}(\vmu')| \\
& \le  \sum_{v\in V}\ \  \sum_{e\in E, e\text{ can reach }v}p_e^{\vmu, S}|\mu_e-\mu'_e| \\
& = \sum_{e\in E}\ \  \sum_{v\in V, v\text{ can be reached from }e} p_e^{\vmu, S}|\mu_e-\mu'_e| \\
&\le B \sum_{e\in E}p_e^{\vmu, S}|\mu_e-\mu'_e|.
\end{align*}
Thus, Inequality~\eqref{eq:IMbandit.TPM} holds.

Furthermore, we argue that it is sufficient to show that 		
	Inequality~\eqref{eq:IMbandit.nodeTPM} holds
	when 
	(1) $\vmu \le \vmu'$, i.e. for every edge $e$, $\mu_e\le \mu'_e$,; and
	(2) $|S|=1$.
The first condition is a straightforward conclusion from the Monotonicity 
	condition (Condition ~\ref{cond:monotone}).
For the second condition, we may assume the seed set $S$ consists of only one node
	without loss of generality.
Otherwise, we may add a super seed node $s^{\circ}$ and add edges from $s^{\circ}$ to $s$ and let $\mu_{(s^{\circ}, s)}=\mu'_{(s^{\circ}, s)}=1$ for every node $s$ in $S$.

Therefore, in the rest of the proof of Lemma~\ref{lem:IMbandit}, 
	we prove that the influence maximization bandit satisfies 
	Inequality~\eqref{eq:IMbandit.nodeTPM} for $\vmu \le \vmu'$ and $|S|=1$.
Let $s$ be the single seed node, and $S = \{s\}$.

\subsubsection{Paths}
In this subsection, we define an order of paths and assign the influence to the smallest path.
Consider all the paths from $s$ to $v$.
A path $L$ from $s$ to $v$ is a sequence of edges $(e_1=(s, u_1), e_2=(u_1, u_2), \dots, e_{|L|}=(u_{|L|-1}, v))$.
A simple path is a path that $s, v, u_1, \dots, u_{|L|-1}$ are distinct.

We call each possible value of random vector $X$ an outcome and denote it with 
	vector $\vx \in \{0, 1\}^m$.
We say an edge $e$ is {\em live} (with respect to $\vx$) if the corresponding component of $\vx$ is 1, i.e. influence can propagate through $e$ with the propagation under $\vx$.
Thus, connecting with the terminology in the influence maximization literature
	\cite{kempe03,chen2013information}, $\vx$ corresponds to a {\em live-edge graph}
	in $G$, while $X$ corresponds to a {\em random live-edge graph}.
We say a path $L$ is {\em live} (with respect to $\vx$) if every edge of $L$ is live.
Then we have
$r_S^v(\vmu)=\Pr_{\vx\sim X}\{\text{there is a live path from $s$ to $v$ in $\vx$}\}$.
For each $\vx$ that contains a live path from $s$ to $v$, we designate 
	a path to $\vx$ as follows.
We first list all the edges in an arbitrary order,
and for every different edges $e_1$ and $e_2$, define $e_1<e_2$ if $e_1$ appears before $e_2$.
To compare two paths $L$ and $L'$, we first order the edges in $L$ and $L'$ in 
	the descending order, respectively, and then compare them in 
	the lexicographical order.
In other words, to compare two paths, first compare their largest edges,
if there is a tie, compare their second largest edges, and so on.
If two paths continue to tie on edges and then one path ends with no more edges,
	then the shorter path is smaller.
For every outcome $\vx$ such that there is a live path from $s$ to $v$,
we designate the smallest live path $L$ from $s$ to $v$ in $\vx$ to $\vx$.
Then each path from $s$ to $v$ in the original graph $G$
	has a subset of outcome $\vx$'s that are designated to $L$,
	which means all paths from $s$ to $v$ partition all outcomes $\vx$ by which
	path $\vx$ is designated to.
Thus, let $r_{v\mid L}^{\vmu, S}=\sum_{\vx\text{ is designated to $L$}} \Pr[X=\vx]$, namely the contribution of path $L$ through the outcome $\vx$ designated to $L$, and
	we have $r_S^{v}(\vmu)=\sum_{L\text{ is a path from $s$ to $v$}} r_{v\mid L}^{\vmu, S}$.
That is, we decompose $r_S^{v}(\vmu)$ by $r_{v\mid L}^{\vmu, S}$'s according to
	paths $L$ from $s$ to $v$.

Before going further, we first figure out some basic properties of the smallest live path.
The smallest live path must be simple, otherwise we can remove loops to get a smaller live path.
Moreover, each substring of the smallest live path in $\vx$ must also be 
	the smallest in $\vx$ for its respective starting and ending nodes.
For a path $L=(e_1=(u_0, u_1), e_2=(u_1, u_2), \dots, e_{|L|}=(u_{|L|-1}, u_{|L|}))$,
a substring is a consecutive subsequence $L_1=(e_i, e_{i+1}, \dots, e_j)$.
If $L$ is the smallest live path from $s$ to $v$ in $\vx$,
any substring $L_1$ must also be the smallest live path from $u$ to $w$ in $\vx$,
where $u$ and $w$ are the start and the end of $L_1$, respectively.
Otherwise, if $L_2$ is a live path from $u$ to $w$ that smaller than $L_1$,
then we can replace $L_1$ with $L_2$ in $L$ to get a smaller live path.

\subsubsection{Bypass}
In this subsection, we define bypass, which is a tool for calculating the probability that a path is \emph{not} the smallest. 
For a path $L=(e_1=(u_0, u_1), e_2=(u_1, u_2), \dots, e_{|L|}=(u_{|L|-1}, u_{|L|}))$,
a bypass is a path from $u_i$ to $u_j$ that
\begin{enumerate}[noitemsep, label=(\arabic*)]
	\item shares no edges with $L$;
	\item is smaller than the substring of $L$ between $u_i$ and $u_j$.
\end{enumerate}
A bypass is live (with respect to $\vx$)  is defined in the same way as a path being live.
For a live path $L$ in $\vx$ from some node $u_0$ to some other node $u_{|L|}$, 
	if there is a live bypass of $L$, then $L$ cannot be the smallest
	live path from $u_0$ to $u_{|L|}$.
The reverse also holds: if a live path $L$ has no live bypasses, then $L$ is the smallest
	live path from $u_{0}$ to $u_{|L|}$.
To prove the reverse direction, 
	assume that there is a live path $L'$ from $u_{0}$ to $u_{|L|}$ smaller than $L$.
Let $e_i$ be the largest edge in $L$ that is not in $L'$.
Because $L' < L$, such $e_i$ must exist, and moreover $e_i$ must
	be larger than all edges in $L'$ but not $L$.
By breaking $L$ at $e_i$, we divide the nodes covered by $L$ into two parts,
the start part and the end part.
Let $w$ be the first node in $L'$ that is in the end part of $L$.
Such node $w$ must exist because the end node $u_{|L|}$ is in the end part of $L$.
Let $u$ be the last node in $L'$ that appears before $w$ in $L'$ and is in the 
	start part of $L$.
Such node $u$ must exist because the starting node $u_0$ is in the start part.
Then the substring of $L'$ between $u$ and $w$
	must share no edges with $L$.
Otherwise, if the substring of $L'$ between $u$ and $w$ shares one edge $(u_j,u_{j+1})$ with $L$, $(u_j,u_{j+1})$ cannot be $e_i$, so
	$u$ cannot be $u_j$ and $w$ cannot be $u_{j+1}$.
Then, (a) if $u_{j+1}$ is in the end part of $L$, 
	then $u_{j+1}$ appearing before $w$ in $L'$ contradicts
		to $w$'s definition; and
	(b) if $u_{j+1}$ is in the start part of $L$, $u_{j+1}$ appearing after $u$ and before $w$ in $L'$ contradicts to the definition of $u$.
Therefore, the substring of $L'$ between $u$ and $w$ shares no edges with $L$.
Then since $e_i$ is larger than any edge in $L'$ and not in $L$, 
	the substring of $L'$ between $u$ and $w$ is indeed 
	a bypass of $L$.

For a path $L=(e_1=(u_0, u_1), e_2=(u_1, u_2), \dots, e_{|L|}=(u_{|L|-1}, u_{|L|}))$,
let $p_L^{\vmu, S}$ be the probability that $L$ is the smallest live path from its start to its end. 
Note that if $L$ is a path from $s$ to $v$, then we have
	$p_L^{\vmu, S} = r_{v\mid L}^{\vmu, S}$.
With bypass, we have $p_L^{\vmu, S}=p_{1, L}^{\vmu, S}p_{2, L}^{\vmu, S}$,
where $p_{1, L}^{\vmu, S}$ is the probability that $L$ is live
and $p_{2, L}^{\vmu, S}$ is the probability that there is no live bypasses of $L$.
It is clear that $p_{1, L}^{\vmu, S} = \prod_{i=1}^{|L|} \mu_{e_i}$,
	and $p_{2, L}^{\vmu, S}$ is the probability that
	some subset of edges in $E\setminus L$ forming a live bypass of
	$L$ does not occur.
These two events are independent, since they are about two disjoint subsets of $E$.

\subsubsection{Bottom-up modification}
\begin{figure}
	\centering
	\subcaptionbox{A sample network}
	[.5\textwidth]
	{%
		\begin{tikzpicture}[every node/.style = {circle, draw}]
		\node (v1) at (0,1.5) {$s$};
		\node (v2) at (-1,0) {$u$};
		\node (v4) at (1,0) {$w$};
		\node (v3) at (0,-1.5) {$v$};
		\tikzset{edge/.style={->, > = latex}}
		\draw[edge]  (v1) -> (v2);
		\draw[edge]  (v2) -> (v3);
		\draw[edge] (v1) -> (v4);
		\draw[edge]  (v4) -> (v3);
		\draw  (v2) -> (v4);
		\end{tikzpicture}%
	}%
	\subcaptionbox{Search tree with $\Node$ marked in each node}
	[.5\textwidth]
	{%
		\begin{tikzpicture}[every node/.style = {circle, draw}]
		\node (v1) at (0,0) {$s$};
		\node (v2) at (-1,-1) {$u$};
		\node (v6) at (1,-1) {$w$};
		\node (v3) at (-1.5,-2) {$v$};
		\node (v4) at (-0.5,-2) {$w$};
		\node (v5) at (0,-3) {$v$};
		\node (v7) at (0.5,-2) {$v$};
		\node (v8) at (1.5,-2) {$u$};
		\node (v9) at (2,-3) {$v$};
		\draw  (v1) edge (v2);
		\draw  (v2) edge (v3);
		\draw  (v2) edge (v4);
		\draw  (v4) edge (v5);
		\draw  (v1) edge (v6);
		\draw  (v7) edge (v6);
		\draw  (v6) edge (v8);
		\draw  (v8) edge (v9);
		\end{tikzpicture}%
	}
	\caption{A sample network and its search tree}
	\label{fig:IMbandit}
\end{figure}
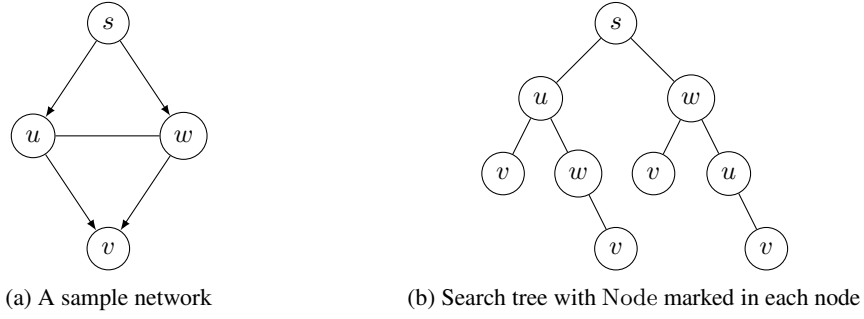

We now describe the search tree formed from all simple paths from $s$ to $v$.
We use $y, z$ to denote nodes in this tree.
Each node $y$ is corresponding to a prefix of a path from $s$ to $v$, which is also a path denoted by $\Path(y)$.
Denote the end node of $\Path(y)$ with $\Node(y)$.
Denote the last edge of $\Path(y)$ with $\Edge(y)$.
Denote the root of the tree with $\rootnode$.
$\Path(\rootnode)$ is the empty path $\emptyset$.
Specifically, $\Node(\rootnode)=s$, as $s$ is the start node of every path in our consideration. $\Edge(\rootnode)$ is undefined.
For every non-root node $y$ in the tree, its parent is the node $z$
such that $\Path(z)$ is the $(|\Path(y)|-1)$-prefix of $\Path(y)$.
Figure~\ref{fig:IMbandit} shows a sample of this tree structure.

For a node $y$ in the tree, we simplify the notation
	$p_{\Path(y)}^{\vmu, S}$ to $p_y^{\vmu, S}$.
Similarly, for a leaf node $y$ in the tree, we simplify the notation
	$r_{v\mid \Path(y)}^{\vmu, S}$ to $r_{v\mid y}^{\vmu, S}$.
Then we have 
$r_S^{v}(\vmu)=\sum_{y\text{ is leaf}} r_{v\mid y}^{\vmu, S}
=\sum_{y\text{ is a leaf}} p_y^{\vmu, S}$.

We want to show that for all $\vmu \le \vmu'$, we have
\begin{align} \label{eq:IMbandit.nodeTPM2}
r_S^{v}(\vmu') - r_S^{v}(\vmu)
= \sum_{y\text{ is a leaf}} \left(p_y^{\vmu', S} - p_y^{\vmu, S}\right)
\le p_e^{\vmu, S} \sum_{e \in E} (\mu'_e - \mu_e),
\end{align}
which is the same as Inequality~\eqref{eq:IMbandit.nodeTPM} that we want to show.

Let $\vmu^{(y)}$ be the vector that
$$\mu^{(y)}_e = \begin{cases}
\mu_e, &\text{if } e\in \Path(y),\\
\mu'_e, &\text{if } e\not\in \Path(y).\\
\end{cases}$$

Thus we have
$p_y^{\vmu^{(y)}, S}=p_{1, y}^{\vmu, S}p_{2, y}^{\vmu', S}$.
Since for all edges $e \not\in \Path(y)$, $\mu_e \le \mu'_e$,
	the probability that there is no live bypasses of $\Path(y)$ is higher
	under $\vmu$ than under $\vmu'$, 
	that is, $p_{2, y}^{\vmu', S} \le p_{2, y}^{\vmu, S}$.
Therefore, $p_y^{\vmu^{(y)}, S} \le p_y^{\vmu, S}$, which
	means that, to prove Inequality~\eqref{eq:IMbandit.nodeTPM2}, it is
	enough to prove
\begin{align} \label{eq:IMbandit.nodeTPM3}
\sum_{y\text{ is a leaf}} \left(p_y^{\vmu', S} - p_y^{\vmu^{(y)}, S}\right)
\le p_e^{\vmu, S} \sum_{e \in E} (\mu'_e - \mu_e).
\end{align}

We now consider the bottom-up modification of the expectations in $\Path(y)$.
\begin{equation}
p_y^{\vmu', S} - p_y^{\vmu^{(y)}, S}
= \sum_{i=1}^{|\Path(y)|} \left(p_y^{\vmu^{(z_{i-1})}, S} - p_y^{\vmu^{(z_{i})}, S}\right),
\end{equation}
where $z_i$ is the ancestor of $y$ at depth $i$. (Root has depth 0.)
By switching summations and regrouping the summands 
	$\left(p_y^{\vmu^{(z_{i-1})}, S} - p_y^{\vmu^{(z_{i})}, S}\right)$ under $z_i$,
	we have
\begin{equation}
\sum_{y\text{ is a leaf}}\left( p_y^{\vmu', S} - p_y^{\vmu^{(y)}, S} \right)
= \sum_{y\text{ is a non-root node}} \; \sum_{z\text{ is a leaf under }y}
\left(p_{z}^{\vmu^{(\Parent(y))}, S} - p_{z}^{\vmu^{(y)}, S}\right).
\label{eq:IMbandit.diffallocate}
\end{equation}
We generalize the definition of $r_{v\mid y}^{\vmu, S}$ to non-leaf  nodes $y$ by
$$r_{v\mid y}^{\vmu, S} = \sum_{z\text{ is a leaf under }y} p_{z}^{\vmu, S}.$$
It is clear that this definition coincides the old one when $y$ is a leaf.
Now
\begin{equation}
\eqref{eq:IMbandit.diffallocate} = \sum_{y\text{ is a non-root node}}
\left(r_{v\mid y}^{\vmu^{(\Parent(y))}, S} - r_{v\mid y}^{\vmu^{(y)}, S}\right).
\label{eq:IMbandit.usesubtreereward}
\end{equation}

$$r_{v\mid y}^{\vmu, S} = \sum_{z\text{ is a leaf under }y} p_{z}^{\vmu, S}
= \sum_{z\text{ is a leaf under }y} p_{1, z}^{\vmu, S}p_{2, z}^{\vmu, S}
= p_{1, y}^{\vmu, S}\sum_{z\text{ is a leaf under }y} \frac{p_{1, z}^{\vmu, S}}{p_{1, y}^{\vmu, S}}p_{2, z}^{\vmu, S}.
$$
$\frac{p_{1, z}^{\vmu, S}}{p_{1, y}^{\vmu, S}}p_{2, z}^{\vmu, S}$ does not depend on
$\mu_e$ for every $e\in \Path(y)$.
So
\begin{align}
r_{v\mid y}^{\vmu^{(\Parent(y))}, S} - r_{v\mid y}^{\vmu^{(y)}, S}
&=\left(p_{1, y}^{\vmu^{(\Parent(y))}, S} - p_{1, y}^{\vmu^{(y)}, S}\right)\sum_{z\text{ is a leaf under }y} \frac{p_{1, z}^{\vmu', S}}{p_{1, y}^{\vmu', S}}p_{2, z}^{\vmu', S}\nonumber\\
&=\left(\mu'_{\Edge(y)} - \mu_{\Edge(y)}\right) p_{1, \Parent(y)}^{\vmu, S}\sum_{z\text{ is a leaf under }y} \frac{p_{1, z}^{\vmu', S}}{p_{1, y}^{\vmu', S}}p_{2, z}^{\vmu', S}.
\label{eq:IMbandit.subtreediff}
\end{align}

For each leaf $z$ under $y$, the event that $\Path(z)$ is the smallest live path from $s$ to $v$ is exclusive from each other.
And that event is included in that $\Path(y)$ is the smallest live path from $s$ to $\Node(y)$. So
$$\sum_{z\text{ is a leaf under }y}p_{1, z}^{\vmu', S}p_{2, z}^{\vmu', S}
\le p_{1, y}^{\vmu', S}p_{2, y}^{\vmu', S},$$
and thus
$$\sum_{z\text{ is a leaf under }y} \frac{p_{1, z}^{\vmu', S}}{p_{1, y}^{\vmu', S}}p_{2, z}^{\vmu', S} \le p_{2, y}^{\vmu', S} \le p_{2, y}^{\vmu, S}.$$
So
$$\eqref{eq:IMbandit.subtreediff} \le \left(\mu'_{\Edge(y)} - \mu_{\Edge(y)}\right) p_{1, \Parent(y)}^{\vmu, S} p_{2, y}^{\vmu, S}.$$
Then
\begin{equation}
\eqref{eq:IMbandit.usesubtreereward} \le \sum_{y\text{ is a non-root node}}\left(\mu'_{\Edge(y)} - \mu_{\Edge(y)}\right) p_{1, \Parent(y)}^{\vmu, S} p_{2, y}^{\vmu, S}
=\sum_{e\in E} (\mu'_e - \mu_e) \sum_{\Edge(y)=e} p_{1, \Parent(y)}^{\vmu, S} p_{2, y}^{\vmu, S}.
\label{eq:IMbandit.final}
\end{equation}

We then show 
\begin{equation}
	\sum_{\Edge(y)=e} p_{1, \Parent(y)}^{\vmu, S}p_{2, y}^{\vmu, S} \le p_e^{\vmu, S},
	\label{eq:IMbandit.boundTP}
\end{equation}
for every edge $e$.
If $e$ is a directed edge from $u$ to $w$ , $p_e^{\vmu, S} \ge \sum_{\Edge(y)=e} p_{\Parent(y)}^{\vmu, S}$,
	since $p_{\Parent(y)}^{\vmu, S}$ is the probability that the path $\Path(\Parent(y))$
	is the smallest live path from $s$ to $\Node(\Parent(y))=u$, and thus such events are
	mutually exclusive for different $y$ with $\Edge(y) = e$.
Then $p_e^{\vmu, S} \ge \sum_{\Edge(y)=e} p_{1, \Parent(y)}^{\vmu, S}p_{2, y}^{\vmu, S}$
	as $p_{2, \Parent(y)}^{\vmu, S}\ge p_{2, y}^{\vmu, S}$.
Thus we have \eqref{eq:IMbandit.boundTP}.

Combining Inequalities \eqref{eq:IMbandit.final} and \eqref{eq:IMbandit.boundTP}, 
	we prove the key Inequality~\eqref{eq:IMbandit.nodeTPM3}, 
	which in turn shows that
	the influence maximization bandit satisfies the TPM bounded smoothness
	condition with 
	$B=\max_{u\in V} |\{v\in V\mid v\text{ can be reached from }u\} |$.

\chgins{
\section{Detailed Comparison with \cite{Wen2016} on the Regret Bounds for Influence Maximization Bandits}
\label{app:compareWen}
	
\begin{table}
	\begin{center}
		\begin{tabular}{|c|c|c|}
			\hline
			Topology & Bound in \citep{Wen2016} & Our bound\\
			\hline
			bar graphs & $\tilde O\left(|V|\sqrt{kT}\right)$ & $\tilde O\left(\sqrt{k|V|T}\right)$\\
			\hline
			star graphs & $\tilde O\left(|V|^2\sqrt{kT}\right)$ & $\tilde O\left(|V|^2\sqrt{T}\right)$\\
			\hline
			ray graphs & $\tilde O\left(|V|^{\frac{9}{4}}\sqrt{kT}\right)$ & $\tilde O\left(|V|^2\sqrt{T}\right)$\\
			\hline
			tree graphs & $\tilde O\left(|V|^{\frac{5}{2}}\sqrt{T}\right)$ & $\tilde O\left(|V|^2\sqrt{T}\right)$\\
			\hline
			grid graphs & $\tilde O\left(|V|^{\frac{5}{2}}\sqrt{T}\right)$ & $\tilde O\left(|V|^2\sqrt{T}\right)$\\
			\hline
			complete graphs & $\tilde O\left(|V|^4\sqrt{T}\right)$ & $\tilde O\left(|V|^3\sqrt{T}\right)$\\
			\hline
		\end{tabular}
	\end{center}
	\caption{Regret bound comparison with \citep{Wen2016}.}
	\label{tb:CompWen}
\end{table}

Let $G=(V,E)$ be the social graph we consider.
By Lemma~\ref{lem:IMbandit}, our Theorem~\ref{thm:1-normTPM} can be applied to the influence maximization bandit with
$B=\tilde{C}\le |V|$, which gives concrete $O(\log T)$ distribution-dependent and $O(\sqrt{T\log T})$ 
distribution-independent bounds for the influence maximization bandit.
\citet{Wen2016} also study the influence maximization bandit and eliminate the exponential factor $1/p^*$.
They use a complexity term $C_*$ to characterize their regret bound, where $C_*$ has complicated relationship with network topology
	and edge probabilities.
\citet{Wen2016} list several families of graphs with concrete regret bounds, ignoring the effect of edge probabilities on their complexity
	term $C_*$.
Our regret bounds with complexity term $\tilde{C}$ can also be applied to these graph families, and Table~\ref{tb:CompWen} list
	the comparison results between our regret bounds and their regret bounds. 
The comparison shows that our regret bounds are always better than their bounds, with an improvement factor from $O(\sqrt{k})$ to $O(|V|)$, 
	where $V$ is the set of
	nodes in the graph, and $k$ is the number of seeds to be selected in each round.
This indicates that, in terms of characterizing the topology effect on the regret bound, our simple complexity term $\tilde{C}$ is more effective
	than their complicated term $C_*$.
}

\section{Lower Bound Proofs (for Section~\ref{sec:lowerbound})}

\subsection{Proof of Theorem~\ref{thm:indlowerbound}} \label{sec:proof.indlowerbound}

\begin{algorithm}
	\begin{algorithmic}[1]
		\REQUIRE $m, T_{\cmab}, p$
		\COMMENT{$m$ is the number of arms,
			$T_{\cmab}$ is the number of rounds in CMAB,
			and $p$ is triggering probability.}
		\FOR{$t=1, \dots, T_\cmab$}
		\STATE sample $\gamma_t$ i.i.d. from Bernoulli distribution $B_p$
		\ENDFOR
		\STATE $\mathcal{H} \leftarrow  \emptyset$; $t_{\mab} \leftarrow 0$
		\FOR{$t=1,\dots,T_{\cmab}$}
		\STATE $S_{i_t} \leftarrow \textsf{CMAB-Oracle}(\mathcal{H})$ \COMMENT{Oracle decides the CMAB-T action based on the execution history}
		\IF{$\gamma_t=1$}
		\STATE $t_{\mab} \leftarrow t_{\mab} + 1$
		\STATE In MAB, play arm $i_t$ in round $t_{\mab}$, obtain feedback 
		$\tilde{X}^{(t_{\mab})}_{i_t}$
		\STATE In CMAB-T, $i_t$ is triggered with feedback $X^{(t)}_{i_t} = \tilde{X}^{(t_{\mab})}_{i_t}$, and set reward as $p^{-1}X^{(t)}_{i_t}$
		\STATE $\mathcal{H} \leftarrow \textsf{Append}(\mathcal{H}, (S_{i_t},\{i_t\},  X^{(t)}_{i_t})$ \COMMENT{$\{i_t\}$ is the set of triggered arms}
		\ELSE 
		\STATE \COMMENT{$\gamma_t=0$, and MAB is not played in this case}
		\STATE In CMAB-T, no arm is triggered, and the reward is 0
		\STATE $\mathcal{H} \leftarrow \textsf{Append}(\mathcal{H}, (S_{i_t}, \emptyset, -))$ \COMMENT{triggering set is empty, so no feedback}
		\ENDIF
		\ENDFOR \COMMENT{In the end, $T_{\mab} = t_{\mab}$}
	\end{algorithmic}
	\caption{Reduce MAB to CMAB-T}
	\label{alg:reduction}
\end{algorithm}

\begin{figure}
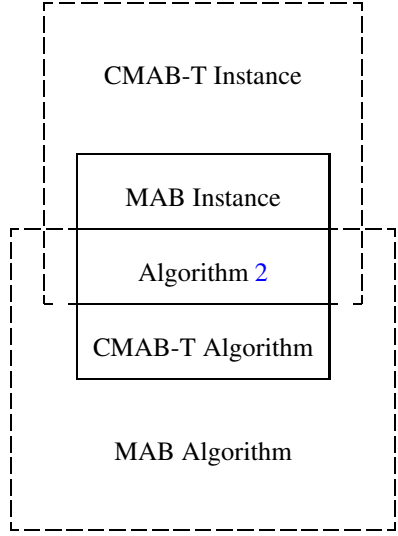

	\centering
	\renewcommand{\arraystretch}{2.6}
	\dashlinedash=6pt
	\dashlinegap=2pt
	\begin{tabular}{ccccc}
		\cdashline{2-4}
		\multicolumn{1}{c:}{}  & & \multirow{2}*{CMAB-T Instance}                    & & \multicolumn{1}{:c}{}  \\ 
		\multicolumn{1}{c:}{}  & &                                                      & & \multicolumn{1}{:c}{}  \\ \cline{3-3}
		\multicolumn{1}{c:}{}  & & \multicolumn{1}{|c|}{MAB Instance}                & & \multicolumn{1}{:c}{}  \\ \cdashline{1-2}\cline{3-3}\cdashline{4-5}
		\multicolumn{1}{:c:}{} & & \multicolumn{1}{|c|}{Algorithm~\ref{alg:reduction}} & & \multicolumn{1}{:c:}{} \\ \cdashline{2-2}\cline{3-3}\cdashline{4-4}
		\multicolumn{1}{:c}{}  & & \multicolumn{1}{|c|}{CMAB-T Algorithm}                  & & \multicolumn{1}{c:}{}  \\ \cline{3-3}
		\multicolumn{1}{:c}{}  & & \multirow{2}*{MAB Algorithm}                            & & \multicolumn{1}{c:}{}  \\ 
		\multicolumn{1}{:c}{}  & &                                                      & & \multicolumn{1}{c:}{}  \\ \hdashline
	\end{tabular}
	\caption{Reduction Structure}
	\label{fig:reduction}
\end{figure}

We prove the theorem by reducing classical MAB to this CMAB-T game instance by Algorithm~\ref{alg:reduction}.
For convenience, we define Bernoulli random variable $\gamma_t = \I\{\tau_t(S_{i_t},X^{(t)}) = \{i_t\} \}$,
	where $S_{i_t}$ is the action played in round $t$, and thus $\gamma_t$ is an indicator representing whether
a base arm is triggered in round $t$.
Moreover, to distinguish the environment outcome in MAB and CMAB-T  in the reduction, we use $\tilde{X}^{(t_{\mab})}$ to denote
the environment outcome in round $t_{\mab}$ of MAB, and $X^{(t)}$ to denote the environment outcome
in round $t$ of CMAB-T.

Figure~\ref{fig:reduction} shows the structure of reduction.
Algorithm~\ref{alg:reduction} adapts the CMAB-T algorithm to an MAB algorithm.
Conversely, it also adapts the MAB instance to the corresponding CMAB-T instance.
Thus when Algorithm~\ref{alg:reduction} runs, we have one MAB instance and one CMAB-T 
instance running simultaneously.
Let $T_\cmab$ be the total number of rounds in the CMAB-T instance and
$T_\mab$ be the total number of rounds in the MAB instance.
For convenience, we use $t$ to refer to the index of rounds in CMAB-T, while
$t_{\mab}$ is the index of rounds in MAB.
In Algorithm~\ref{alg:reduction}, we fix $T_\cmab$ and thus $T_\mab$ is a random variable.
We have $T_\mab = \sum_{t=1}^{T_\cmab}\gamma_t$.
So $\E[T_\mab]=pT_\cmab$ and
we have following lemma about the distribution of $T_\mab$.

\begin{mylem}
	If $pT_\cmab\ge 6$, then $\Pr\left[T_\mab \ge \frac{1}{2}pT_\cmab\right]\ge \frac{1}{2}$.
\end{mylem}

\begin{proof}
	$T_\mab=\sum_{t=1}^{T_\cmab} \gamma_t$.
	By multiplicative Chernoff bound (Fact~\ref{fact:chernoff}),
	$$
	\Pr[T_\mab\ge \frac{1}{2}pT_\cmab]
	\ge 1-\left(\frac{e^{-\frac{1}{2}}}{\left(\frac{1}{2}\right)^{\frac{1}{2}}}\right)
	^{pT_\cmab} 
	\ge \frac{1}{2},
	$$
	when $pT_\cmab\ge 6$.
	
	\wei{I did not see how the above formula is related to Fact~\ref{fact:chernoff}. My version would be below:}
	
	\qinshi{Current version (Fact~\ref{fact:chernoff}) of multiplicative Chernoff bound is not tight. My formula above is tighter. But it doesn't matter. $pT_\cmab\ge 5$ and $pT_\cmab\ge 6$ are virtually the same.}
	$$
	\Pr[T_\mab\ge \frac{1}{2}pT_\cmab]
	\ge 1-\left(e^{-\frac{1}{8}pT_\cmab}\right)
	\ge \frac{1}{2},
	$$
	when $pT_\cmab\ge 6$.
	
\end{proof}

In the following, we overload the notation $\calD$ to also represent a probabilistic distribution of
the environment instance (a.k.a. outcome distribution) $D$, and use $D \sim \calD$ to represent a random 
environment instance $D$ drawn from the distribution $\calD$.

\begin{mylem}
	\label{lem:reduction}
	Consider a random MAB environment instance $D$ drawn from a distribution $\calD$.
	Assume we have a lower bound $L(T_\mab)$ of expected regret,
	i.e. for every natural number $T_\mab$, any MAB algorithm $A$ has expected regret
	$$\E_{D\sim \calD}[Reg_{\mab,D}^A(T_\mab)]\ge L(T_\mab).$$
	Then consider the corresponding CMAB-T environment instance $D$.
	For every natural number $T_\cmab\ge 5p^{-1}$, any CMAB-T algorithm $A$ has expected regret
	\begin{equation}
	\E_{D\sim \calD}[Reg_{\cmab, D}^A(T_\cmab)]\ge \frac{1}{2}p^{-1}L(\frac{1}{2}pT_\cmab).
	\end{equation}
\end{mylem}

\begin{proof}
	Without loss of generality, we may assume $L(T)$ is non-decreasing,
	as regret of any strategy increases as $T$ increases.
	
	We prove the lemma using the reduction described above.
	We run Algorithm~\ref{alg:reduction} with $A$ be the CMAB-T oracle and $D$ be the environment instance.
	Let $\gamma$ be the vector $(\gamma_1, \gamma_2, \ldots, \gamma_{T_{\cmab}})$.
	Every possible value of $\gamma$ parameterizes Algorithm~\ref{alg:reduction}
	into an algorithm plays MAB problem for $T_\mab=\sum_{t=1}^{T_\cmab}\gamma_t$ rounds.
	We denote this MAB algorithm with $A_\gamma$.
	By our assumption, $\E_{D\sim \calD}[Reg_{\mab, D}^{A_\gamma}(T_\mab)]\ge L(T_\mab)$.
	
	Then we compare the regret in both cases.
	For a given distribution $D$, let $\mu_{i,D} = \E_{X\sim D}[X_i]$ and
	$\mu^*_D =\max_i \mu_{i,D}$.
	For MAB problem and every $\gamma$,
	\begin{align*}
	\E_{D\sim \calD}[Reg_{\mab, D}^{A_\gamma}(T_\mab)]
	&=  \E_{D\sim \calD} \left[ T_\mab\cdot\mu^*_D - \E\left[\sum_{t=1}^{T_\cmab}\gamma_t X_{i_t}\right] \right]\\
	&=  \E_{D\sim \calD} \left[\E\left[\sum_{t=1}^{T_\cmab} \gamma_t(\mu^*_D-X_{i_t})\right] \right]\\
	&=  \E_{D\sim \calD}\left[ \E\left[ \sum_{t=1}^{T_\cmab} \gamma_t(\mu^*_D-\mu_{i_t,D}) \right] \right],
	\end{align*}
	where the inner expectation is taken over the rest randomness, including the randomness of $i_t$, which is based on the random feedback history and
	the possible randomness of algorithm $A_{\gamma}$.
	For CMAB-T, we have
	\begin{align*}
	&\E_{D\sim \calD}[Reg_{\cmab, D}^A(T_\cmab)] \\
	&=	\E_{D\sim \calD} \left[ T_\cmab\cdot\mu^*_D - \E_{ \gamma\sim B_p^{T_{\cmab}}}\left[ \E\left[ \sum_{t=1}^{T_\cmab}\gamma_t p^{-1} X_{i_t}\right] \right] \right]\\
	&=	\E_{D\sim \calD} \left[T_\cmab\cdot\mu^*_D - \E_{\gamma\sim B_p^{T_{\cmab}}}\left[\E\left[ \sum_{t=1}^{T_\cmab}\gamma_t p^{-1} \mu_{i_t,D} \right] \right] \right]\\
	&=	\E_{D\sim \calD} \left[ p T_\cmab\cdot p^{-1} \mu^*_D - \E_{\gamma\sim B_p^{T_{\cmab}}}\left[\E\left[\sum_{t=1}^{T_\cmab}\gamma_t p^{-1} \mu_{i_t,D} \right] \right] \right]\\
	&=	\E_{D\sim \calD} \left[ \E_{\gamma\sim B_p^{T_{\cmab}}}\left[ \sum_{t=1}^{T_\cmab}\gamma_t p^{-1} \mu^*_D \right]
	- \E_{\gamma\sim B_p^{T_{\cmab}}}\left[\E\left[ \sum_{t=1}^{T_\cmab}\gamma_t p^{-1} \mu_{i_t,D} \right] \right] \right] \\
	&=	p^{-1}\E_{D\sim \calD, \gamma\sim B_p^{T_{\cmab}}}\left[\E\left[\sum_{t=1}^{T_\cmab}\gamma_t(\mu^*-\mu_{i_t,D}) \right] \right],
	\end{align*}
	where the innermost expectation is taken over the rest randomness such as the randomness of $i_t$.
	Therefore
	$$\E_{D\sim \calD}[Reg_{\cmab, D}^A(T_\cmab)]=p^{-1}\E_{D\sim \calD, \gamma\sim B_p^{T_{\cmab}}}[Reg_{\mab, D}^{A_\gamma}(T_\mab)].$$
	Calculation above also shows $\E_{D\sim \calD}[Reg_{\mab, D}^{A_\gamma}(T_\mab)] \ge 0$. \wei{Why do we need this? Also $\gamma$ could be all $0$, so
		the claim is not true for every $\gamma$.}
	\qinshi{I am not true with what I meant. It seems to be $\ge 0$.
		And we need it as we use a part of $\gamma$ for lower bound and this formula means the remaining part has non-negative regret.}
	And by monotonicity of $L(T)$,
	\begin{align*}
	\E_{D}[Reg_{\cmab, D}^A(T_\cmab)]
	&=	 p^{-1}\E_{D, \gamma}[Reg_{\mab, D}^{A_\gamma}(T_\mab)] \\
	&\ge p^{-1}\E_{D, \gamma}[\I\{T_\mab\ge \frac{1}{2}pT_\cmab\}Reg_{\mab, D}^{A_\gamma}(T_\mab)] \\
	&\ge p^{-1}\E_{D, \gamma}[\I\{T_\mab\ge \frac{1}{2}pT_\cmab\}L(\frac{1}{2}pT_\cmab)] \\
	&=	 p^{-1}\Pr_{D, \gamma}\{T_\mab\ge \frac{1}{2}pT_\cmab\}L(\frac{1}{2}pT_\cmab) \\
	&\ge\frac{1}{2}p^{-1}L(\frac{1}{2}pT_\cmab).
	\end{align*}
\end{proof}

\begin{mylem}
	\label{lem:mab indep lower bound}
	Let $m$ be the number of arms and $T$ be the number of rounds.
	Let $\eps=\frac{1}{10}\sqrt{m/T}$.
	Then define the family of MAB outcome distributions $\calD = \{D_1, \dots, D_m\}$ with
	$$\Pr_{D_j}\{X_i=1\}=\begin{cases}
	\frac{1}{2}      &\text{if } i\neq j\\
	\frac{1}{2}+\eps &\text{if } i=j
	\end{cases}.$$
	Let $D$ be a random environment instance uniformly drawn from $\calD$,
	then for any MAB algorithm $A$,
	$$\E_{D\sim \calD} \left[ Reg_{\mab, D}^A(T) \right] \ge \frac{\eps T}{6}=\frac{1}{60}\sqrt{mT}.$$
\end{mylem}

\qinshi{I found the proof the lemma above in a lecture note, but I cannot find a paper we can cite.}
\wei{We do need to cite this result. How is this related to the MAB independent lower bound, I believe first appear
	in Auer et al.'s non-stochastic MAB paper in 2002?}
\qinshi{I browsed Auer's papers and it seems this result (or a similar version) first appears in Section~7 of the journal version in 1998 of ``Gambling in a rigged casino: The adversarial multi-armed bandit problem''. I cannot access their conference version in 1995. In the journal version, regret lower bound is taken minimum with $T$, as regret cannot exceed $T$. In our current version, the author of the lecture note I referred in fact assumed $\eps<\frac{1}{2}$. I also did not pay attention to this issue.}

\begin{proof} [Proof of Theorem~\ref{thm:indlowerbound}]
	Let $\calD$ be the family of outcome distributions defined in Lemma~\ref{lem:mab indep lower bound}, and $D$ is
	uniformly drawn from $\calD$.
	Applying the result of Lemma~\ref{lem:mab indep lower bound} to  Lemma~\ref{lem:reduction}, with
	$L(T)=\frac{1}{60}\sqrt{mT}$ in Lemma~\ref{lem:reduction}, we have
	\begin{align*}
	\E_{D\sim \calD} \left[ Reg_{\cmab, D}^A(T) \right]
	&\ge \frac{1}{2}p^{-1}L(\frac{1}{2}pT)\\
	&=   \frac{1}{2}p^{-1}\cdot \frac{1}{60}\sqrt{\frac{1}{2}mpT}\\
	&>   \frac{1}{170}\sqrt{\frac{m T}{p}}.
	\end{align*}
	Since $D$ is uniformly drawn from $\calD$, then there must exists a $D\in \calD$ such that
	\begin{align*}
	Reg_{\cmab, D}^A(T) \ge \frac{1}{170}\sqrt{\frac{m T}{p}}.
	\end{align*}
\end{proof}

It is easy to show corresponding CMAB-T problem satisfies original bounded smoothness (Condition~\ref{cond:boundedsmoothness}) with $f(x)=x$.
So the theorem above gives an example that the upper bound in \cite{CWYW16} is tight up to a 
$O(\sqrt{\log T})$ factor.

\qinshi{(2.21) {\bf Caution!} This instance does not satisfy TPM bounded smoothness.
	It satisfies original bounded smoothness instead.}

\wei{I think the above reference to Condition~\ref{cond:globalTPM} is not accurate, and I rewrite it below.}

\qinshi{This problem instance satisfies Condition~2 with $f(x)=x$, and Condition~3 with $f(x)=p^{-1}x$, not Condition~3 with $f(x)=x$.}


\wei{Please check if the discussion in the above paragraph is reasonable.}

\qinshi{It's good to compare with original bound in with more clear clauses.
	I am afraid it is not good to say our new result is bad in this case.
	We may claim our new upper bound can be timed by $O(p_{\max})$ if $p_{\max}$ is small. This is easy to see when checking the proof. But this is not the case we major care about. So we did not write this form in our theorem.}

\subsection{Proof of Theorem~\ref{thm:deplowerbound}} \label{sec:proof.deplowerbound}
\begin{proof} [Proof of Theorem~\ref{thm:deplowerbound}]
	We regard this kind of CMAB-T problem instances as a variant of classical MAB,
	that each arm gives three possible outcomes, $0$, $1$, and $\none$.
	Denote these arms with random variables $X'_1, \dots, X'_n$.
	The reward is $p^{-1}$ times of the outcome if the outcome is $0$ or $1$, while the reward is $0$ if the outcome is $\none$.
	This variant is equivalent to the CMAB-T instances:
	Outcome $X'_i = \none$ corresponds to Bernoulli base arm $X_i$ in CMAB-T not being triggered,
	outcome $X'_i = 1$ or $0$ corresponds to Bernoulli base arm $X_i$ being triggered and $X_i=1$ 
	or $0$, respectively.
	Thus $\Pr[X'_i=\none]=1-p$, $\Pr[X'_i=0]=p(1-\mu_i)$, and $\Pr[X'_i=1]=p\mu_i$, where $p$ is the triggering probability and $\mu_i$ is the expectation of $X_i$.
	
	\wei{Is the above definition for MAB or CMAB-T? If it is for CMAB-T, we should talk about optimal
		action, not optimal arm? But as the optimal action may not exist when action space is infinite,
		how should we handle it?}
	\qinshi{Actually, it is neither MAB nor CMAB-T, if strictly speaking.
		It is a generalization of MAB that allows the algorithm to get a symbol which implies the reward, instead of only the reward -- e.g. the algorithm is able to differentiate 0 and $\none$, although the reward is all 0.
		It is easy to see this bandit has a 1-1 correspondence with our ``corresponding CMAB-T instance''.}
	
	Let $X$ and $Y$ be random variables whose values are in the same finite set $V$. Define the KL-divergence
	$$\kl(X, Y) = \sum_{x\in V} \Pr\{X=x\} \ln \frac{\Pr\{X=x\}}{\Pr\{Y=x\}}.$$
	For example the KL-divergence between $X'_1$ and $X'_2$ is
	\begin{align*}
	\kl(X'_1, X'_2)
	& = \Pr\{X'_1=\none\} \ln \frac{\Pr\{X'_1=\none\}}{\Pr\{X'_2=\none\}}
	+ \Pr\{X'_1=0\} \ln \frac{\Pr\{X'_1=0\}}{\Pr\{X'_2=0\}} \\
	& \quad \quad 	+ \Pr\{X'_1=1\} \ln \frac{\Pr\{X'_1=1\}}{\Pr\{X'_2=1\}}\\
	& = (1-p)\ln\frac{1-p}{1-p} + p(1-\mu_1)\ln\frac{p(1-\mu_1)}{p(1-\mu_2)} + p\mu_1 \ln \frac{p\mu_1}{p\mu_2}\\
	& = 0 + p(1-\mu_1) \ln \frac{1-\mu_1}{1-\mu_2} + p\mu_1 \ln \frac{\mu_1}{\mu_2}\\
	& = p\cdot\left[(1-\mu_1) \ln \frac{1-\mu_1}{1-\mu_2} + \mu_1 \ln \frac{\mu_1}{\mu_2}\right]\\
	& = p\cdot\kl(X_1, X_2).
	\end{align*}
	
	Thus, intuitively it takes $p^{-1}$ times more rounds to differentiate $X'_1$ and $X'_2$ than $X_1$ and $X_2$, which is stated formally in theorem below.
	
	\qinshi{A necessary assumption is $\mu^*<1$ (strictly smaller). Bubeck and Cesa-Bianchi did not write this assumption clearly in their survey.}
	
	\wei{Should we say in CMAB-T in the following theorem? Also, in the theorem do we want to say
		``there exists a problem instance of CMAB-T''? Is this expression used for the classical MAB?}
	
	\wei{I think before the algorithm, we can say that the following theorem and its proof is adapted
		from the original result for Bernoulli random variables (cite Lai and Robbin?), and for
		completeness, we include the entire proof here.}
\end{proof}
\begin{proof}
	The analysis is generalized from the case that the arms are Bernoulli random variables. For an arm $i$, we use $N_i(T)$ to denote the number of times the arm $i$ is played in $T$ rounds. For each non-optimal arm $i$, i.e. $\mu_i<\mu^*<1$, we show
	\begin{equation}
	\liminf_{T\to +\infty} \frac{\E[N_i(T)]}{\ln T} \ge \frac{p^{-1}}{\kl(X_i, X_{i^*})}
	= \frac{1}{\kl(X'_i, X'_{i^*})}.\label{eq:deplow.play}
	\end{equation}
	Then by formula
	$$Reg_{\vmu}^A(T) = \sum_{i:\mu_i<\mu^*} \E[N_i(T)]\Delta_i,$$
	the theorem holds.
	
	\qinshi{In fact, maybe we can directly say the result holds using original proof and point out where change is needed.}
	
	Without loss of generality, we may assume arm 1 is an optimal arm and arm 2 is non-optimal. We prove Eq.~\eqref{eq:deplow.play} for arm 2 and then the inequality holds for every arm. Consider that if we replace arm 2 with a fictional arm $2'$, which has an expectation $\mu_{2'}$ slightly greater than $\mu_1$, then arm 1 will become non-optimal and strategy $A$ will play arm 1 for $o(n^a)$ times for any $a>0$. So strategy $A$ must play arm $2$ for enough times, to differentiate from arm $2'$.
	
	Formally, let $\eps>0$ be any positive real number. Let $\mu_{2'}$ be a real number such that $\mu_{2'}>\mu_1$ and
	\begin{equation}
	\kl(X_2, X_{2'})
	= (1-\mu_2)\ln\frac{1-\mu_2}{1-\mu_{2'}} + \mu_2\ln\frac{\mu_2}{\mu_{2'}}
	< (1+\eps)\kl(X_2, X_1). \label{eq:deplow.mu2'}
	\end{equation}
	There exists such $\mu_{2'}$, because the left hand side of $\eqref{eq:deplow.mu2'}$ is continuous as a function of $\mu_{2'}$. We use $\Ep$ and $\Prp$ to denote expectation and probability in the circumstance that arm $X_2$ is replaced by arm $X_{2'}$.
	
	We define the empirical KL-divergence after the first $s$ samples of the arm $2$/$2'$,
	\begin{equation*}
	\ekl_s=\sum_{t=1}^s Y_t, 
	\end{equation*}
	where
	\begin{equation*}
	Y_t = \begin{cases}
	\ln\frac{1-\mu_2}{1-\mu_{2'}}, &\mbox{if } X'_{2, t} = 0,\\
	\ln\frac{\mu_2}{\mu_{2'}},     &\mbox{if } X'_{2, t} = 1,\\
	0,                             &\mbox{if } X'_{2, t} = \none.
	\end{cases}
	\end{equation*}
	and $X'_{2, t}$ is result of the $t$-th sample of arm $2$/$2'$.
	Note that $(Y_t)$ are independent and $\E[Y_t]=\kl(X'_2, X'_{2'})$.
	
	First we prove
	\begin{equation}
	\Pr\left\{N_2(T)<\frac{1-\eps}{\kl(X'_2, X'_{2'})}\ln T \land \ekl_{N_2(T)} \le \left(1-\frac{\eps}{2}\right)\ln T\right\} = o(1). \label{eq:deplow.first}
	\end{equation}
	We use the shorthands
	\begin{equation}
	C_T=\left\{N_2(T)<\frac{1-\eps}{\kl(X'_2, X'_{2'})}\ln T \land \ekl_{N_2(T)} \le \left(1-\frac{\eps}{2}\right)\ln T\right\}, \label{eq:deplow.event}
	\end{equation}
	and
	\[f_T = \frac{1-\eps}{\kl(X'_2, X'_{2'})}\ln T.\]
	If arm $2$ is replaced by arm $2'$, we have
	\[\Prp\{C_T\} \le \Prp\{N_2(T) < f_T\} \le \frac{\Ep[T-N_2(T)]}{T-f_T},\]
	where the second inequality is due to Markov's inequality.
	Recall the definition of consistent strategy, as $2'$ is the only optimal arm, we have $\Ep[T-N_2(T)] = o(T^{\frac{\eps}{2}})$. And by $T-f_T=\Omega(T)$,
	$\Prp\{C_T\} = o(T^{\frac{\eps}{2}-1})$.
	Then we use the property of KL-divergence
	\[\Pr\{C_T\} = \Ep\left[\I\{C_T\} \cdot \exp\left(\ekl_{N_2(T)}\right)\right],\]
	then
	\[\Pr\{C_T\} = \Ep\left[\I\{C_n\} \cdot \exp\left(\ekl_{N_2(T)}\right)\right]
	\le \Prp\{C_T\} \cdot \exp\left[\left(1-\frac{\eps}{2}\right)\ln T\right]
	= \Prp\{C_T\}\cdot T^{1-\frac{\eps}{2}} = o(1). \]
	
	Second, we prove
	\begin{equation}
	\Pr\left\{N_2(T)<f_T \land \ekl_{T_2(T)} > \left(1-\frac{\eps}{2}\right) \ln T\right\} = o(1). \label{eq:deplow.second}
	\end{equation}
	We have
	\begin{align*}
	\Pr\left\{N_2(T)<f_T \land \ekl_{N_2(T)} > \left(1-\frac{\eps}{2}\right) \ln T\right\}
	&\le \Pr\left\{N_2(T)<f_T \land \max_{s\le f_T} \ekl_s > \left(1-\frac{\eps}{2}\right) \ln T\right\}\\
	&\le \Pr\left\{\max_{s\le f_T} \ekl_s > \left(1-\frac{\eps}{2}\right) \ln T\right\}.
	\end{align*}
	\qinshi{We may need a reference for maximal version of the strong law of large numbers. In the survey, the authors gave one, but the description is not understandable, I think.}
	Recall the definition of $\ekl_s$, which is a summation of independent random variables with the same distribution over a finite support, whose expectation is $\kl(X'_2, X'_{2'})$.
	So we apply the maximal version of the strong law of large numbers,
	and then \eqref{eq:deplow.second} holds, as $f_T\cdot \kl(X'_2, X'_{2'})=(1-\eps)\ln T$.
	
	In conclusion, combining Eq. \eqref{eq:deplow.first} and \eqref{eq:deplow.second}, we have $\Pr\{N_2(T)<f_T\} = o(1)$, implying
	\begin{align*}
	\E\left[N_2(T)\right]
	&\ge (1-o(1))\cdot f_T\\
	&= (1-o(1)) \cdot \frac{1-\eps}{\kl(X'_2, X'_{2'})}\ln T\\
	&\ge (1-o(1)) \cdot \frac{1-\eps}{1+\eps}\frac{\ln T}{\kl(X'_2, X'_1)}.
	\end{align*}
	Then \eqref{eq:deplow.play} holds, as $\eps$ can be any positive real number, and thus the theorem holds.
\end{proof}

\section{Results with $\infty$-norm TPM Conditions}

\subsection{TPM Conditions with the $\infty$-norm}

We first restate the original bounded smoothness condition in \cite{CWYW16} below, which is an
$\infty$-norm based condition.

\begin{mycond}[Bounded Smoothness]
	\label{cond:boundedsmoothness}
	We say that a CMAB-T problem instance satisfies {\em bounded smoothness}, if 
	there exists a continuous, strictly increasing (and thus invertible) function $f(\cdot)$ with $f(0)=0$,
	such that for any two distributions $D, D'\in \calD$ with expectation vectors 
	$\vmu=(\mu_1, \ldots, \mu_m)$ and $\vmu' = (\mu'_1, \ldots, \mu'_m)$, and for any $\Lambda > 0$, 
	we have $|r_{\vmu}(S) - r_{\vmu'}(S)| \le f(\Lambda)$ if $\max_{i \in \trig{S}}|\mu_i - \mu_i'| \le \Lambda$, 
	for all $S\in \calS$, where $\trig{S} = \{i\in [m] \mid \Pr_{X\sim D, \tau} \{i\in \tau(S,X) \} >0 \}$ is
	the set of arms that could be triggered by action $S$.
\end{mycond}
Note that $f(\cdot)$ may depend on problem instance parameters such as $m$, but not on 
action $S$ or mean vectors $\vmu$, $\vmu'$.

Similar to the 1-norm case, we use triggering probabilities to modulate the bounded smoothness condition to obtain
the following TPM version:

\begin{mycond} {\bf ($\infty$-Norm TPM Bounded Smoothness)}
	\label{cond:globalTPM}
	We say a CMAB-T problem instance
	satisfies the {\em triggering-probability-modulated (TPM) bounded smoothness} 
	with bounded smoothness function $f(x)$,
	if for any two distributions $D,D'\in \calD$ with
	expectation vectors $\vmu$ and $\vmu'$, any action $S$ and any $\Lambda>0$,
	we have $|r_S(\vmu) - r_S(\vmu')| \le f(\Lambda)$ if
	$\max_{i\in [m]} p_i^{D, S}|\mu_i-\mu'_i|\le \Lambda$.
\end{mycond}

Note that Condition~\ref{cond:globalTPM} is stronger than Condition~\ref{cond:boundedsmoothness} under
the same bounded smoothness function $f$.
This is because if we have $\max_{i\in [m]} |\mu_i-\mu'_i|\le \Lambda$, then
we have $\max_{i\in [m]} p_i^{D, S}|\mu_i-\mu'_i|\le \Lambda$.
Then if Condition~\ref{cond:globalTPM} holds, we have $|r_S(\vmu) - r_S(\vmu')| \le f(\Lambda)$.
This means that if Condition~\ref{cond:globalTPM} holds, we have $|r_S(\vmu) - r_S(\vmu')| \le f(\Lambda)$
if $\max_{i\in [m]} |\mu_i-\mu'_i|\le \Lambda$, which is exactly Condition~\ref{cond:boundedsmoothness}.

\wei{I think we do not need to say again that the cascading bandit or influence maximization bandit satisfy these
	$\infty$-norm conditions, since they could achieve better regret bounds with the 1-norm version.}

\subsection{Theorem and Proofs with $\infty$-norm TPM Conditions}

\begin{mythm}
	\label{thm:inf-norm}
	Suppose a CMAB-T problem instance $([m],\calS, \calD, \Dtrig, R)$ satisfies monotonicity (Condition~\ref{cond:monotone}).
	For a fixed environment instance $D\in \calD$ with expectation vector $\vmu$, 
	the $T$-round $(\alpha, \beta)$-approximation regret bound using an
	$(\alpha, \beta)$-approximation oracle in various cases are given below.
	\begin{enumerate}[(1)]
		\item \label{upper:1}
		For the CUCB algorithm on a problem instance that satisfies 
		TPM bounded smoothness (Condition~\ref{cond:globalTPM})
		with bounded smoothness function $f(x)$, together with $\Delta_{\min}>0$,
		the regret is at most
		\begin{align*}
		&\sum_{i\in [m]} 78\ln T \left(\frac{\Delta_{\min}^i}{f^{-1}(\Delta_{\min}^i)^2}
		+ \int_{\Delta_{\min}^i}^{\Delta_{\max}^i}\frac{1}{f^{-1}(x)^2}\dx\right)\\
		&\qquad + m \cdot \left[\left(\frac{\pi^2}{6}+1\right)
		\lceil-\log_2 f^{-1}(\Delta_{\min})\rceil_{0} + \frac{\pi^2}{3}+1\right] \cdot \Delta_{\max};
		\end{align*}
		\item \label{upper:2}
		For the CUCB algorithm on a problem instance that satisfies 
		TPM bounded smoothness (Condition~\ref{cond:globalTPM})
		with bounded smoothness function $f(x)=ax$, 
		the regret is at most
		$$25a\sqrt{mT\ln T} + m \cdot \left[\left(\frac{\pi^2}{6}+1\right)
		\left\lceil-\log_2 (\sqrt{156m\ln T/T})\right\rceil_0 + \frac{\pi^2}{3}+1\right] \cdot \Delta_{\max};$$
	\end{enumerate}
\end{mythm}

We have several remarks on Theorem~\ref{thm:inf-norm}.
First, the condition $\Delta_{\min}>0$ 
automatically holds if
the action space $\calS$ is finite.
Thus it is not an extra condition comparing to the result in~\cite{CWYW16} when actions
are set of base arms.
If $\Delta_{\min}$ is zero due to infinite $\calS$, 
then we do not have regret bound as in \ref{upper:1}, but we still have regret bound as in \ref{upper:2}.
Second, the regret bound in \ref{upper:1} is distribution-dependent bound, since it depends on $\Delta_{\min}^i$, which is determined
by the distribution $D$;
	regret bounds in \ref{upper:2} is distribution-independent bound, since $\Delta_{\max}$ can be easily replaced by a quantity only
depending on the problem instance, such as the maximum possible reward value.
%
Third, when $\Delta^i_{\min}=+\infty$, $\frac{\Delta_{\min}^i}{f^{-1}(\Delta_{\min}^i)^2} = 0$.

%
%

\subsubsection{Proof of Theorem~\ref{thm:inf-norm}}

In this subsection, we focus on giving a roadmap to prove Theorem~\ref{thm:inf-norm}
and showing the new techniques we invented to improve the regret bound.
The remaining part of the proof is roughly the
new calculation based on the old techniques (c.f. \cite{CWYW16}).

In this subsection, we omit $(\alpha, \beta)$-approximation for clarity,
in other words, we assume $\alpha=\beta=1$.
Generalization to accommodate $(\alpha, \beta)$ approximation can be found in the
discussion section.


To exploit the advantage of TPM bounded smoothness condition (Conditions~\ref{cond:globalTPM}),
for each arm $i$, we divide actions into groups according to $p_i^{D, S}$.

For convenience, we also allow to index the counters with $q_i^{D, S_t}>0$,
such that $N_{i, q_i^{D, S_t}}$ indicates the same counter as $N_{i, j}$ with $q_i^{D, S_t}=2^{-j}$.

We use a shorthand as follows. 
For every arm $i$ and action $S$, define
$$q_i^{D, S}=\begin{cases}
2^{-j}, &\mbox{if } S\in \calS_{i, j}^{D},\\
0,      &\mbox{if } p_i^{D, S}=0.
\end{cases}$$

%

\qinshi{I realized that just when $f^{-1}(\Delta)>p_i$, it is al right to not sample arm $i$,
	instead of requesting $f^{-1}(\Delta)>p_i$.
	Then we can decrease the constant in $\ell_T(\Delta, q)$ from $144$ to $72$
	as you can see.
	If using the factor $2.65$ instead of $3$ when applying Chernoff bound,
	the constant can be further improved to $64$.
	I am not sure if this improvement needs to be shown.
}

\begin{mydef}
	\label{def:ell}
	$$\ell_t(\Delta, q) = \begin{cases}
	0,                                                  &\mbox{if } q\le \frac{1}{2}f^{-1}(\Delta),\\
	\lfloor\frac{6\ln t}{f^{-1}(\Delta)^2}\rfloor+1,   &\mbox{if } q=1,\\
	\lfloor\frac{72q\ln t}{f^{-1}(\Delta)^2}\rfloor+1, &\mbox{otherwise}.
	\end{cases}$$
\end{mydef}

\qinshi{I need to introduce $f^{-1}(\Delta)\ge 2q_i$ into the concept of sufficiently sampled and insufficiently sampled.
	One way to do so is to add it directly to the filter.
	But this will make the formula too long.
	A second way is to change the definition of $\ell_T(\Delta, q)$.
	But our definition of sufficient is $N_{i,j}$ strictly greater than $\ell_T(\Delta, q)$,
	then we need to define $\ell_T(\Delta, q)=-1$ for $f^{-1}(\Delta)\ge 2q_i$.
	That is unnatural and also introduces a negative term in regret.
	So my plan is to change the definition of $\ell_T(\Delta, q)$ as above
	add define sufficient to be $N_{i,j}\ge \ell_T(\Delta, q)$.
}

To unify the proofs for distribution-dependent and distribution-independent bounds,
we introduce a positive real number $M$.
To prove the distribution-dependent bound, we will let $M=\Delta_{\min}$ or $M=\Delta_{\min}^i$ in some circumstances.
To prove the distribution-independent bound, we will let $M=\tilde{\Theta}(T^{-1/2})$
to balance bounds for $Reg(\{\Delta_{S_t} \ge M\}$ and $Reg(\{\Delta_{S_t} < M\})$.
And we implement $\Nt_t$ (Definition~\ref{def:nt}) with
$j_{\max}^i = j_{\max}(M) = \lceil -\log_2 f^{-1}(M)\rceil_0$
The following are three technical claims used in the main proof,
and we define the proofs of these claims to Section~\ref{sec:proofdetails}.

\begin{myclm}[Bound of insufficiently sampled regret]
	\label{clm:insuf}
	For any CMAB-T problem instance, any bounded smoothness function $f(x)$, any algorithm, any arm $i$, any natural number $j$ and any positive real number $M$,
	\begin{align*}
	Reg(\{\Delta_{S_t} \ge M, S_t\in \calS_{i, j}, N_{i, j, t-1} < \ell_T(\Delta_{S_t}, 2^{-j})\})
	\le \ell_T(M, 2^{-j})M + \int_{M}^{\max\{\Delta_{\max}^i, M\}}\ell_T(x, 2^{-j})\dx.
	\end{align*}
\end{myclm}

\begin{myclm}[Bound of sufficiently sampled regret for CUCB]
	\label{clm:suf1}
	For the CUCB algorithm on a problem instance that satisfies 
	TPM bounded smoothness (Condition~\ref{cond:globalTPM})
	with bounded smoothness function $f(x)$, 
	$$Reg(\{\Delta_{S_t} \ge M, \forall i, N_{i, q_i^{S_t}, t-1}\ge \ell_T(\Delta_{S_t}, q_i^{S_t})\})
	\le m \cdot (\lceil-\log_2 f^{-1}(M)\rceil_0+2)\cdot\frac{\pi^2}{6}\cdot\Delta_{\max}.$$
\end{myclm}

\wei{In context of (\ref{upper:1}) and (\ref{upper:2}) is unclear, because they have different conditions, one require $\Delta_{\min} > 0$ while the other require $f(x) = a x$.
	We need to specify the condition exactly. Is it only on Condition 3? Similarly, for the next two claims, are they only for Condition 4?}

\qinshi{Maybe a better way is to say ``as everything defined in main body of Theorem~1 and the problem satisfies Condition 3.'' By ``in context of (1) and (2)'', I mean, both in (1) and (2), the claim holds. Actually, it holds as long as we have the conditions that are stated both in (1) and (2) --- those in main body, and Condition 3 with $f(x)$.}

We continue the proof of Theorem~\ref{thm:inf-norm}. Fix a value $M > 0$, we have
\begin{align}
Reg(\{\})
&=   Reg(\{\Delta_{S_t} < M\}) + Reg(\{\Delta_{S_t} \ge M\}) \nonumber\\
&=   Reg(\{\Delta_{S_t} < M\}) + Reg(\{\Delta_{S_t} \ge M, \forall i, N_{i, q_i^{S_t}, t-1}\ge \ell_T(\Delta_{S_t}, q_i^{S_t})\}) \nonumber \\
& \quad\quad
+ Reg(\{\Delta_{S_t} \ge M, \exists i, N_{i, q_i^{S_t}, t-1}< \ell_T(\Delta_{S_t}, q_i^{S_t})\}) \nonumber\\
&\le Reg(\{\Delta_{S_t} < M\}) + Reg(\{\Delta_{S_t} \ge M, \forall i, N_{i, q_i^{S_t}, t-1}\ge \ell_T(\Delta_{S_t}, q_i^{S_t})\}) \nonumber\\
&\quad\quad + \sum_{i\in [m]} Reg(\{\Delta_{S_t} \ge M, N_{i, q_i^{S_t}, t-1}< \ell_T(\Delta_{S_t}, q_i^{S_t})\}) \nonumber\\
&\le Reg(\{\Delta_{S_t} < M\}) + Reg(\{\Delta_{S_t} \ge M, \forall i, N_{i, q_i^{S_t}, t-1}\ge \ell_T(\Delta_{S_t}, q_i^{S_t})\}) \nonumber\\
&\quad\quad + \sum_{i\in [m]} \sum_{j\ge 0} Reg(\{\Delta_{S_t} \ge M, S_t\in \calS_{i, j}, N_{i, q_i^{S_t}, t-1}< \ell_T(\Delta_{S_t}, q_i^{S_t})\}) \nonumber\\
&=   Reg(\{\Delta_{S_t} < M\}) + Reg(\{\Delta_{S_t} \ge M, \forall i, N_{i, q_i^{S_t}, t-1}\ge \ell_T(\Delta_{S_t}, q_i^{S_t})\}) \nonumber\\
&\quad\quad + \sum_{i\in [m]} \sum_{j\ge 0} Reg(\{\Delta_{S_t} \ge M, S_t\in \calS_{i, j}, N_{i, j, t-1}< \ell_T(\Delta_{S_t}, 2^{-j})\}). \label{eq:upper.expand}
\end{align}

For the last part, if $j\ge \lceil -\log_2 f^{-1}(M)\rceil_0+1$, then
$2^{-j}\le \frac{1}{2}f^{-1}(M)$ and 
$$\frac{1}{2}f^{-1}(\Delta_{S_t})\ge \frac{1}{2}f^{-1}(M)\ge 2^{-j}.$$
By Definition~\ref{def:ell}, $\ell_T(\Delta_{S_t}, 2^{-j})=0$.
Then $N_{i, j, t-1}<\ell_T(\Delta_{S_t}, 2^{-j})$ is impossible,
so
$$\sum_{j \ge \lceil -\log_2 f^{-1}(M)\rceil_0+1}
Reg(\{\Delta_{S_t} \ge M, S_t\in \calS_{i, j}, N_{i, j, t-1}< \ell_T(\Delta_{S_t}, 2^{-j})\})=0.$$

\begin{mylem}
	\label{lem:summation}
	For every arm $i$, the event-filtered regret
	\begin{align}
	&\sum_{j \ge 0} Reg(\{\Delta_{S_t} \ge M, S_t\in \calS_{i, j}, N_{i, j, t-1} < \ell_T(\Delta_{S_t}, 2^{-j})\}) \label{eq:summation.reg}\\
	&\qquad \le 78\ln T\left(
	\frac{M}{f^{-1}(M)^2}
	+\int_{M}^{\max\{\Delta_{\max}^i, M\}}\frac{1}{f^{-1}(x)^2}\dx
	\right) + (j_{\max}(M)+1) \cdot \Delta_{\max}^i. \nonumber
	\end{align}
\end{mylem}

\begin{proof}
	If $M > \Delta_{\max}^i$, it is impossible to have $\Delta_{S_t} \ge M$ and $S_t\in \calS_{i, j}$ at the same time and then $\eqref{eq:summation.reg} = 0$.
	Then the lemma holds trivially.
	So we may assume that $M \le \Delta_{\max}^i$.
	By Claim~\ref{clm:insuf},
	\begin{align}
	\eqref{eq:summation.reg}
	&= \sum_{j=0}^{j_{\max}(M)} Reg(\{\Delta_{S_t} \ge M, S_t\in \calS_{i, j}, N_{i, j, t-1} < \ell_T(\Delta_{S_t}, 2^{-j})\}) \nonumber\\
	&\le \sum_{j=0}^{j_{\max}(M)} \left(\ell_T(M, 2^{-j})M + \int_{M}^{\max\{\Delta_{\max}^i, M\}}\ell_T(x, 2^{-j})\dx\right) \nonumber\\
	&=   \sum_{j=0}^{j_{\max}(M)} \left(\ell_T(M, 2^{-j})M + \int_{M}^{\Delta_{\max}^i}\ell_T(x, 2^{-j})\dx\right) \nonumber\\
	&=   \sum_{j=0}^{j_{\max}(M)} \ell_T(M, 2^{-j})M + \int_{M}^{\Delta_{\max}^i}\sum_{j=0}^{j_{\max}(M)} \ell_T(x, 2^{-j})\dx. \label{eq:summation.apply}
	\end{align}
	We then expand the notation $\ell_T(\Delta, q)$ (c.f. Definition~\ref{def:ell}) with
	$$\ell_T(\Delta, q) \le \begin{cases}
	\frac{6\ln T}{f^{-1}(\Delta)^2}+1,   &\mbox{if } q=1,\\
	\frac{72q\ln T}{f^{-1}(\Delta)^2}+1, &\mbox{otherwise}.
	\end{cases}$$
	So for any $x\in [M, \Delta_{\max}^i]$,
	\begin{align}
	\sum_{j=0}^{j_{\max}(M)} \ell_T(x, 2^{-j})
	&=   \ell_T(x, 1) + \sum_{j=1}^{j_{\max}(M)} \ell_T(x, 2^{-j}) \nonumber\\
	&\le \left(\frac{6\ln T}{f^{-1}(x)^2}+1\right) + \sum_{j=1}^{j_{\max}(M)} \left(\frac{72 \cdot 2^{-j}\ln T}{f^{-1}(x)^2}+1\right) \nonumber\\
	&=   \frac{6\ln T}{f^{-1}(x)^2} + \sum_{j=1}^{j_{\max}(M)} \frac{72 \cdot 2^{-j}\ln T}{f^{-1}(x)^2} + j_{\max}(M) + 1\nonumber\\
	&\le   \frac{6\ln T}{f^{-1}(x)^2} + \frac{72\ln T}{f^{-1}(x)^2} + j_{\max}(M) + 1\nonumber\\
	&=   \frac{78\ln T}{f^{-1}(x)^2} + j_{\max}(M) + 1\nonumber.
	\end{align}
	Then we continue \eqref{eq:summation.apply} with
	\begin{align*}
	\eqref{eq:summation.apply}
	&\le \left(\frac{78\ln T}{f^{-1}(M)^2} + j_{\max}(M) + 1\right) \cdot M
	+ \int_{M}^{\Delta_{\max}^i}\left(\frac{78\ln T}{f^{-1}(x)^2} + j_{\max}(M) + 1\right)\dx \\
	&=   \frac{78\ln T}{f^{-1}(M)^2} \cdot M
	+ \int_{M}^{\Delta_{\max}^i}\frac{78\ln T}{f^{-1}(x)^2}\dx
	+ (j_{\max}(M) + 1) \cdot \Delta_{\max}^i\\
	&=   78\ln T\left(\frac{M}{f^{-1}(M)^2}
	+ \int_{M}^{\Delta_{\max}^i}\frac{1}{f^{-1}(x)^2}\dx\right)
	+ (j_{\max}(M) + 1) \cdot \Delta_{\max}^i.
	\end{align*}
	Hence the lemma holds.
\end{proof}

\begin{mylem}
	\label{lem:insM}
	For event-filtered regret
	\begin{equation}
	Reg(\{\Delta_{S_t} < M\}) + \sum_{i\in [m]} \sum_{j\ge 0} Reg(\{\Delta_{S_t} \ge M, S_t\in \calS_{i, j}, N_{i, j, t-1} < \ell_T(\Delta_{S_t}, 2^{-j})\}), \label{eq:insM.reg}
	\end{equation}
	\begin{enumerate}[(1)]
		\item take $M=\Delta_{\min}$ when $\Delta_{\min} > 0$,
		\begin{equation*}
		\eqref{eq:insM.reg} \le \sum_{i\in [m]} 78\ln T \left(\frac{\Delta_{\min}^i}{f^{-1}(\Delta_{\min}^i)^2}
		+ \int_{\Delta_{\min}^i}^{\Delta_{\max}^i}\frac{1}{f^{-1}(x)^2}\dx\right)
		+ m \cdot (j_{\max}(\Delta_{\min}) + 1) \cdot \Delta_{\max};
		\end{equation*}
		\item if $f(x)=ax$, then take $M=a\sqrt{156m\ln T/T}$,
		\begin{equation*}
		\eqref{eq:insM.reg} < 25a\sqrt{mT\ln T} + m \cdot (j_{\max}(a\sqrt{156m\ln T/T}) + 1) \cdot \Delta_{\max}.
		\end{equation*}
	\end{enumerate}
\end{mylem}

\begin{proof}
	\begin{enumerate}[(1)]
		\item If $\Delta_{S_t} < M = \Delta_{\min}$, then $\Delta_{S_t}=0$. So $Reg(\{\Delta_{S_t} < M\}) \le 0$.
		For every $i\in [m]$ and every integer $j$,
		we may replace $M$ with $\Delta_{\min}^i$ as below.
		\begin{align}
		&Reg(\{\Delta_{S_t} \ge M, S_t\in \calS_{i, j}, N_{i, j, t-1} <  \ell_T(\Delta_{S_t}, 2^{-j})\}) \label{eq:insM.part}\\
		=\; & Reg(\{\Delta_{S_t} \ge \Delta_{\min}, S_t\in \calS_{i, j}, N_{i, j, t-1} <  \ell_T(\Delta_{S_t}, 2^{-j})\}) \nonumber\\
		=\; & Reg(\{\Delta_{S_t} \ge \Delta_{\min}^i, S_t\in \calS_{i, j}, N_{i, j, t-1} <  \ell_T(\Delta_{S_t}, 2^{-j})\}). \nonumber
		\end{align}
		Then apply Lemma~\ref{lem:summation} with $M = \Delta_{\min}^i$, we have
		\begin{align*}
		\eqref{eq:insM.reg}
		&=   \sum_{i\in [m]} \sum_{j\ge 0} Reg(\{\Delta_{S_t} \ge M, S_t\in \calS_{i, j}, N_{i, j, t-1} < \ell_T(\Delta_{S_t}, 2^{-j})\})\\
		&\le \sum_{i\in [m]} \left[78\ln T
		\left(\frac{\Delta_{\min}^i}{f^{-1}(\Delta_{\min}^i)^2} + \int_{\Delta_{\min}^i}^{\Delta_{\max}^i}\frac{1}{f^{-1}(x)^2}\dx\right)
		+ (j_{\max}(\Delta_{\min}^i)+1) \cdot \Delta_{\max}^i\right]\\
		&\le \sum_{i\in [m]} 78\ln T \left(\frac{\Delta_{\min}^i}{f^{-1}(\Delta_{\min}^i)^2} + \int_{\Delta_{\min}^i}^{\Delta_{\max}^i}\frac{1}{f^{-1}(x)^2}\dx\right)
		+ m \cdot (j_{\max}(\Delta_{\min}) + 1) \cdot \Delta_{\max}.
		\end{align*}
		\item By Lemma~\ref{lem:summation}, for every arm $i$,
		\begin{align}
		&\sum_{j\ge 0} Reg(\{\Delta_{S_t} \ge M, S_t\in \calS_{i, j}, N_{i, j, t-1} < \ell_T(\Delta_{S_t}, 2^{-j})\}) \nonumber\\
		\le\;& 78\ln T\left(\frac{M}{f^{-1}(M)^2}
		+ \int_{M}^{\Delta_{\max}^i}\frac{1}{f^{-1}(x)^2}\dx\right)
		+ (j_{\max}(M) + 1) \cdot \Delta_{\max}^i \nonumber\\
		=  \;& 78\ln T\left(\frac{M}{(a^{-1}M)^2}
		+ \int_{M}^{\Delta_{\max}^i}\frac{1}{(a^{-1}x)^2}\dx\right)
		+ (j_{\max}(M) + 1) \cdot \Delta_{\max}^i \nonumber\\
		=  \;& 78\ln T\left(\frac{1}{a^{-2}M}
		+ \int_{M}^{\Delta_{\max}^i}\frac{1}{a^{-2}x^2}\dx\right)
		+ (j_{\max}(M) + 1) \cdot \Delta_{\max}^i \nonumber\\
		\le\;& 78\ln T\left(\frac{1}{a^{-2}M} + \frac{1}{a^{-2}M}\right)
		+ (j_{\max}(M) + 1) \cdot \Delta_{\max}^i \nonumber\\
		=  \;& \frac{156\ln T}{a^{-2}M} + (j_{\max}(M) + 1) \cdot \Delta_{\max}. \label{eq:insM.int}
		\end{align}
		$Reg(\{\Delta_{S_t} < M\}) < TM$ as the regret in each round is less than $M$.
		So by \eqref{eq:insM.int} and take $M=a\sqrt{156m\ln T/T}$,
		\begin{align*}
		\eqref{eq:insM.reg}
		&<   TM + \frac{156m\ln T}{a^{-2}M} + m \cdot (j_{\max}(M) + 1) \cdot \Delta_{\max}\\
		&=   a\sqrt{156mT\ln T} + a\sqrt{156mT\ln T} + m \cdot (j_{\max}(M) + 1) \cdot \Delta_{\max}\\
		&<   25a\sqrt{mT\ln T} + m \cdot (j_{\max}(a\sqrt{156m\ln T/T}) + 1) \cdot \Delta_{\max}.
		\end{align*}
	\end{enumerate}
\end{proof}

\begin{proof}[Proof of Theorem~\ref{thm:inf-norm}]
	\begin{enumerate}[(1)]
		\item Since $\Delta_{\min}>0$, we can take $M=\Delta_{\min}$. By Lemma~\ref{lem:insM}(1) and Claim~\ref{clm:suf1}, we continue Inequality \eqref{eq:upper.expand} as below.
		\begin{align*}
		\eqref{eq:upper.expand}
		&\le \sum_{i\in [m]} 78\ln T \left(\frac{\Delta_{\min}^i}{f^{-1}(\Delta_{\min}^i)^2}
		+ \int_{\Delta_{\min}^i}^{\Delta_{\max}^i}\frac{1}{f^{-1}(x)^2}\dx\right)
		+ m \cdot (j_{\max}(\Delta_{\min}) + 1) \cdot \Delta_{\max}\\
		&\qquad + m \cdot (j_{\max}(\Delta_{\min}) + 2) \cdot \frac{\pi^2}{6} \cdot \Delta_{\max}\\
		&=   \sum_{i\in [m]} 78\ln T \left(\frac{\Delta_{\min}^i}{f^{-1}(\Delta_{\min}^i)^2}
		+ \int_{\Delta_{\min}^i}^{\Delta_{\max}^i}\frac{1}{f^{-1}(x)^2}\dx\right)\\
		&\qquad + m \cdot \left[\left(\frac{\pi^2}{6}+1\right)
		\lceil-\log_2 f^{-1}(\Delta_{\min})\rceil_0 + \frac{\pi^2}{3}+1\right] \cdot \Delta_{\max}.
		\end{align*}
		\item Take $M=a\sqrt{156m\ln T/T}$, by Lemma~\ref{lem:insM}(2) and Claim~\ref{clm:suf1}, we continue Inequality \eqref{eq:upper.expand} as below.
		\begin{align*}
		\eqref{eq:upper.expand}
		&\le 25a\sqrt{mT\ln T} + m \cdot (j_{\max}(a\sqrt{156m\ln T/T}) + 1) \cdot \Delta_{\max}\\
		&\qquad + m \cdot (j_{\max}(a\sqrt{156m\ln T/T}) + 2) \cdot \frac{\pi^2}{6} \cdot \Delta_{\max}\\
		&=   25a\sqrt{mT\ln T} + m \cdot \left[\left(\frac{\pi^2}{6}+1\right)
		\left\lceil-\log_2 (\sqrt{156m\ln T/T})\right\rceil_0 + \frac{\pi^2}{3}+1\right] \cdot \Delta_{\max}.
		\end{align*}
	\end{enumerate}
\end{proof}

\subsubsection{Proof details} \label{sec:proofdetails}

In this subsection, we finish the remaining part of the proof,
i.e. the proofs of the claims. We first prove the bound of sufficiently sampled part, namely Claims~\ref{clm:suf1}.
To do so, we define two kinds of niceness, that the difference between $\mu_i$ and $\hat\mu_i$
is small enough and that $T_i$ is large enough comparing with $N_{i,j}$, 
and then show that both kinds of niceness are satisfied with high probability
and if so, it is impossible to play a bad action.
We then prove Claim~\ref{clm:insuf}.
In this subsection we assume $M$ is already defined as a positive real number as in the proof of Theorem~\ref{thm:inf-norm}. 
Notations $\hat\vmu_t, \hat\mu_{i, t}, \bar\vmu_t, \bar\mu_{i, t}$ denote the values of $\hat\vmu, \hat\mu_i, \bar\vmu, \bar\mu_i$ at the end
of round $t$, respectively.

\wei{I think we still need to explicitly put subscript $t$ to avoid confusion. Actually,
	by convention, $\hat\mu_{i,t}$ is the value of $\hat\mu_i$ at the end of round $t$, but the following
	definition refers to the beginning of round $t$, so I need to use $\hat\mu_{i,t-1}$.
	Please check such subscripts that I have changed.}

We now prove the claims.
\begin{proof}[Proof of Claim~\ref{clm:suf1}]
	Explicitly,
	\begin{align}
	& Reg(\{\Delta_{S_t} \ge M, \forall i, N_{i, q_i^{S_t}, t-1} \ge \ell_T(\Delta_{S_t}, q_i^{S_t})\}) 
	\nonumber \\
	&=   \sum_{t=1}^T \E[\Delta_{S_t} \cdot \I\{\Delta_{S_t} \ge M,
	\forall i, N_{i, q_i^{S_t}, t-1} \ge \ell_T(\Delta_{S_t}, q_i^{S_t})\}] \nonumber \\
	&\le \sum_{t=1}^T \Pr\{\Delta_{S_t} \ge M,
	\forall i, N_{i, q_i^{S_t}, t-1} \ge \ell_T(\Delta_{S_t}, q_i^{S_t})\} \cdot \Delta_{\max}. \label{eq:sufsample1}
	\end{align}
	We only need to bound
	$\Pr\{\Delta_{S_t} \ge M, \forall i, N_{i, q_i^{S_t}, t-1} \ge \ell_T(\Delta_{S_t}, q_i^{S_t})\}$,
	i.e.\ the probability that for every $i$, there is $N_{i, q_i^{S_t}, t-1} \ge \ell_T(\Delta_{S_t}, q_i^{S_t})$,
	but an action $S_t$ with $\Delta_{S_t} \ge M$ is still played.
	Let event $\calE_t = \{ \Delta_{S_t} \ge M, \forall i, N_{i, q_i^{S_t}, t-1} \ge \ell_T(\Delta_{S_t}, q_i^{S_t}) \}$.
	We now prove the claim that event $\calE_t$ is not empty only when $\lnot(\Ns_t\land \Nt_t)$, or equivalently
	if both the sampling and triggering are nice at the beginning of round $t$, then event $\calE_t$ is empty.
	If the sampling is nice at the beginning of round $t$,
	then
	$$\bar\mu_{i,t-1}=\min\{\hat\mu_{i,t-1}+\rho_{i,t}, 1\}\ge \mu_i.$$
	By monotonicity, $r_S(\bar\vmu_{t-1})\ge r_S(\vmu)$
	for every action $S$, so $\opt_{\bar\vmu_{t-1}} \ge \opt_{\vmu}$.
	As action $S_t$ is chosen by \textsf{Oracle} with input $\bar\vmu_{t-1}$,
	it must be that $r_{S_t}(\bar\vmu_{t-1}) = \opt_{\bar{\vmu}_{t-1}} \ge \opt_{\vmu}$,
	so $r_{S_t}(\bar\vmu_{t-1})-r_{S_t}(\vmu) \ge \opt_{\vmu} -r_{S_t}(\vmu) = \Delta_{S_t}$.
	We are going to show the claim by assuming $\Ns_t \land \Nt_t$
	and showing $\forall i, p_i^{S_t}|\bar\mu_{i,t-1}-\mu_i|<f^{-1}(\Delta_{S_t})$,
	then by $\infty$-norm TPM bounded smoothness (Condition~\ref{cond:globalTPM}),
	$r_{S_t}(\bar\vmu_{t-1})-r_{S_t}(\vmu)<\Delta_{S_t}$, which is a contradiction.
	Note that here we do need strict inequality ``$<$'' instead of ``$\le$'' when applying Condition~\ref{cond:globalTPM}.
	This can be done because $i$ has at most $m$ choices and the bounded smoothness function $f$
	is continuous and strictly increasing, so we can use a small enough $\varepsilon > 0$ such that
	$\forall i, p_i^{S_t}|\bar\mu_{i,t-1}-\mu_i| \le f^{-1}(\Delta_{S_t} - \varepsilon)$, and thus
	$r_{S_t}(\bar\vmu_{t-1})-r_{S_t}(\vmu)\le \Delta_{S_t} - \varepsilon < \Delta_{S_t}$.
	
	Below we omit $S_t$ from $\Delta_{S_t}$, $p_i^{S_t}$ and $q_i^{S_t}$. If $f^{-1}(\Delta)>p_i$, then $p_i|\bar\mu_{i,t-1}-\mu_i|\le p_i|1-0|<f^{-1}(\Delta)$
	without any dependency on sampling.
	If $f^{-1}(\Delta)\le p_i$, then $q_i \le 2^{\lceil-\log_2f^{-1}(\Delta)\rceil} \le 2^{j_{\max}(M)}$.
	When the sampling is nice (Definition~\ref{def:ns}), $\bar\mu_{i,t-1}\le \hat\mu_{i,t-1}+ \rho_{i,t} <\mu_i+2 \rho_{i,t}$.
	On the other hand, $|\bar{\mu}_{i, t-1} - \mu_i| \le |1-0| = 1$.
	When the triggering is nice (Definition~\ref{def:nt}),
	if $\sqrt{\frac{6\ln t}{\frac{1}{3}N_{i, q_i, t-1}\cdot q_i}}\le 1$,
	then $2\rho_{i, t} \le \sqrt{\frac{6\ln t}{\frac{1}{3}N_{i, q_i, t-1}\cdot q_i}}$.
	So regardless whether $\sqrt{\frac{6\ln t}{\frac{1}{3}N_{i, q_i, t-1}\cdot q_i}}\le 1$,
	$|\bar{\mu}_{i, t-1} - \mu_i| \le \sqrt{\frac{6\ln t}{\frac{1}{3}N_{i, q_i, t-1}\cdot q_i}}$.
	Event $\calE_t$ implies that
	$N_{i,q_i, t-1} \ge \ell_T(\Delta, q_i) \ge \ell_t(\Delta, q_i)$ (since $t\le T$).
	So
	\begin{align*}p_i|\bar{\mu}_{i, t-1} - \mu_i|
	\le p_i\sqrt{\frac{6\ln t}{\frac{1}{3}N_{i, q_i, t-1}\cdot q_i}}
	\le p_i\sqrt{\frac{6\ln t}{\frac{1}{3}\ell_t(\Delta, q_i)\cdot q_i}}
	< p_i\sqrt{\frac{6\ln t}{\frac{1}{3}\frac{72q_i\ln t}{f^{-1}(\Delta)^2}\cdot q_i}}\\
	= p_i\sqrt{\frac{f^{-1}(\Delta)^2}{4q_i^2}}
	\le p_i\sqrt{\frac{f^{-1}(\Delta)^2}{p_i^2}}
	= f^{-1}(\Delta).
	\end{align*}
	Hence, the claim holds.
	
	The claim implies that $\Pr\{\calE_t \} \le \Pr\{\lnot(\Ns_t\land \Nt_t)\} \le \Pr\{\lnot \Ns_t \} + \Pr\{\lnot  \Nt_t\}$.
	By Lemmas~\ref{lem:ns} and~\ref{lem:nt}, we have $\Pr\{\calE \} \le (2+j_{\max}(M))mt^{-2}$.
	Plugging it into Inequality~\eqref{eq:sufsample1}, we have
	\begin{align*}
	Reg(\{\Delta_{S_t} \ge M, \forall i, N_{i, q_i^{S_t}, t-1} \ge \ell_T(\Delta_{S_t}, q_i^{S_t})\})
	& \le \sum_{t=1}^T (2+j_{\max}(M))mt^{-2} \cdot\Delta_{\max}\\
	& \le  m \cdot (\lceil-\log_2 f^{-1}(M)\rceil_0+2)\cdot\frac{\pi^2}{6}\cdot\Delta_{\max}.
	\end{align*}
\end{proof}


\begin{proof}[Proof of Claim~\ref{clm:insuf}]
	Let $x$ be any real number that $x \ge M > 0$.
	In any round when an action $S$ with $S \in \calS_{i, j}$ is played,
	$N_{i, j}$ is increased by $1$. So
	$$\sum_{t=1}^{T} \Pr\{S_t \in \calS_{i, j}, N_{i, j, t-1} < \ell_T(x, 2^{-j})\}
	\le \ell_T(x, 2^{-j}).
	$$
	If we add an additional restriction $\Delta_{S^t} \ge x$,
	the probability will not increase, so
	$$\sum_{t=1}^{T} \Pr\{\Delta_{S_t} \ge x, S_t \in \calS_{i, j},
	N_{i, j, t-1} < \ell_T(x, 2^{-j})\} \le \ell_T(x, 2^{-j}).
	$$
	We use the shorthand $\calE_{i, j}^{S_t}$ to denote
	the event $\{S_t \in \calS_{i, j}, N_{i, j, t-1} < \ell_T(x, 2^{-j})\}$.
	Suppose $X$ is a non-negative random variable with $\Pr\{X \ge M\}=p$ and $\Pr\{X = 0\} = 1-p$.
	Then by the basic principal on expectation, we have 
	\begin{align*}
	\E[X] 
	& = \int_{0}^{+\infty} \Pr\{X \ge x\} \dx = \int_{0}^{M} \Pr\{X \ge x\} \dx + \int_{M}^{+\infty} \Pr\{X \ge x\} \dx \\
	& = p M + \int_{M}^{+\infty} \Pr\{X \ge x\} \dx.
	\end{align*}
	Applying the above, we have
	\begin{align*}
	& Reg(\{\Delta_{S_t} \ge M\} \cap \calE_{i, j}^{S_t}) \\
	&= \sum_{t=1}^T \E[\I(\{\Delta_{S_t} \ge M\} \cap \calE_{i, j}^{S_t}) \cdot \Delta_{S_t}] \\
	&= \sum_{t=1}^T
	\left(\Pr[\{\Delta_{S_t} \ge M\} \cap \calE_{i, j}^{S_t}] \cdot M
	+ \int_{M}^{+\infty} \Pr[\{\Delta_{S_t} \ge x\} \cap \calE_{i, j}^{S_t}]\dx\right)\\
	&= \sum_{t=1}^T \Pr[\{\Delta_{S_t} \ge M\} \cap \calE_{i, j}^{S_t}] \cdot M
	+ \int_{M}^{+\infty} \sum_{t=1}^T \Pr[\{\Delta_{S_t} \ge x\} \cap \calE_{i, j}^{S_t}]\dx\\
	&= \sum_{t=1}^T \Pr[\{\Delta_{S_t} \ge M\} \cap \calE_{i, j}^{S_t}] \cdot M
	+ \int_{M}^{\max\{\Delta_{\max}^i, M\}} \sum_{t=1}^T \Pr[\{\Delta_{S_t} \ge x\} \cap \calE_{i, j}^{S_t}]\dx\\
	&\le \ell_T(M,2^{-j}) M
	+ \int_{M}^{\max\{\Delta_{\max}^i, M\}} \ell_T(x,2^{-j}) \dx.
	\end{align*}
\end{proof}


\subsection{Comparison between 1-norm and $\infty$-norm}
In this paper, we give upper bounds of regret for CMAB-T problems
that satisfy TPM bounded smoothness with 1-norm or with $\infty$-norm.
We emphasis Theorem~\ref{thm:1-normTPM} and Theorem~\ref{thm:inf-norm} do not imply each other.
For clarity, we use $a_1$ and $a_{\infty}$ in place of $a$ in bounded smoothness function $f(x)=ax$.
If a CMAB-T problem instance satisfies TPM bounded smoothness with 1-norm with $f(x)=a_1x$,
then it also satisfies TPM bounded smoothness with $\infty$-norm with $f(x)=a_{\infty}x$,
where $a_{\infty}=Ka_1$.
Conversely, if a CMAB-T problem instance satisfies TPM bounded smoothness with $\infty$-norm with $f(x)=a_{\infty}x$,
then it also satisfies TPM bounded smoothness with 1-norm with $f(x)=a_1x$,
where $a_1=a_{\infty}$.
For distribution-dependent upper bound, according to Theorems~\ref{thm:1-normTPM} and \ref{thm:inf-norm},
we have $O(\frac{a_1^2Km\ln T}{\Delta})$ and $O(\frac{a_{\infty}^2m\ln T}{\Delta})$ respectively.
For a problem instance that satisfies TPM bounded smoothness with 1-norm with $f(x)=a_1x$,
if we use the bound for $\infty$-norm with $a_{\infty}=Ka_1$, the result will be $O(\frac{a_1^2K^2m\ln T}{\Delta})$.
For a problem instance that satisfies TPM bounded smoothness with $\infty$-norm with $f(x)=a_{\infty}x$,
if we use the bound for 1-norm with $a_1=a_{\infty}$, the result will be $O(\frac{a_{\infty}^2Km\ln T}{\Delta})$.
Both give an additional $K$ factor.
It is similar for distribution-independent bound, which will have an additional $\sqrt{K}$ factor in both cases.

	\section{Refined regret bounds for probabilistically triggered linear bandit} \label{sec:linbandit}
\newcommand{\va}{\ensuremath{\boldsymbol a}}%
\newcommand{\vB}{\ensuremath{\boldsymbol B}}%
\newcommand{\ro}{\mathrm{ro}}%

In this section, we present our new result 
that takes advantages from special properties on certain instances of
combinatorial semi-bandits, e.g.
matroid bandits \cite{KWAEE14} and classical combinatorial semi-bandits \cite{KWAS15}.
In Theorem~\ref{thm:1-normTPM},
the regret bound is taking summation over all base arms.
In those special cases, the achieved regret bounds
only take summation over the base arms that are not in the optimal action.

We give a general condition to characterize this property
and generalize to the case with probabilistic triggering.
In particular, it is satisfied by CMAB-T with linear reward,
  which are CMAB-T problem instances such that,
  for each action $S$, the expected reward is linear
  with respect to the expectations of arms and independent to the arms with $p_i^S=0$,
  i.e. there exists a vector $\va_S\in \R^m$, such that
\begin{equation}
  r_S(\vmu)=\va_S\cdot\vmu,
\end{equation}
and $p_i^S=0$ implies $(a_S)_i=0$.
We show that our regret bounds given in this section implies
  state-of-the-art regret bounds for classical combinatorial semi-bandits.
We also show that the regret bounds can asymptotically match
  the state-of-the-art regret bounds for matroid semi-bandits by exploiting its special property.

\subsection{Model}
First, we make a natural assumption that there exists at least one optimal action $S_{\opt}$.
This assumption is always satisfied unless the action space $\calS$ is infinite
  and the supremum of expected reward is not achieved by any action.

Although the CUCB algorithm does not know the actual expectation vector $\vmu$,
  we use $\vmu$ in the analysis.
For every action $S$, arbitrarily designate a reference optimal action $\ro(S)$,
  which is an optimal action under $\vmu$.

The intuition of the following refined condition is as follows.
For CMAB-T with linear reward,
although most regret comes from the over-estimation in CUCB algorithm,
we notice that sometimes over-estimation does not result in choosing a non-optimal action
because it increases the expected reward for both optimal actions and non-optimal actions.
For any arm $i$,
if $\left(a_S\right)_i\le \left(a_{\ro(S)}\right)_i$, then the over-estimation on arm $i$
  favors the optimal action $\ro(S)$.
So if $S$ is played, it is due to the over-estimation on other arms.
\wei{I do not quite follow the above sentence. Are you saying it for a particular $i$? If so, how is it true?
I am assuming that $S$ is the action selected in round $t$, under an over estimate $\bar{\vmu}$. 
So why is that if 	$\left(a_S\right)_i\le \left(a_{\ro(S)}\right)_i$, we will prefer selecting $\ro(S)$ instead of $S$?
We only gain in dimension $i$, what above other dimensions? 
Overall, I think I can follow the technical part starting from the condition below, but does not quite follow the intuition part.
}
And if $\left(a_S\right)_i> \left(a_{\ro(S)}\right)_i$,
we can replace the bounded smoothness factor with the relative factor $(\left(a_S\right)_i - \left(a_{\ro(S)}\right)_i)$.
We characterize this property as the following general condition.
\begin{mycond}[Relative 1-Norm TPM Bounded Smoothness]
\label{cond:relative.1-normTPM}
For a CMAB-T problem instance, a distribution $D\in\calD$ with expectation vector $\vmu$,
  and a reference optimal action mapping $\ro(\cdot)$,
  we say they satisfy relative 1-norm TPM bounded smoothness,
  if there exists a vector $\vB\in \R^m$ such that,
  for every distribution $D'\in\calD$ with expectation vector $\vmu'$ that $\vmu'\ge \vmu$,
  if $r_S(\vmu')\ge r_{\ro(S)}(\vmu')$,
  then 
  \begin{equation}
      r_{\ro(S)}(\vmu) - r_S(\vmu) \le \sum_{i\in [m]} B_ip_i^{D, S}(\mu'_i - \mu_i).
      \label{eq:relative.1-normTPM}
  \end{equation}
\end{mycond}

\begin{mylem}
For problem instance of CMAB-T with linear reward and a reference optimal action mapping $\ro(\cdot)$, let
\[ B_i=\sup_{S\in \calS\mid p_i^{D, S}>0} \left\{\max\left(\left(a_S\right)_i - \left(a_{\ro(S)}\right)_i, 0\right)/p_i^{D, S}\right\}. \]
Then they satisfy relative 1-norm TPM bounded smoothness.
\end{mylem}
\begin{proof}
We need to show that $\vmu'\ge \vmu$ and $r_S(\vmu')\ge r_{\ro(S)}(\vmu')$
implies \eqref{eq:relative.1-normTPM}.
Since $r_S(\vmu')\ge r_{\ro(S)}(\vmu')$,
we have $\left(\va_S - \va_{\ro(S)}\right)\cdot\vmu'\ge 0$.
Then we complete the proof by
\[
    \va_{\ro(S)}\cdot \vmu-\va_S\cdot \vmu \le \left(\va_S-\va_{\ro(S)}\right)\cdot (\vmu'-\vmu)
    \le \sum_{i\in [m]} B_ip_i^{D, S}(\mu'_i - \mu_i).
\]
\end{proof}

If the reference optimal action $\ro(S)$ are the same for all the actions, we can denote it as $\ro$. For linear bandit,
the definition of $B_i$ can be simplified to $B_i=(a_{\max})_i-(a_{\ro})_i$, where $(a_{\max})_i=\sup_{S\in \calS} (a_S)_i$.

\subsection{Results}
Theorem~\ref{thm:relative.1-normTPM} shows that the regret bounds given in Theorem~\ref{thm:1-normTPM}
still holds when the CMAB-T problem satisfies the relative 1-norm bounded smoothness.

\begin{mythm}
\label{thm:relative.1-normTPM}
For the CUCB algorithm on a CMAB-T problem instance that satisfies monotonicity (Condition~\ref{cond:monotone}),
  and a distribution and a reference optimal action mapping $\ro(\cdot)$
  that satisfy the relative 1-norm TPM bounded smoothness (Condition~\ref{cond:relative.1-normTPM})
  with $\vB$,
(1) if $\Delta_{\min} > 0$, we have gap-dependent bound
    \begin{align}
    &Reg_{\vmu, \alpha, \beta}(T) \le \sum_{i\in[m]} \frac{576B_i^2K\ln T}{\Delta_{\min}^i} 
    + \sum_{i\in[m]}\left(\left\lceil\log_2 \frac{2B_iK}{\Delta_{\min}^i}\right\rceil_0 + 2\right)
    \cdot \frac{\pi^2}{6} \cdot \Delta_{\max} +\sum_{i\in [m]}4B_i;
    \end{align}
(2) we have \emph{gap}-independent bound (it is not distribution-independent, because it depends on $\ro(\cdot)$)
    \begin{align}
    &Reg_{\vmu, \alpha, \beta}(T) \le 12\sqrt{\sum_{i\in[m]}B_i^2mT\ln T}
    + \left(\left\lceil\log_2 \frac{T}{18\ln T}\right\rceil_0+2\right) \cdot m \cdot \frac{\pi^2}{6}\cdot \Delta_{\max}
    +\sum_{i\in[m]}2B_i.
    \end{align}
\end{mythm}

\begin{proof}
The proof is mostly the same as the proof of Theorem~\ref{thm:1-normTPM} and its auxiliary lemmas.
For Lemma \ref{lem:TPMkappa},
  we use the definitions of $\kappa_{i, j, T}(M, s)$ and $\ell_{i, j, T}(M)$
  defined in Appendix~\ref{app:refineB}, where $B$ is replaced by $B_i$.
Then Lemma~\ref{lem:TPMkappa} still holds as follows.
In the proof of Lemma~\ref{lem:TPMkappa},
  the 1-norm TPM bounded smoothness is used to show
\begin{equation}
    \Delta_{S_t}\le B\sum_{i\in \trig{S_t}} p_i^{D, S_t}(\bar\mu_{i, t} - \mu_i). \tag{\ref{eq:TPMkappa.applycondition}}
\end{equation}
Since $\Delta_{S_t}=r_{\ro(S_t)}(\vmu)-r_{S_t}(\vmu)$,
  the relative 1-norm TPM bounded smoothness implies
\[ \Delta_{S_t} \le \sum_{i\in \trig{S_t}} B_ip_i^{D, S_t}(\bar\mu_{i, t} - \mu_i). \]
Then the remaining proof of Lemma~\ref{lem:TPMkappa} still holds.

We replace $B$ with $B_i$, then the derivation of Lemma~\ref{lem:1-norm.insuf} still works.

One argument in the proof of Theorem~\ref{thm:1-normTPM} needs to be modified.
When setting the maximal group number $j_{\max}^i$
  in the proof of Theorem~\ref{thm:1-normTPM},
  we now set $j_{\max}^i=\left\lceil\log_2\frac{2B_iK}{\Delta_{\min}{i}}\right\rceil_0$.
Then the theorem follows the proof of Theorem~\ref{thm:1-normTPM}.
\end{proof}

\subsection{Comparison}
Matroid bandits \cite{KWAEE14} and classical linear combinatorial semi-bandits~\cite{KWAS15} are
special cases of CMAB-T with linear reward, where $\va_S$ are a 0-1 vectors.
We show that these kinds of bandits can be extended to the case of probabilistically triggered arms
while the regret bounds are asymptotically the same.


Let $\ro$ be the reference optimal action.
Since the actions in matroid bandits and classical linear combinatorial semi-bandits
are virtually sets, we use $i\in \ro$ to denote arm $i$ is in action $\ro$,
i.e. $(a_{\ro})_i=1$.
Since $\va_S$ is a 0-1 vector for every $S$, we have $(a_{\max})_i=1$ and
\[B_i=(a_{\max})_i-(a_{\ro})_i=
\begin{cases}
  0,& i\in \ro;\\
  1,& i\not\in \ro.
\end{cases}
\]
Theorem~\ref{thm:relative.1-normTPM} gives regret bound
\[
Reg_{\vmu, \alpha, \beta}(T) \le \sum_{i\in[m]} \frac{576B_i^2K\ln T}{\Delta_{\min}^i} + O(1)
=\sum_{i\not\in\ro} \frac{576K\ln T}{\Delta_{\min}^i} + O(1).
\]

For the classical linear combinatorial semi-bandits,
we make the following arguments.
First,
if without probabilistically triggered arms,
the adaptation used to prove Theorem~\ref{thm:relative.1-normTPM}
can be used in Theorem~\ref{thm:1-norm.nontriggering},
to get regret bound (see Eq.\eqref{eq:1-norm.nontriggering1} for the detail of the constant $O(1)$)
\[\sum_{i\not \in \ro}\frac{48K\ln T}{\Delta_{\min}^i}+O(1),\]
which asymptotically matches the result in \cite{KWAS15} and improves on the constant factor.
Second,
this shows that the classical linear combinatorial semi-bandits can be generalized to
the case with probabilistically triggered arms while
still enjoying the benefit that the regret bound takes summation only over non-optimal arms.
Recall the definition of $\Delta_{\min}^i$, which is the minimal gap for non-optimal actions that might trigger arm $i$.
If $i$ is in the optimal action, then very likely $\Delta_{\min}^i=\Delta_{\min}$.
That shows why it is important to exclude arms in the optimal action from the regret bound.


\qinshi{Matroid bandits' regret bound is $\sum\frac{\ln T}{\Delta_i}$, instead of $\sum\frac{K\ln T}{\Delta_i}$.
We cannot prove regret bound like that with 1-norm condition.
One might think using $\infty$-norm condition can fill this gap.
But I double checked that this is not possible, either.
The key of $\sum\frac{\ln T}{\Delta_i}$ regret is that
error on different arms cannot accumulate to mislead the player to play a non-optimal action.
If an non-optimal action is played, there must be an arm such that
$\bar{\vmu}_i\ge \vmu_{\text{least opt}}$.
}


For matroid bandits,
the analysis for classic combinatorial semi-bandits gives a regret bound
\[\sum_{i\not \in \ro}\frac{48K\ln T}{\Delta_{\min}^i}+O(1),\]
with an extra $O(K)$ factor comparing with \cite{KWAEE14}.
That is because of the special property of matroid 
that the error on each arm cannot accumulate.
So if a non-optimal action is played, that must be caused by the error on a single arm.
That makes it valid
to set the bound of sufficiently sampled threshold to
$\ell_{T}(\Delta_{\min}^i)=\Theta(\frac{\ln T}{(\Delta_{\min}^i)^2})$
instead of $\ell_{T}(\Delta_{\min}^i)=\Theta(\frac{K^2\ln T}{(\Delta_{\min}^i)^2})$.
Refining this part of analysis is possible,
but it would be too specific for the discussion in this paper.
}{\empty}
		
\end{document}